\documentclass{article}

\usepackage{microtype}
\usepackage{graphicx}
\usepackage{booktabs} %
\usepackage{hyperref}

\usepackage{caption}
\usepackage{natbib}

\usepackage{amsmath,amssymb,amsfonts,amsthm}

\usepackage{fullpage}

\usepackage{mathtools}
\usepackage{dsfont}
\usepackage{hyperref}
\usepackage{bm}
\usepackage{nicefrac}
\usepackage{wrapfig}
\usepackage{lipsum}

\usepackage{algorithm,algorithmic}
\usepackage[linesnumbered, ruled,vlined,algo2e]{algorithm2e}

\newcounter{assumption}%
\renewcommand{\theassumption}{\arabic{assumption}}

\newenvironment{assumption}[1][]{\begin{trivlist}\item[] \refstepcounter{assumption}%
 {\bf{Assumption\ \theassumption\ }}{(#1)}.  }{%
 \ifvmode\smallskip\fi\end{trivlist}}

\newcommand{\rd}{\mathrm{d}}

\newcommand{\lip}{\widehat{L}}
\newcommand{\conv}{\widehat{m}}

\newcommand{\pot}{U}

\newcommand{\hpot}{\widehat{U}}

\def\rd{\mathrm{d}}
\def\rT{{\rm T}}

\def\mI{\mathrm{I}}

\def\mI{\mathrm{I}}

\newcommand*\E[1]{\mathbb{E}\left[#1\right]}
\newcommand*\Ep[2]{\mathbb{E}_{#1}\left[#2\right]}

\newcommand*\lrbb[1]{\left\{#1\right\}}
\newcommand*\lrp[1]{\left(#1\right)}
\newcommand*\lrn[1]{\left\|#1\right\|}

\newcommand*\ind[1]{{\mb{I}\lrbb{#1}}}

\def\mI{\mathrm{I}}

\def\rd{\mathrm{d}}
\def\rT{{\rm T}}

\def\mI{\mathrm{I}}

\newcommand{\hmu}{\widehat{\mu}}

\newcommand{\real}{\ensuremath{\mathbb{R}}}

\newtheorem{lemma}{Lemma}
\newtheorem{theorem}{Theorem}

\newtheorem{proposition}{Proposition}
\newtheorem{remark}{Remark}

\newtheorem{example}{Example}

\newtheorem{innercustomgeneric}{\customgenericname}
\providecommand{\customgenericname}{}
\newcommand{\newcustomtheorem}[2]{%
  \newenvironment{#1}[1]
  {%
   \renewcommand\customgenericname{#2}%
   \renewcommand\theinnercustomgeneric{##1}%
   \innercustomgeneric
  }
  {\endinnercustomgeneric}
}

\newcustomtheorem{assump}{Assumption}

\DeclareMathOperator*{\argmax}{argmax}
\newtheorem{theorem2}{Theorem}[section]
\usepackage{multirow}
\usepackage{caption}
\usepackage{subcaption} 

\newcommand{\mb}{\mathbb}
\newcommand{\mc}{\mathcal}
\newcommand{\rar}{\rightarrow}
   \usepackage{xcolor}

\title{On Thompson Sampling with Langevin Algorithms}
\author{
Eric Mazumdar$^{\dagger,}\thanks{Equal contribution.}$\\
\texttt{mazumdar@berkeley.edu}\\
\and
Aldo Pacchiano$^{\dagger,}\footnotemark[1]$\\
\texttt{pacchiano@berkeley.edu}\\
\and
Yi-An Ma$^{\diamond,}\footnotemark[1]$\\
\texttt{yianma@google.com}\\
\and
Peter L. Bartlett$^{\dagger, \ddagger}$\\
\texttt{peter@berkeley.edu}\\
\and
Michael I. Jordan$^{\dagger, \ddagger}$\\ \texttt{jordan@cs.berkeley.edu}\\
}
\date{%
    $\dagger$Department of Electrical Engineering and Computer Sciences\\%
    $\ddagger$Department of Statistics, University of California, Berkeley\\%
    $\diamond$Google Research
}

\begin{document}
\maketitle

%
%
%
%

%
%
%
%

%
%
%




%
%
%


%

%
%
%
%
%

%

\begin{abstract}
Thompson sampling for multi-armed bandit problems is known to enjoy favorable performance in both theory and practice. However, it suffers from a significant limitation computationally, arising from the need for samples from posterior distributions at every iteration.
We propose two Markov Chain Monte Carlo (MCMC) methods tailored to Thompson sampling to address this issue. We construct quickly converging Langevin algorithms to generate approximate samples that have accuracy guarantees, and we leverage novel posterior concentration rates to analyze the regret of the resulting approximate Thompson sampling algorithm. Further, we specify the necessary hyperparameters for the MCMC procedure to guarantee optimal instance-dependent frequentist regret while having low computational complexity. In particular, our algorithms take advantage of both posterior concentration and a sample reuse mechanism to ensure that only a constant number of iterations and a constant amount of data is needed in each round. The resulting approximate Thompson sampling algorithm has logarithmic regret and its computational complexity does not scale with the time horizon of the algorithm.
\end{abstract}

\section{Introduction}
Sequential decision making under uncertainty has become one of the fastest developing fields of machine learning.  A central theme in such problems is addressing exploration-exploitation tradeoffs \citep{AuerUCB,TorCsabaBandits}, wherein an algorithm must balance between exploiting its current knowledge and exploring previously unexplored options.

The classic stochastic multi-armed bandit problem has provided a theoretical laboratory for the study of exploration/exploitation tradeoffs~\citep{LaiRobbins}. A vast literature has emerged that provides algorithms, insights, and matching upper and lower bounds in many cases.  The dominant paradigm in this literature has been that of \emph{frequentist analysis}; cf.\ in particular the analyses devoted to the celebrated upper confidence bound (UCB) algorithm~\citep{AuerUCB}. Interestingly, however, Thompson sampling, a Bayesian approach first introduced almost a century ago~\citep{Thompson} has been shown to be competitive and sometimes outperform UCB algorithms in practice~\citep{BayesianBandit,Approx_Thompson1}. Further, the fact that Thompson sampling, being a Bayesian method, explicitly makes use of prior information, has made it particularly popular in industrial applications \citep[see, e.g.,][and the references therein]{intro_Thompson}. 

Although most theory in the bandit literature is focused on non-Bayesian methods, there is a smaller, but nontrivial, theory associated with Thompson sampling. In particular, Thompson sampling has been shown to achieve optimal risk bounds in multi-armed bandit settings with Bernoulli rewards and beta priors~\citep{Kauffman2012,AgrawlFurther}, Gaussian rewards with Gaussian priors~\citep{AgrawlFurther}, one-dimensional exponential family models with uninformative priors~\citep{KordaExpFamBandits}, and finitely-supported priors and observations~\citep{GopalanMannor}. Thompson sampling has further been shown to asymptotically achieve optimal instance-independent performance~\citep{RussoVanRoyInfoTheory}.

Despite these appealing foundational results, the deployment of Thompson sampling in complex problems is often constrained by its use of samples from posterior distributions, which are often difficult to generate in regimes where the posteriors do not have closed forms. A common solution to this has been to use \emph{approximate} sampling techniques to generate samples from \emph{approximations} of the posteriors~\citep{intro_Thompson,Approx_Thompson1,Approx_Thompson2,VanRoyEnsembleSampling}. Such approaches have been demonstrated  to work effectively in practice~\citep{EmpiricalApprox1,EmpiricalApprox2}, but it is unclear how to maintain performance over arbitrary time horizons while using approximate sampling. Indeed, to the best of our knowledge the strongest regret guarantees for Thompson sampling with approximate samples are given by \citet{VanRoyEnsembleSampling} who require a model whose complexity grows with the time horizon to guarantee optimal performance.  Further, it was recently shown theoretically by~\citet{phan2019thompson} that a na\"ive usage of approximate sampling algorithms with Thompson sampling can yield a drastic drop in performance.

\paragraph{Contributions}
In this work we analyze Thompson sampling with approximate sampling methods in a class of multi-armed bandit algorithms where the rewards are unbounded, but their distributions are log-concave. In Section~\ref{sec:exact} we derive posterior contraction rates for posteriors when the rewards are generated from such distributions and under general assumptions on the priors. Using these rates, we show that Thompson sampling with samples from the true posterior achieves finite-time optimal frequentist regret. Further, the regret guarantee we derive has explicit constants and explicit dependencies on the dimension of the parameter spaces, variance of the reward distributions, and the quality of the prior distributions. 

In Section~\ref{sec:approx} we present a simple counterexample demonstrating the relationship between the approximation error to the posterior and the resulting regret of the algorithm. Building on the insight provided by this example, we propose two approximate sampling schemes based on Langevin dynamics to generate samples from approximate posteriors and analyze their impact on the regret of Thompson sampling. 
We first analyze samples generated from the unadjusted Langevin algorithm (ULA) and specify the runtime, hyperparameters, and initialization required to achieve an approximation error which provably maintains the optimal regret guarantee of exact Thompson sampling over finite-time horizons. 
Crucially, we initialize the ULA algorithm from the approximate sample generated in the previous round to make use of the posterior concentration property and ensure that only a \emph{constant} number of iterations are required to achieve the optimal regret guarantee. 
Under slightly stronger assumptions, we then demonstrate that a stochastic gradient variant called \emph{stochastic gradient Langevin dynamics} (SGLD) 
requires only a \emph{constant} batch size in addition to the constant number of iterations to achieve logarithmic regret.
Since the computational complexity of this sampling algorithm does not scale with the time horizon, the proposed method is a true ``anytime'' algorithm. Finally, we conclude in Section~\ref{sec:numex_main} by validating these theoretical results in numerical simulations where we find that Thompson sampling with our approximate sampling schemes maintain the desirable performance of exact Thompson sampling.

Our results suggest that the tailoring of approximate sampling algorithms to work with Thompson sampling can overcome the phenomenon studied in \citet{phan2019thompson}, where approximation error in the samples can yield linear regret. Indeed, our results suggest that it is possible for Thompson sampling to achieve order-optimal regret guarantees with an efficiently implementable approximate sampling algorithm.

\section{Preliminaries}
\label{sec:prelims}
In this work we analyze Thompson sampling strategies for the $K$-armed stochastic multi-armed bandit (MAB) problem. In such problems, there is a set of  $K$  options or ``arms'', $\mathcal{A}=\{1,...,K\}$, from which a player must choose at each round $t = 1,2,...$. After choosing an arm $A_t\in \mc{A}$ in round $t$, the player receives a real-valued reward $X_{A_t}$ drawn from a fixed yet unknown distribution associated with the arm, $p_{A_t}$. The random rewards obtained from playing an arm repeatedly are i.i.d.\ and independent of the rewards obtained from choosing other arms. 

Throughout this paper, we assume that the reward distribution for each arm is a member of a parametric family parametrized by $\theta_a \in \mb{R}^{d_a}$ such that the true reward distribution is $p_a(X)=p_a(X ;\theta_a^*)$, where $\theta_a^*$ is unknown. Moreover, we assume throughout this paper that the parametric families are log-concave and Lipschitz smooth in $\theta_a$:
\begin{assump}{1-Local}[Assumption on the family $p_a(X|\theta_a)$ around $\theta_a^*$]
\label{assumption:family_assumption}
Assume that $\log p_a(x|\theta_a)$ is $L_a$-smooth and $m_a$-strongly concave around $\theta_a^*$ for all $X \in \mathbb{R}$:
%
\begin{multline*}
    -\log p_a(x| \theta_a^*) -\nabla_{\theta} \log p_a(x| \theta_a^*)^\top \left( \theta_a - \theta_a^*    \right) + \frac{m_a}{2}\| \theta_a - \theta_a^*\|^2
    \leq -\log p_a(x|\theta_a) \\ 
    \leq -\log p_a(x| \theta_a^*) -\nabla_{\theta} \log p_a(x| \theta_a^*)^\top \left( \theta_a - \theta_a^*    \right) + \frac{L_a}{2}\| \theta_a - \theta_a^*\|^2, 
    \quad \forall \theta_a \in \mathbb{R}^{d_a}, x \in \mb{R}.
\end{multline*}
\end{assump}
Additionally we make assumptions on the true distribution of the rewards:
\setcounter{assumption}{1}
\begin{assumption}[Assumption on true reward distribution $p_a(X|\theta^\ast_a)$]\label{assumption:true_reward_assumption}
For every $a \in \mathcal{A}$ assume that $p_a(X;\theta_a^\ast)$ is strongly log-concave in $X$  with some parameter $\nu_a$, and that $\nabla_\theta \log p_a(x|\theta_a^*)$ is $L_a$-Lipschitz in $X$:
\begin{align*}
    - \left(\nabla_x \log p_a(x|\theta_a^*)-\nabla_x \log p_a(x'|\theta_a^*)\right)^\rT (x-x')
    \ge \nu_a \|x-x'\|_2^2, \quad \forall x,x' \in \mathbb{R}.
\end{align*} 
\begin{align*}
    \|\nabla_\theta \log p_a(x|\theta_a^*)-\nabla_\theta \log p_a(x'|\theta_a^*)\|
    \le L_a\|x-x'\|_2, \quad \forall x,x' \in \mathbb{R}.
\end{align*}
\end{assumption}
Parameters $\nu_a$ and $L_a$ provide lower and upper bounds to the sub- and super-Gaussianity of the true reward distributions.
We further define $\kappa_a = \max\lrbb{L_a/m_a, L_a/\nu_a}$ to be the condition number of the model class.
Finally, we assume that for each arm $a \in \mathcal{A}$ there is a linear map such that for all $\theta_a \in \mb{R}^{d_a}$, $ \mathbb{E}_{x \sim p_a(x | \theta_a)}\left[    X  \right]=\alpha_a^T\theta_a$, with $\|\alpha_a\|=A_a$. 

We now review Thompson sampling, the pseudo-code for which is presented in Algorithm~\ref{alg:Thomspson}. A key advantage of Thompson sampling over frequentist algorithms for multi-armed bandit problems is its flexibility of incorporating prior information.  
In this paper, we assume that the prior distributions $\pi_a(\theta_a)$ over the parameters of the arms have smooth log-concave densities:
\begin{assumption}[Assumptions on the prior distribution] 
\label{assumption:prior_assumption} 
For every $a \in \mathcal{A}$ assume that $\log\pi_a(\theta_a)$ is concave with $L_a$-Lipschitz gradients for all $\theta_a \in \mb{R}^{d_a}$:\footnote{We remark that the Lipschitz constants are all assumed to be the same to simplify notation.}
\[ 
\|\nabla_\theta \pi_a(\theta_a)-\nabla_\theta\pi_a(\theta_a')\|\le L_a\|\theta_a-\theta_a' \|, \quad \forall \theta_a, \theta_a' \in \mb{R}^{d_a}.
\]
\end{assumption}
%
%
{
Thompson sampling proceeds by maintaining a posterior distribution over the parameters of each arm $a$ at each round $t$. Given the likelihood family, $p(X|\theta_a)$, the prior, $\pi(\theta_a)$, and the $n$ data samples from an arm $a$, $X_{a,1},\cdots,X_{a,n}$, let  $F_{n,a} : \mathbb{R}^{d_a} \rightarrow \mathbb{R}$ be $F_{n,a}(\theta_a) = \frac{1}{n} \sum_{i=1}^{n} \log p_a(X_{a,i} | \theta_a)$, be the average log-likelihood of the data. Then the posterior distribution over the parameter $\theta_a$ at round $t$, denoted $\mu_a^{(n)}$, satisfies:
\[ p_a(\theta_a|X_{a,1},\cdots,X_{a,n}) \propto \pi_a(\theta_a) \prod_{i=1}^{t} \lrp{ p_a(X_{t}|\theta_a) }^{\ind{A_t=a}} =\exp\lrp{ n F_{n,a}(\theta_a) + \log \pi(\theta_a) } , \]
For any $\gamma_a > 0$ we denote the scaled posterior\footnote{In Section~\ref{sec:exact} we explain the use of scaled posteriors is required to obtain optimal regret guarantees for our bandit algorithms. } as $\mu_a^{(n)}[\gamma_a]$, whose density is proportional to: 
\begin{equation}
\exp\lrp{ \gamma_a  (n F_{n,a}(\theta_a) + \log \pi(\theta_a)) }.
\label{eq:scaled_posterior}
\end{equation}
Letting $T_a(t)$ be the number of samples received from arm $a$ after $t$ rounds, a Thompson sampling algorithm, at each round $t$, first samples the parameters of each arm $a$ from their (scaled) posterior distributions: $\theta_{a,t} \sim \mu_a^{(T_a(t))}[\gamma_a]$ and then chooses the arm for which the sample has the highest value:}
\[ A_t=\argmax_{a \in \mc{A}} \alpha_a^T\theta_{a,t}. \]
A player's objective in MAB problems is to maximize her cumulative reward over any fixed time horizon $T$. The measure of performance most commonly used in the MAB literature is known as the \emph{expected regret} $R(T)$, which corresponds to the expected difference between the accrued reward and the reward that would have been accrued had the learner selected the action with the highest mean reward during all steps $t=1, \cdots, T$.\footnote{We remark that the analysis of Thompson sampling has often been focused on a different quantity known as the Bayes regret, which is simply the expectation of $R(T)$ over the priors: $\mb{E}_\pi[R(T)]$. However, in an effort to demonstrate that Thompson sampling is an effective alternative to frequentist methods like UCB, we analyze the frequentist regret $R(T)$.} Recalling that $\bar{r}_a$ is the  mean reward for arm $a \in \mathcal{A}$, the regret is given by:
\begin{equation*}
    R(T) := \mathbb{E}\left[    \sum_{t=1}^T \bar{r}_{a^*} - \bar{r}_{A_t}   \right],
\end{equation*}
where $\bar{r}_{a^*} = \max_{a \in \mathcal{A}} \bar{r}_a $. Without loss of generality, we assume throughout this paper that the optimal arm, $a^*=\argmax_{a \in \mc{A}}\bar r_a$, is arm $1$. Further, we assume that the optimal arm is unique\footnote{We introduce this assumption merely for the purpose of simplifying our analysis.}: $\bar r_1>\bar r_a$ for $a>1$.  

\begin{algorithm}[t!]
\caption{Thompson sampling}
\label{alg:Thomspson}
\SetKwInOut{Input}{Input}
\SetKwInOut{Output}{Output}
\Input{Priors $\pi_a$ for $a \in \mathcal{A}$, posterior scaling parameter $\gamma_a$}
 Set $\mu_{a,t}=\pi_a$ for $a \in \mathcal{A}$ 
\For{$t=0,1,\cdots$}
{$\,$ Sample $\theta_{a,t} \sim \mu_a^{(T_a(t))}[\gamma_a]$\\
Choose action $A_t=\argmax_{a \in \mathcal{A}} \alpha_a^T\theta_{a,t}$.\\
Receive reward $X_{A_t}$.\\
Update (approximate) posterior distribution for arm $A_t$: $\mu_a^{(T_a(t+1))}$.
}
\end{algorithm}

 Traditional treatment of Thompson sampling algorithms often overlooks one of its most critical aspects: ensuring compatibility between the mechanism that produces samples from the posterior distributions and the algorithm's regret guarantees. This issue is usually addressed by assuming that the prior distributions and the reward distributions are conjugate pairs. Although this approach is simple and prevalent in the literature~\citep[see, e.g.,][]{intro_Thompson}, it fails to capture more complex distributional families for which this assumption may not hold. Indeed, it was recently shown in \citet{phan2019thompson} that if the samples come from distributions that approximate the posteriors with a constant error, the regret may grow at a linear rate. A more nuanced understanding of the relationship between the quality of the samples and the regret of the algorithms is, however, still lacking.
 
 In the following sections we analyze Thompson sampling in two settings. In the first, the algorithm uses samples corresponding to the \emph{true} scaled posterior distributions, $\{\mu_a^{(T_a(t))}[\gamma_a]\}_{a \in \mathcal{A}}$, at each round. In the second, Thompson sampling makes use of samples coming from two approximate sampling schemes that we propose, such that the samples can be seen as corresponding to \emph{approximations} of the scaled posteriors,  $\{\bar \mu_a^{(T_a(t))}[\gamma_a]\}_{a\in\mathcal{A}}$. We refer to the former as \emph{exact} Thompson sampling, and the latter as \emph{approximate} Thompson sampling.

  For the analysis of \emph{exact} Thompson sampling in Section~\ref{sec:exact} we derive posterior concentration theorems which characterize the rate at which the posterior distributions for the arms $\mu_a^{(n)}$ converge to delta functions centered at $\theta_a^*$ as a function of the number of $n$, the number of samples received from the arm. We then use these rates to show that Thompson sampling in this family of multi-armed bandit problems achieves the optimal finite-time regret. Further, our results demonstrate an explicit dependence on the quality of the priors and other problem-dependent constants, which improve upon prior works.
  
  In Section~\ref{sec:approx}, we propose two efficiently implementable Langevin-MCMC-based sampling schemes for which the regret of approximate Thompson sampling still achieves the optimal logarithmic regret. To do so, we derive new results for the convergence of Langevin-MCMC-based sampling schemes in the Wasserstein-$p$ distance which we then use to prove optimal regret bounds.


%
%
%
%
%
%
%
%
%
%
%
%
%
%
%
%
%
%
%
%

%

%
%
%
%
%
        
%
%
%
%

%
%
%
%
%
%
%
%
%
%
%
%
%
%
%
%
%
%
%
%
%
%

%

%
%
%
%
%
%

%

%

%
%
%
%
%
%
%
%
%
%

%

%

%

%

\section{Exact Thompson Sampling}
\label{sec:exact}

In this section we first derive posterior concentration rates on the parameters of the reward distributions when the data, the priors, and the likelihoods satisfy our assumptions. We then make use of these concentration results to give finite-time regret guarantees for exact Thompson sampling in log-concave bandits.

\subsection{Posterior Concentration Results}

Core to the analysis of Thompson sampling is understanding the behavior of the posterior distributions over the parameters of the arms' distributions as the algorithm progresses and samples from the arms are collected.  

The literature on understanding how posteriors evolve as data is collected goes back to \citet{Doob} and his proof of the asymptotic normality of posteriors. More recently, there has been a line of work \citep[see, e.g.,][]{van2008,ghosal2} that derives  rates  of convergence of posteriors in various regimes, mostly following the framework first developed in  \citet{ghosal1} for finite- and infinite-dimensional models. Such results  though quite general, do not have explicit constants or forms which make them amenable for use in analyzing bandit algorithms. Indeed, finite-time rates remain an active area of research but have been developed using information theoretic arguments \citep{shenPC}, and more recently through the analysis of stochastic differential equations \citep{mou2019diffusion}, though in both cases the assumptions, burn-in times, and lack of precise constants make them difficult to integrate with the analysis of Thompson sampling. Due to this, Thompson sampling has, for the most part, been only well understood for conjugate prior/likelihood families like beta/Bernoulli and  Gaussian/Gaussian \citep{AgrawlFurther}, or in more generality in well-behaved families such as one-dimensional exponential families with uninformative priors \citep{KordaExpFamBandits} or finitely supported prior/likelihood pairs \citep{GopalanMannor}.

To derive posterior concentration rates for parameters in $d$-dimensions and for a large class of priors and likelihoods we analyze the moments of a potential function along trajectories of a stochastic differential equation for which the posterior is the limiting distribution. Our results expand upon the recent derivation of novel contraction rates for posterior distributions presented in \citet{mou2019diffusion} to hold for a finite number of samples and may be of independent interest. We make use of these concentration results to show that Thompson sampling with such priors and likelihoods results in order-optimal regret guarantees.

To begin, we note that classic results \citep{SDE_book} guarantee that, as $t \rightarrow \infty$  the  distribution $P_t$ of $\theta_t$ which evolves according to:
\begin{align}
    d\theta_t= \frac{1}{2}\nabla_\theta F_{n, a}(\theta_t)dt+\frac{1}{2n}\nabla_\theta \log\pi_a(\theta_t))dt +\frac{1}{\sqrt{n \gamma_a}}dB_t,
    \label{eq:sde}
\end{align}
is given by:
\[ \lim_{t \rar \infty} P_t(\theta | X_1,...,X_n) \propto \exp(-\gamma_a\left(n F_{n ,a}(\theta)+\log\ \pi_a(\theta)\right)), \]
almost surely. Comparing with Eq.~\eqref{eq:scaled_posterior}, this limiting distribution is the scaled posterior distribution $\mu_a^{(n)}[\gamma_a]$. Thus, by analyzing the limiting properties of $\theta_t$ as it evolves according to the stochastic differential equation, we can derive properties of the scaled posterior distribution. 

To do so, we first show that with high probability the gradient of $F_{n,a}(\theta^*)$ concentrates around zero (given the data $X_1,...,X_n$). More precisely we show in Appendix \ref{section::posterior_concentration_appendix} using well known results on the concentration of Lipschitz functions of strongly log-concave random variables that $\nabla_\theta F_{a,n}(\theta_a^*)$ has sub-Gaussian tails:
\begin{proposition}
\label{prop:grad_conc_main}
    The random variable $\|\nabla_\theta F_{a,n}(\theta_a^*) \|$ is $L_a \sqrt{\frac{d_a}{n\nu_a}}$-sub-Gaussian: 
\end{proposition}
 Given this proposition, we then analyze how the potential function,
\[ V(\theta_t)=\frac{1}{2}e^{c t}\|\theta_t-\theta^\ast\|^2_2, \]
 evolves along trajectories of the stochastic differential equation \eqref{eq:sde}, where $c>0$. By bounding the supremum of $V(\theta_t)$, we construct bounds on the higher moments of the random variable $\|\theta_a-\theta_a^*\|$ where $\theta_a \sim \mu_a^{(n)}[\gamma_a]$. These moment bounds translate directly into the posterior concentration bound of $\theta_a \sim \mu_a^{(n)}[\gamma_a]$ around $\theta^*$ presented in the following theorem (the proof of which is deferred to Appendix~\ref{section::posterior_concentration_appendix}).
\begin{theorem}
\label{thm:theorem1}
Suppose that Assumptions 1-3 hold, then for $\delta \in (0,1)$:
\[ \mathbb{P}_{\theta \sim \mu^{(n)}_{a}[\gamma_a]} \left(\|\theta_a-\theta_a^\ast\|_2 > \sqrt{\frac{2e}{m_an}\left(\frac{d_a}{\gamma_a} +\log B_a + \left(\frac{32}{\gamma_a} + 8 d_a \kappa_a^2 \right) \log\left(1/\delta\right)\right)}\right)<\delta. \]
where $B_a=\max_{\theta \in \mathbb{R}^d} \frac{\pi_a(\theta)}{ \pi_a(\theta_a^*)}. $
\end{theorem}
Theorem~\ref{thm:theorem1} guarantees that the scaled posterior distribution over the parameters of the arms concentrate at rate $\frac{1}{\sqrt{n}}$, where $n$ is the number of times the arm has been pulled. 

We remark that this posterior concentration result has a number of desirable properties. Through the presence of $B_a$, it reflects an explicit dependence on the quality of the prior. In particular, $\log B_a=0$ if the prior is properly centered such that its mode is at $\theta^*$ or if the prior is uninformative or nearly flat everywhere. We further remark that the concentration result also scales with the variance of $\theta_a$ which is on the order of $d_a/(\gamma_a m_an)$. Lastly, we remark that this concentration result holds for any $n>0$ and the constants are explicitly defined in terms of the smoothness and structural assumptions on the priors, likelihoods, and reward distributions. This makes it more amenable for use in constructing regret guarantees, since we do not have to wait for a burn-in period for the result to hold, as in \citet{shenPC} and \citet{mou2019diffusion}. Moreover, the dependence on the dimension  of the parameter space and constants are explicit.

\subsection{Exact Regret for Thompson Sampling}

We now show that, under our assumptions, Thompson sampling with samples from the scaled posterior enjoys optimal finite-time regret guarantees. To provide these results we proceed as is common in regret proofs for multi-armed bandits by upper bounding $T_a(T)$, the number of times a sub-optimal arm $a\in \mathcal{A}$ is pulled up to time $T$. Without loss of generality we assume throughout this section that arm $1$ is the optimal arm, and define the filtration associated with a run of the algorithm as $\mc{F}_{t}=\{A_1,X_1,A_2,X_2,...,A_t,X_t\}$. 

To upper bound the expected number of times a sub-optimal arm is pulled up to time $T$, we first define the low-probability event that the mean calculated from the value of $\theta_{a,t}$ sampled from the posterior at time $t\le T$, $r_{a,t}(T_a(t))$, is greater than $\bar r_1-\epsilon$ (recall that $\bar r_1$ is the optimal arm's mean): $E_a(t)=\{ r_{a,t}(T_a(t))\ge \bar r_1-\epsilon \}$ for some $\epsilon>0$. Given these events, we proceed to decompose the expected number of pulls of a sub-optimal arm $a \in \mathcal{A}$ as:
\begin{align}
    \mb{E}[T_a(T)]&=\mb{E}\left[\sum_{t=1}^T\mb{I}(A_t=a)\right] =\underbrace{\mb{E}\left [\sum_{t=1}^T\mb{I}(A_t=a, E^c_a(t))\right]}_{\mathrm{I}}  +\underbrace{\mb{E}\left [\sum_{t=1}^T\mb{I}(A_t=a, E_a(t))\right]}_{\mathrm{II}}.
    \label{eq:regret_decomp1}
\end{align}
These two terms satisfy the following standard bounds (see e.g. \cite{TorCsabaBandits}):
\begin{lemma}[Bounding I and II]\label{lemma:termI_andII_main}
    For a sub-optimal arm $a\in \mathcal{A}$, we have that:
\begin{align}
    \mathrm{I} & \le \mb{E}\left [\sum_{s=\ell}^{T-1} \frac{1}{p_{1,s}}-1\right];  \label{equation::term_I_explicit}\\
    \mathrm{II} &\le  1+\mb{E}\left [\sum_{s=1}^T\mb{I}\left (p_{a,s}>\frac{1}{T}\right )\right] \label{equation::term_II_explicit}, 
\end{align}
    where $p_{a,s}=\mb{P}(r_{a,t}(s)>\bar r_1-\epsilon |\mc{F}_{t-1})$, for some $\epsilon>0$. 
\end{lemma}
The proof of these results are standard for the regret of Thompson sampling and can be found in Appendix \ref{section::appendix_regret_proofs}, Lemmas~\ref{lemma:termI} and~\ref{lemma:termII_exact} for completeness. 

Given Lemma~\ref{lemma:termI_andII_main}, we see that to bound the regret of Thompson Sampling it is sufficient to bound the two terms $\mathrm{I}$ and $\mathrm{II}$. 

To bound term $\mathrm{I}$, we first show that for all times $t=1,...,T$, and number of samples collected from arm $1$, the probability $p_{1,n}=\mb{P}(r_{1,t}(n)>\bar r_1-\epsilon |\mc{F}_{t-1})$ is lower bounded by a constant depending only on the quality of the prior for arm 1.  This guarantees the posterior for the optimal arm is approximately optimistic with at least a constant probability, and requires a proper choice of $\gamma_1$. We note the unscaled posterior provides the correct concentration with respect to the number of data samples $T_a(t)$, when $T_a(t)$ is large. This is sufficient to upper bound the trailing terms of $\mathrm{I}$, that is, summands in Equation~\ref{equation::term_I_explicit} for large $s$. Unfortunately concentration is not enough to bound term $\mathrm{I}$, since the early summands of Equation~\ref{equation::term_I_explicit} corresponding to small values of $s$ could be extremely large. Intuitively, the random variable $r_{1,t}(s)$ can be thought of as centered around the posterior mean of arm $1$. Though this is close to the true value of $\bar{r}_1$ with high probability, when $T_1(t)$ is small, concentration alone does not preclude the possibility that the posterior mean underestimates $\bar{r}_1$ by a value of at least $\epsilon$. In order to ensure $p_{1,s}$ is large enough in these cases, we require $r_{1,t}(s)$ to have sufficient variance to overcome this potential underestimation bias. We show that a scaled posterior $\mu_a^{(T_a(t))}[\gamma_a]$ with $\gamma_a=\lrp{8d_a \kappa_a^3}^{-1}$ in Algorithm~\ref{alg:Thomspson} ensures $r_{1,t}(s)$ has enough variance.

\begin{lemma}
\label{lemma:lemma2}
Suppose the likelihood and reward distributions satisfy Assumptions 1-3, then for all $n=1,...,T$ and 
$\gamma_1 = \frac{1}{8 d_1 \kappa_1^3}$:
\[ \mb{E}\left[\frac{1}{p_{1,n}}\right]\le  C\sqrt{B_1 \kappa_1}, \]
where C is a universal constant independent of the problem dependent parameters.
\end{lemma}

\begin{remark}
 We find that a proper choice of $\gamma_1$ is required  to ensure that that the posterior on the optimal arm has a large enough variance to guarantee a degree of optimism despite the randomness in its mean. Scaling up the posterior was also noted to be necessary in linear bandits (see e.g.\cite{Thompson_bound_Linearl,LinearRevisited}) to ensure optimal regret. In practice, since we do not \emph{a priori} know which is the optimal arm, we must scale the posterior of each arm by a parameter $\gamma_a$.\end{remark}

The quantity $B_1=\frac{\max_\theta \pi_1(\theta)}{\pi_1(\theta_1^*)}$ captures a worst case dependence on the quality of the prior for the optimal arm, and can be seen as the expected number of samples from the prior until an optimistic sample is observed.

By using this upper bound in conjunction with the posterior concentration result derived in Theorem~\ref{thm:theorem1}, we can further bound $\mathrm{I}$ and $\mathrm{II}$. We note that in contrast with simple subgaussian concentration bounds, our posterior concentration rates have a bias term decreasing at a rate of $1/\sqrt{\text{number of samples}}$. In our analysis we carefully track and control the effects of this bias term ensuring it does not compromise our $\log$-regret guarantees. Indeed, using the posterior concentration in the bounds from Lemma \ref{lemma:termI_andII_main} we show that, for 
$\gamma_a=\frac{1}{8d_a \kappa_a^3}$ there are two universal constants $C_1,C_2>0$ independent of the problem dependent parameters such that:
     \begin{align*}
     \mathrm{I} &\le  C_1\sqrt{\kappa_1 B_1} \left\lceil \frac{A_1^2}{m_1\Delta_a^2}(D_1+\sigma_1) \right\rceil +1;\\
     \mathrm{II} &\le \frac{C_2A_a^2}{m_a\Delta_a^2}(D_a+\sigma_a\log(T)),
\end{align*}
where for $a \in \mc{A}$,  $D_a$ and $\sigma_a$ are given by:
\[ 
D_a= \log B_a + d_a^2\kappa_a^3,  
\qquad \sigma_a = d_a \kappa_a^3 + d_a \kappa_a^2.
\] 
Finally, combining all these observations we obtain the following regret guarantee:
\begin{theorem}[Regret of Exact Thompson Sampling]
 When the likelihood and true reward distributions satisfy Assumptions 1-3 and $\gamma_a=\frac{1}{8d_a \kappa_a^3}$ we have that the expected regret after $T>0$ rounds of Thompson sampling with exact sampling satisfies:
    \begin{align*}
        \mb{E}[R(T)]&\le \sum_{a>1} \frac{ C A_a^2}{m_a\Delta_a}\left(\log B_a+d_a^2\kappa_a^3+d_a\kappa_a^3 \log(T)\right)+\sqrt{\kappa_1 B_1}\frac{C A_1^2}{m_1\Delta_a} \left(1+\log B_1+d_1^2\kappa_1^3\right) +\Delta_a
    \end{align*}
    Where $C$ is a universal constant independent of problem-dependent parameters.
    \label{thm:exact_regret}
\end{theorem}
The proof of the theorem is deferred to Appendix~\ref{section::appendix_regret_proofs}, where we also provide the exact value of the universal constant $C$. We remark that this regret bound gives an $O\left(\frac{\log{(T)}}{\Delta}\right) $ asymptotic regret guarantee, but holds for any $T>0$. This further highlights that Thompson sampling is a competitive alternative to UCB algorithms since it achieves the optimal problem-dependent rate for multi-armed bandit algorithms first presented in \citet{LaiRobbins}. 

Our bound also has explicit dependencies on the dimension of the parameter space of the likelihood distributions for each arm, as well as on the quality of the priors through the presence of $B_a$ and $B_1$. We note that the dependence on the priors does not distinguish between ``good'' and ``bad'' priors. Indeed, the parameter $B_a\ge1$ is worst case, and does not capture the potential advantages of good priors in Thompson sampling, that we observe in our numerical experiments in Section~\ref{sec:numex_main}. Further, we remark that our bound exhibits a worse dependence on the prior for the optimal arm ($O(\sqrt{B_1}\log(B_1))$) than for sub-optimal arms ($O(\log(B_a))$). This is also a worst case dependence which captures the expected number of samples from the prior until an approximately optimistic sample is observed, which we believe to be unavoidable. 

Finally, we note that our regret bound scales with the variances of the reward and likelihood families since  $\frac{1}{m_a}$ and $\frac{1}{\nu_a}$ reflect the variance of the likelihoods in $\theta$ and the rewards $X_a$ respectively. 

Thus, through the use of the posterior contraction rates we are able to get finite-time regret bounds for Thompson sampling with multi-dimensional log-concave families and arbitrary log-concave priors. This generalizes the result of \citet{KordaExpFamBandits} to a more general class or priors and higher dimensional parametric families.

\section{Approximate Thompson Sampling}
\label{sec:approx}

In this section we present two approximate sampling schemes for generating samples from approximations of the (scaled) posteriors at each round. For both, we give the values of the hyperparameters and computation time needed to guarantee an approximation error which does \emph{not} result in a drastic change in the regret of the Thompson sampling algorithm. 

Before doing so, however, we first present a simple counterexample to illustrate that in the worst case, Thompson sampling with approximate samples incurs  an irreducible regret dependent on the error between the posterior and the approximation to the posterior. In particular, by allowing the approximation error to decrease over time, we extract a  relationship between the order of the regret and the level of approximation.
\begin{example}
\label{ex:counterexx}
Consider a Gaussian bandit instance of two arms $\mathcal{A} = \{1,2\}$ having mean rewards $\bar{r}_1$ and $\bar{r}_2$ and known unit variances. Further assume that the unknown parameters are the means of the distributions such that $\theta^*_a = \bar{r}_a$, and consider the case where the learner makes use of a zero-mean, unit-variance Gaussian prior over $\theta_a$ for $a=1,2$. Under these assumptions, after obtaining samples $X_{a,1}, \cdots, X_{a,n}$, the posterior updates satisfy the following well-known formula:
\begin{equation*}
    P_{a,n}(\theta_a) \propto \mathcal{N}\left(\frac{n}{n+1} ,\frac{1}{n+1}\right).
\end{equation*}
Let $\bar{r}_1 = 1$ and $\bar{r}_2 = 0$ such that arm $1$ is optimal. We now show there exists an approximate posterior $\tilde{P}_{a,t}$ of arm $2$, satisfying $\mathrm{TV}( \tilde{P}_{2,t}, P_{2,t}) \leq n^{-\alpha}$ and such that if samples from $P_{1,t}$ and $\tilde{P}_{2,t}$ were to be used by a Thompson sampling algorithm, its regret would satisfy $R(T) = \Omega(T^{1-\alpha})$.

We substantiate this claim by a simple construction. Let $\tilde{P}_{a,t}$ be $(1-n^{-\alpha})P_{a,t} + n^{-\alpha}\delta_2$, where $\delta_2$ denotes a delta mass centered at $2$. $\tilde{P}_{a,t}$ is a mixture distribution between the true posterior and a point mass. 

Clearly, for all $t \geq C$ for some universal constant $C$, with probability at least $n^{-\alpha}$ the posterior sample from arm $2$ will be larger than the sample from arm $1$. Since $t>n$, $t^{-\alpha}<n^{-\alpha} $ for $\alpha>0$ and since the suboptimality gap equals $1$, we conclude $R(T)=\Omega(\sum_{t=1}^T t^{-\alpha})$. Thus, to incur logarithmic regret, one needs $TV(\tilde P_{2,t},P_{2,t}) = \Omega(\frac{1}{n})$.
\end{example}

\begin{algorithm}[H]
\caption{(Stochastic Gradient) Langevin Algorithm for Arm $a$}
\label{alg:Approximate_sampling}
\SetKwInOut{Input}{Input}
\SetKwInOut{Output}{Output}
\Input{Data $\{x_{a,1}, \cdots, x_{a,n}\}$; \\
MCMC sample $\theta_{a, N h^{(n-1)}}$ from last round}
 Set $\theta_{0}=\theta_{a,t-1}$ for $a \in \mathcal{A}$\\
\For{$i=0,1,\cdots N$}
{
$\,$ Uniformly subsample $\mathcal{S}\subseteq \{x_{a,1}, \cdots, x_{a,n}\}$. \\
Compute 
$
\nabla \hpot(\theta_{i h^{(n)}}) = - \frac{n}{|\mathcal{S}|} \sum_{x_k\in\mathcal{S} } \nabla \log p_a(x_k|\theta_{i h^{(n)}}) - \nabla \log \pi_a(\theta_{i h^{(n)}}).
$ \\
Sample $\theta_{(i+1) h^{(n)}} \sim \mathcal{N}\lrp{ \theta_{i h^{(n)}} - h^{(n)} \nabla \hpot(\theta_{i h^{(n)}}) , 2 h^{(n)} \mI }$.
}
%
\Output{$\theta_{a,N h^{(n)}}=\theta_{N h^{(n)}}$ and $\theta_{a,t} \sim \mc{N}\left( \theta_{N h^{(n)}} , \frac{1}{{nL_a\gamma_a} }I\right)$}
\end{algorithm}

Example~\ref{ex:counterexx} builds on the insights in \citet{phan2019thompson}, who showed that constant approximation error can incur linear regret, which highlights the fact that to achieve logarithmic regret the total variation distance between  the approximation of the posterior  $\bar \mu_a^{(n)}[\gamma_a]$ and the true posterior $\mu_a^{(n)}$ must decrease as samples are collected. In particular it illustrates that the rate at which the approximation error decreases is directly linked to the resulting regret bound.

Given this result, we first propose an unadjusted Langevin algorithm (ULA)~\citep{durmus2017}, which generates samples from an approximate posterior which monotonically approaches the true posterior as data is collected and provably maintains the regret guarantee of exact Thompson sampling. Important to this effort, we demonstrate that the number of steps inside the ULA procedure does not scale with the time horizon, though the number of gradient evaluations scale with the number of times an arm has been pulled.
To over this issue arising from full gradient evaluation, we propose a stochastic gradient Langevin dynamics (SGLD)~\citep{SGLD} variant of ULA which has appealing computational benefits: under slightly stronger assumptions, SGLD takes a constant number of iterations as well as a constant number of data samples in the stochastic gradient estimate while maintaining the order-optimal regret of the exact Thompson sampling algorithm.

\subsection{Convergence of (Stochastic Gradient) Langevin Algorithms}
As described in  Algorithm~\ref{alg:Approximate_sampling}, in each round $t$  we run the (stochastic gradient) Langevin algorithm for $N$  steps to generate a sample of desirable quality for each arm. In particular, we first run a Langevin MCMC algorithm to generate a sample from an approximation to the unscaled posterior. To achieve the scaling with $\gamma_a$ that we require for the analysis of the regret, we add zero-mean Gaussian noise with variance $\frac{1}{\gamma_a L_a n}$ to this sample. The distribution of the resulting sample has the same characteristics as those from the scaled posterior analyzed in Sec.~\ref{sec:exact}.

Given Assumptions~\ref{assumption:Uniform_strong_concavity} and~\ref{assumption:prior_assumption},
we prove (in Theorem~\ref{theorem_ULA} in the Appendix) that running ULA with exact gradients provides appealing convergence properties. In particular, for a number of iterations independent of the number of rounds $t$ or the number of samples from an arm, $n=T_a(t)$, ULA converges to an accuracy in Wasserstein-$p$ distance which maintains the logarithmic regret of the exact algorithm  (for more information on such metrics see \citet{Villani_optimal_transport}). 
We note parenthetically that working with the Wasserstein-$p$ distance provides us with a tighter MCMC convergence analysis (than with the total variation distance used in Example~\ref{ex:counterexx}) that helps in conjunction with the regret bounds.
The proofs of the ULA and SGLD convergence require a uniform strong log-concavity and Lipschitz smoothness condition of the family $p_a(X|\theta_a)$ over the parameter $\theta_a$, a strengthening of Assumption~\ref{assumption:family_assumption}.
\begin{assump}{1-Uniform}[Assumption on the family $p_a(X|\theta_a)$: strengthened for approximate sampling]
\label{assumption:Uniform_strong_concavity}
Assume that $\log p_a(x|\theta_a)$ is $L_a$-smooth and $m_a$-strongly concave over the parameter $\theta_a$:
\begin{multline*}
    -\log p_a(x|\theta_a') -\nabla_{\theta} \log p_a(x|\theta_a')^\top \left( \theta_a - \theta_a' \right) + \frac{m_a}{2}\| \theta_a - \theta_a' \|^2
    \leq -\log p_a(x|\theta_a) \\ 
    \leq -\log p_a(x|\theta_a') -\nabla_{\theta} \log p_a(x|\theta_a')^\top \left( \theta_a - \theta_a' \right) + \frac{L_a}{2}\| \theta_a - \theta_a'\|^2, 
    \quad \forall \theta_a, \theta_a' \in \mathbb{R}^{d_a}, x \in \mb{R}.
\end{multline*}
\end{assump}  
Although the number of iterations required for ULA to converge is constant with respect to the time horizon $t$, the number of gradient computations over the likelihood function within each iteration is $T_a(t)$. To tackle this issue, we sub-sample the data at each iteration and use a stochastic gradient MCMC method~\citep{completesample}. To be able to get convergence guarantees despite the larger variance this method incurs, we make a slightly stronger Lipschitz smoothness assumption on the parametric family of likelihoods.
\begin{assumption}[Joint Lipschitz smoothness of the family $\log p_a(X|\theta_a)$: for SGLD]
\label{assumption:sg_lip}
Assume a joint Lipschitz smoothness condition, which strengthens Assumptions~\ref{assumption:Uniform_strong_concavity} and~\ref{assumption:true_reward_assumption}
to impose the Lipschitz smoothness on the entire bivariate function $\log p_a(x;\theta)$:\footnote{For simplicity of notation, we let Lipschitz constants $L_a^*=L_a$ in the main paper.}
\begin{align*}
\lrn{\nabla_\theta \log p_a(x|\theta_a) - \nabla_\theta \log p_a(x'|\theta_a)}
\leq L_a\lrn{\theta_a - \theta_a'} + L_a^*\lrn{x-x'},
\quad \forall \theta_a, \theta_a' \in \mathbb{R}^{d_a}, x, x' \in \mb{R}.
\end{align*}
\end{assumption}

Under this stronger assumption, we prove the fast convergence of the SGLD method in the following Theorem~\ref{theorem3}. Specifically, we demonstrate that for a suitable choice of stepsize $h^{(n)}$, number of iterations $N$, and size of the minibatch $k=|\mathcal{S}|$, samples generated by Algorithm \ref{alg:Approximate_sampling} are distributed sufficiently close to the true posterior to ensure the optimal regret guarantee.
By examining the number of iterations $N$ and size of the minibatch $k$, we confirm that the algorithmic and sample complexity of our method do not grow with the number of rounds $t$, as advertised.
\begin{theorem}[SGLD Convergence]
\label{theorem3}
Assume that the family $\log p_a(x;\theta)$, prior distributions, and that the true reward distributions satisfy Assumptions~\ref{assumption:Uniform_strong_concavity} through~\ref{assumption:sg_lip}.
If we take 
the batch size $k=\mathcal{O} \lrp{ \kappa_a^2 }$, step size $h^{(n)} =
\mathcal{O}\lrp{ \frac{1}{n} \frac{1}{\kappa_a L_a}}$ and number of steps 
$N = \mathcal{O}\lrp{ \kappa_a^2 }$ in the SGLD algorithm,
then for $\delta_1 \in (0,1)$, with probability at least $1-\delta_1$ with respect to $X_{a,1}, ... X_{a,n}$, we have convergence of the SGLD algorithm in the Wasserstein-$p$ distance.
In particular, between the $n$-th and the $(n+1)$-th pull to arm $a$, samples $\theta_{a,t}$ approximately follow the posterior $\mu_a^{(n)}$:
\begin{align*}
    W_p\lrp{\hmu_a^{(n)}, \mu_a^{(n)}} 
    \leq \sqrt{\frac{8}{n m_a}} \left( d_a + \log B_a +\left(32 + 8 d_a \kappa_a^2 \right)p\right)^{\frac{1}{2}},
\end{align*}
where $\hmu_a^{(n)}$ is the probability measure associated with any of the sample(s) $\theta_{a,Nh_a^{(n)}}$ between the $n$-th and the $(n+1)$-th pull of arm $a$.
\end{theorem}
We remark that we are able to keep the number of iterations, $N$, for both algorithms constant by initializing the current round of the approximate sampling algorithm using the output of the last round of the Langevin MCMC algorithm. If we initialized the algorithm independently from the prior, we would need $O(\log{T_a(t)})$ iterations to achieve this result, which would in turn yield a Thompson sampling algorithm for which the computational complexity grows with the time horizon.  We note that this warm-starting complicates the regret proof for the approximate Thompson sampling algorithms since the samples used by Thompson sampling are no longer independent. 

By scrutinizing the stepsize $h^{(n)}$ and the accuracy level of the sample distribution $W_p\lrp{\hmu_a^{(n)}, \mu_a^{(n)}}$, we note that we are taking smaller steps to get increasingly accurate MCMC samples as more data are being collected. This is due to the need of decreasing the error incurred by discretizing the continuous Langevin dynamics and stochastically estimating the gradient of the log posterior. 
However, the number of iterations and subsampled gradients are not increasing since the concentration of the posterior provides us with stronger contraction of the continuous Langevin dynamics and requires less work because $\mu_a^{(n)}$ and $\mu_a^{(n+1)}$ are closer.

We restate Theorem~\ref{theorem3} and give explicit values of the hyper-parameters in Theorem~\ref{theorem_SG} in the appendix, but remark that the proof of this theorem is novel in the MCMC literature.
It builds upon and strengthens~\cite{Moulines_ULA} by taking into account the discretization and stochastic gradient error to achieve strong convergence guarantees in the Wasserstein-$p$ distance up to any finite order $p$.
Other related works on the convergence of ULA can provide upper bounds in the Wassertein distances up to the second order (i.e., for $p\leq2$)~\citep[see, e.g.,][]{Dalalyan_user_friendly,Xiang_overdamped,MCMC_nonconvex,wibisono_iso}.
This bound in the Wasserstein-$p$ distance for arbitrarily large $p$ is necessary in guaranteeing the following Lemma~\ref{lemma:approx_dist}, a similar concentration result as in Theorem~\ref{thm:theorem1} for the approximate samples $\theta_{a,t} \sim \bar \mu_a^{(n)}[\gamma_a]$. 


\begin{lemma}
Suppose that Assumptions~\ref{assumption:Uniform_strong_concavity} through~\ref{assumption:sg_lip} hold, then for $\delta \in (0,1)$, the sample $\theta_{a,t}$ resulting from running the (stochastic gradient) ULA with $N$ steps, a step size of $h^{(n)}$, and a batch-size $k$ as defined in Theorem~\ref{theorem3} satisfies:
\[ 
\mathbb{P}_{\theta_{a,t} \sim \bar\mu^{(n)}_{a}[\gamma_a]} \left(\|\theta_{a,t}-\theta_a^\ast\|_2 > \sqrt{\frac{36e}{m_a n} \left( d_a+\log B_a +2\left(\sigma_a+\frac{d_a }{18\kappa_a \gamma_a }\right)\log{1/\delta} \right)} \right)<\delta. 
\]
where $\sigma_a=16+4d_a \kappa_a^2$.
\label{lemma:approx_dist}
\end{lemma}
%

\subsection{Thompson Sampling Regret with (Stochastic Gradient) Langevin Algorithms}
\label{sec:TS_ULA}

Given that the concentration results of the samples from ULA and SGLD have the same form as that of exact Thompson sampling, we now show that approximate Thompson sampling achieves the same \emph{finite}-time optimal regret guarantees (up to constant factors) as the exact Thompson sampling algorithm. To show this, we require an analgous result to Lemma~\ref{lemma:lemma2} on the anti-concentration properties of the approximations to the scaled posteriors:

\begin{lemma}
\label{lemma:approx_anti_conc_main}
Suppose the likelihood and true reward distributions satisfy Assumptions 1-4: then if $\gamma_1=O\left(\frac{1}{d_1\kappa_1^3}\right)$, for all $n=1,...,T$ all samples from the the (stochastic gradient) ULA method with the hyperparameters and runtime as described in Theorem~\ref{theorem3} satisfy: 

\[ \mb{E}\left[ \frac{1}{p_{1,n}} \right] \le C\sqrt{B_1}\]
where $C$ is a universal constant independent of problem-depedent parameters.
\end{lemma}

The proof of Lemma~\ref{lemma:approx_anti_conc_main} is similar to that of ~\ref{lemma:lemma2}, but we are able to save a factor of $\sqrt{\kappa_1}$ due to the fact that the last step of the approximate sampling scheme samples $\theta_{a,t}$ from a a Gaussian distribution as opposed to a strongly-log concave distribution which we must approximate with a Gaussian.

Given this lemma and our concentration results presented in the previous section, the proof of logarithmic regret is essentially the same as that of the regret for exact Thompson sampling. However,  more care has to be taken to  deal with the fact that the samples from the approximate posteriors are no longer independent due to the fact that we warm-start our proposed sampling algorithms using previous samples. We cope with this issue by constructing concentration rates (of a similar form as in Lemma~\ref{lemma:approx_dist}) on the distributions of the samples given the initial sample is sufficiently well behaved (see Lemmas~\ref{lemma:W_p_delta_ULA} and~\ref{lemma:W_p_delta_SGLD}). We then show that this happens with sufficiently high probability to maintain similar upper bounds on terms $I$ and $II$ from Lemma~\ref{lemma:termI_andII_main}  in Lemma~\ref{lemma:approx_Regret_decomposition}, which in turn allows us to prove the following Theorem in Appendix~\ref{sec:approx_regret_proof_appendix}.

\begin{theorem}[Regret of Thompson sampling with a (stochastic gradient) Langevin algorithm]
\label{thm:approx_regret}
When the likelihood and true reward distributions satisfy Assumptions 1-4: we have that the expected regret after $T>0$ rounds of Thompson sampling with the (stochastic gradient) ULA method with the hyper-parameters and runtime as described in Theorem~\ref{theorem3} satisfies:
\begin{align*}
        \mb{E}[R(T)]\le& \, \sum_{a>1} \frac{C A_a^2}{ m_a\Delta_a}\left( \log B_a+d_a+ d_a^2\kappa_a^2 \log T \right)+\sqrt{B_1}\frac{C  A_1^2}{m_1\Delta_a}\left( 1+\log B_1+d_1^2\kappa_1^2 +d_1\kappa_1^2\log{T}\right) +3\Delta_a.
\end{align*}       
    where $C$ is a universal constant that is independent of problem dependent parameter and the scale parameter $\gamma_a=O\left(\frac{1}{d_a \kappa_a^3}\right)$.
\end{theorem}

We note that Theorem~\ref{theorem3} allows for SGLD to be implemented with a constant number of steps per iteration and a constant batch-size with only the step-size decreasing linearly with the number of samples. Combining this with our regret guarantee shows that an anytime algorithm for Thompson sampling with approximate samples can indeed achieve logarithmic regret. 

Further, we remark that this bound exhibits a \emph{worse} dependence on the quality of the prior on the optimal arm than in the exact sampling regime. In particular, we pay a  $d_1^2\sqrt{B_1}\log T$ in this regret bound as opposed to $d_1^2\sqrt{B_1}$. Our regret bound in the approximate sampling regime does exhibit a slightly better dependence on the condition number of the family. This, we believe, is an artifact of our analysis and is due to the fact that a lower bound on the exact posterior was needed to invoke Gaussian anti-concentration results which were not needed in the approximate sampling regime due to the design of the proposed sampling algorithm.

\section{Numerical Experiments}
\label{sec:numex_main}

We empirically corroborate our theoretical results with numerical experiments of approximate Thompson sampling in log-concave multi-armed bandit instances. 
We benchmark against both UCB and exact Thompson sampling across three different multi-armed bandit instances, where in the first instance, the priors reflect correct ordering of the mean rewards for all arms; in the second instance, the priors are agnostic of the ordering; in the third instance, the priors reflects the complete opposite ordering.
See Appendix~\ref{sec:numex} for details of the experimental settings.

As suggested in our theoretical analysis in Section~\ref{sec:approx}, we use a constant number of steps for both ULA and SGLD (with constant number of data points in the stochastic gradient evaluation) to generate samples from the approximate posteriors.
The regret of the three algorithms averaged across 100 runs is displayed in Figure~\ref{fig:numex}, where we see approximate Thompson sampling with samples generated by ULA and SGLD perform competitively against both exact Thompson sampling and UCB across all three instances.

\begin{figure}[h!]
\centering
         \includegraphics[width=\textwidth]{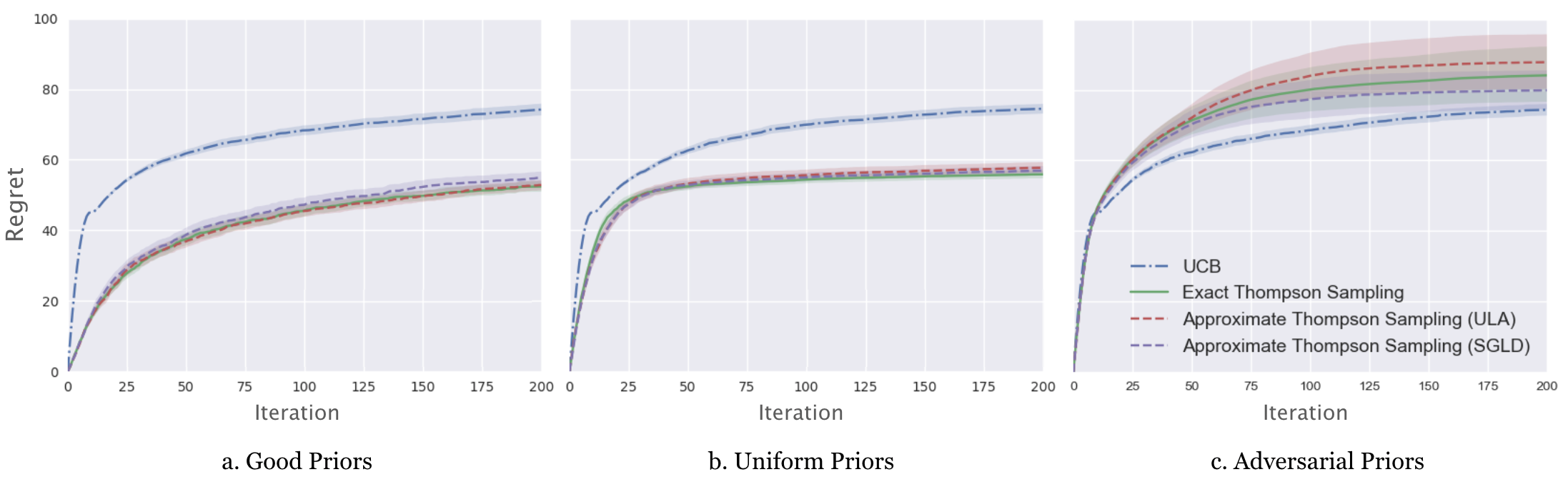}
         \caption{Performance of exact and approximate Thompson sampling vs UCB on Gaussian bandits with (a) ``good priors'' (priors reflecting the correct ordering of the arms' means), (b) the same priors on all the arms' means, and (c) ``bad priors'' (priors reflecting the exact opposite ordering of the arms' means). The shaded regions represent the 95\% confidence interval around the mean regret across 100 runs of the algorithm.}
     \label{fig:numex}
\end{figure}

We observe significant performance gains from the (approximate) Thompson sampling approach over the deterministic UCB algorithm when the priors are suggestive or even non-informative of the appealing arms.
When the priors are adversarial to the algorithm, the UCB algorithm outperforms the Thompson sampling approach as expected. (This case corresponds to the constant $B_a$ in the Theorems~\ref{thm:exact_regret} and~\ref{thm:approx_regret} being large).
Also as the theory predicts, we observe little difference between the exact and the approximate Thompson sampling methods in terms of the regret.
If we zoom in and scrutinize further, we can see that SGLD slightly outperforms the exact Thompson sampling method in the adversarial prior case. This might be due to the added stochasticity from the approximate sampling techniques, which improves the robustness against bad priors.

\section{Conclusions}

Although Thompson sampling has been used successfully in real-world problems for decades and has been shown to have  appealing theoretical properties there remains a lack of understanding of how approximate sampling affects its regret guarantees.

In this work we derived new posterior contraction rates for log-concave likelihood families with arbitrary log-concave priors which capture key dependencies between the posterior distributions and various problem-dependent parameters such as the prior quality and the parameter dimension. We then used these rates to show that exact Thompson sampling in MAB problems where the reward distributions are log-concave achieves the optimal finite-time regret guarantee for MAB bandit problems from \citet{LaiRobbins}. 
As a direction for future work, we note that although our regret bound demonstrates a dependence on the quality of the prior, it still is unable to capture the potential advantages of good priors. 

We then demonstrated that Thompson sampling using samples generated from ULA, and under slightly stronger assumptions, SGLD, could still achieve the optimal regret guarantee with constant algorithmic as well as sample complexity in the stochastic gradient estimate. Thus, by designing approximate sampling algorithms specifically for use with Thompson sampling, we were able to construct a computationally tractable anytime Thompson sampling algorithm from approximate samples with end-to-end guarantees of logarithmic regret.  

\bibliographystyle{plainnat}
\bibliography{arxivbib}

\appendix

\onecolumn

\section{Notation}

Before presenting our proofs, we first include a table summarizing our notation.

\begin{center}
 \begin{tabular}{|c | c |} 
 \hline
 Symbol & Meaning\\ [0.5ex] 
 \hline\hline
 $\mc{A}$ & set of arms in bandit environment \\
 \hline
 $K$ & number of arms in the bandit environment $|\mc{A}|$\\
 \hline
 $T$ &  Time horizon\\
 \hline
 $A_t$ & arm pulled at time $t$ by the algorithm $A_t \in \mc{A}$\\
 \hline
 $T_a(t)$ & number of times arm $a$ has been pulled by time $t$\\
 \hline
  $X_{A_t}$ & reward from choosing arm $A_t$ at time $t$\\
 \hline
 $\theta_a$ & parameters of likelihood functions such that, $
  \theta_a \in \mb{R}^{d_a}$ \\
 \hline 
 $d_a$ & dimension of parameter space for arm $a$\\
 \hline 
  $p_a(x|\theta_a)$ & parametric family of reward distributions for arm $a$\\
 \hline
 $\pi_a(\theta_a)$ & prior distribution over the parameters for arm $a$\\
  \hline
 \multirow{2}{*}{$\mu_a^{(n)}$} & probability measure associated with the posterior over the parameters of arm $a$\\
 & after $n$ samples from arm $a$\\
 \hline
 \multirow{2}{*}{$\mu_a^{(n)}[\gamma_a]$} & probability measure associated with the (scaled) posterior over the parameters of arm $a$\\
 & after $n$ samples from arm $a$\\
  \hline
 \multirow{2}{*}{$\widehat \mu_a^{(n)}$} & probability measure resulting from running the Langevin MCMC algorithm\\
 & described in Algorithm~\ref{alg:Approximate_sampling} which approximates $\mu_a^{(n)}$\\
 \hline
 \multirow{2}{*}{$\bar \mu_a^{(n)}[\gamma_a]$} & probability measure resulting from an approximate sampling method\\
 & which approximates $\mu_a^{(n)}[\gamma_a]$\\
 \hline
  $\theta^*_a$ & true parameter value for arm $a$\\
 \hline
 $\theta_{a,t}$ & sampled parameter for arm $a$ at time $t$ of the Thompson Sampling algorithm: $\theta_{a,t} \sim \mu_a^{(n)}$\\
 \hline
$\bar r_a$ & mean of the reward distribution for arm $a$: $\bar r_a= \mb{E}[X_a|\theta_a^*]$ \\
 \hline
 $\alpha_a^T$ & vector in $\mb{R}^{d_a}$ such that $\bar r_a=\alpha_a^T\theta_a^*$\\
 \hline
 $r_{a,t}(T_a(t))$ & estimate of mean of arm $a$ at round $t$: $r_{a,t}(T_a(t))=\alpha_a^T\theta_{a,t}$\\
 \hline
 $A_a$ & norm of $\alpha_a$\\
 \hline
 $m_a$ & Strong log-concavity parameter of the family $p_a(x;\theta)$ in $\theta$ for all $x$.\\
 \hline
 $\nu_a$ & Strong log-concavity parameter of the true reward distribution $p_a(x;\theta^*)$ in $x$.\\
 \hline
 $F_{n,a}(\theta_a)$ & Averaged log likelihood over the data points: $F_{n,a}(\theta_a) = \frac{1}{n}\sum_{i=1}^n \log p_a(X_i, \theta_a)$ \\
 \hline
 $L_a$ & Lipschitz constant for the true reward distribution, and likelihood families $p_a(x;\theta^*)$ in $x$.\\
  \hline
 $\kappa_a$ & condition number of the likelihood family $\kappa_a=\max\left(\frac{L_a}{m_a},\frac{L_a}{\nu_a}\right)$.\\
 \hline
 $B_a$ & reflects the quality of the prior: $B_a=\frac{\max_\theta \pi_a(\theta)}{\pi_a(\theta^*)}$\\
 \hline
\end{tabular}
\end{center} %
We also define a few notations used within the approximate sampling Algorithm~\ref{alg:Approximate_sampling}.
\begin{center}
 \begin{tabular}{|c | c |} 
 \hline
 Symbol & Meaning\\ [0.5ex] 
 \hline\hline
 $N$ & number of steps of the approximate sampling algorithm\\
  \hline
 $h^{(n)}$ & step size of the approximate sampling algorithm after $n$ samples from the arm\\
 \hline
 $\theta_{ih^{(n)}}$ & MCMC sample generated within $i$-th iteration of Algorithm~\ref{alg:Approximate_sampling}\\
 \hline
 $\mu_{ih^{(n)}}$ &  measure of $\theta_{ih^{(n)}}$\\
  \hline
 $k$ & batch-size of the stochastic gradient Langevin algorithm \\
 \hline
\end{tabular}
\end{center} %

\section{Posterior Concentration Proof}\label{section::posterior_concentration_appendix}

To begin the proof of Theorem~\ref{thm:theorem1}, we first prove that under our assumptions, the gradients of the population likelihood function concentrates.

\begin{proposition}
If the prior distribution over $\theta_a$ satisfies Assumption \ref{assumption:prior_assumption}, then we have:
\[ \sup_{\mathbb{R}^{d_a}} \nabla \log \pi_a(\theta_a)^T(\theta_a-\theta_a^*) \le g_a^*-\log\pi_a(\theta_a^*), \]
where $g_a^* = \max_{\theta \in \mathbb{R}^d} \log \pi_a(\theta_a)$.
\end{proposition}

\begin{proof}
Let $\log \pi_a(\theta_a)=g(\theta_a)$. From the concavity of $g$, we know that

\[ \nabla g(\theta_a)^T(\theta_a-\theta_a^*) \le g(\theta_a)-g(\theta_a^*)\]
Since this holds for all $\theta \in \mathbb{R}^{d_a}$, we take the supremum of both sides and get that:
\[ \sup_{\mathbb{R}^{d_a}}\nabla g(\theta_a)^T(\theta_a-\theta_a^*) \le g^*-g(\theta_a^*)\]
\end{proof}
Let $ \log B_a := g_a^* - \log \pi_a(\theta_a^*)$. If the prior is centered on the correct value of $\theta_a^*$, then $\log B_a=0$. Our posterior concentration rates will depend on $B_a$.

Before proving the posterior concentration result we first show the empirical likelihood function at $\theta^\ast_a$ is a sub-Gaussian random variable:
\begin{proposition}
\label{prop:grad_conc}
 The random variable $\|\nabla_\theta F_{a,n}(\theta_a^*) \|$ is $L_a \sqrt{\frac{d_a}{n\nu_a}}$-sub-Gaussian: 
\end{proposition}

\begin{proof}

Recall that the true density $p_a(x| \theta_a^*)$ is $\nu_a$-strongly log-concave in $x$ and that $\nabla_\theta \log p_a(x| \theta_a^*)$ is $L_a$-Lipschitz in $x$. Notice that $\nabla_\theta F_a(\theta^*_a) = 0$ since $\theta_a^*$ is the point maximizing the population likelihood. 

Let's consider the random variable $Z = \nabla_\theta \log p_a(x| \theta_a^*)$. Since $\mathbb{E}[Z] =  \nabla_\theta F_a(\theta_a^*)$, the random variable $Z$ is centered. 

We start by showing $Z$ is a subgaussian random vector. Let $v \in \mathbb{S}_{d_a}$ be an arbitrary point in the $d_a-$dimensional sphere and define the function $V: \mathbb{R}^{d_a} \rightarrow \mathbb{R}$ as $V(x) = \langle \nabla_\theta \log p_a(x| \theta_a^*), v\rangle$. This function is $L_a-$Lipschitz. Indeed let $x_1, x_2 \in \mathbb{R}^{d_a}$ be two arbitrary points in $\mathbb{R}^{d_a}$:
\begin{align*}
  |  V(x_1) - V(x_2) | &= | \langle \nabla_\theta \log p_a(x_1| \theta_a^*) - \nabla_\theta \log p_a(x_2| \theta_a^*), v \rangle |\\
  &\leq \|  \nabla_\theta \log p_a(x_1| \theta_a^*) - \nabla_\theta \log p_a(x_2| \theta_a^*) \|_2 \|v\|_2 \\
  &= \|  \nabla_\theta \log p_a(x_1| \theta_a^*) - \nabla_\theta \log p_a(x_2| \theta_a^*) \|_2 \\
  &\leq L_a \|x_1 - x_2 \|
\end{align*}
The first inequality follows by Cauchy-Schwartz, the second inequality by the Lipschitz assumption on the gradients. After a simple application of Proposition 2.18 in \citet{ledoux2001concentration}, we conclude that $V(x)$ is subgaussian with parameter $\frac{ L_a }{ \sqrt{\nu_a}}$. %

Since the projection of $Z$ onto an arbitrary direction $v$ of the unit sphere is subgaussian, with a parameter independent of $v$, we conclude the random vector $Z$ is subgaussian with the same parameter $\frac{ L_a }{ \sqrt{\nu_a}}$. Consequently, the vector $\nabla_\theta F_{a,n}(\theta_a^*) $, being an average of $n$ i.i.d. subgaussian vectors with parameter $\frac{L_a }{ \sqrt{\nu_a}}$ is also subgaussian with parameter $\frac{ L_a }{ \sqrt{n \nu_a}}$.

Since $\nabla_\theta F_{a,n}(\theta_a^*)$ is a subgaussian vector with parameter $\frac{ L_a }{ \sqrt{n \nu_a}}$, Lemma 1 of \citep{jin2019short} implies it is norm subgaussian with parameter $\frac{ L_a \sqrt{d_a} }{ \sqrt{n \nu_a}}$. 

%

%
%
%

%
%

%

%

%

%
%
%
%
%
%
%
%

%
%
%
%
%
%
%
%

%
%
%
%
%
%
%
%
%
%
%

\end{proof}

Given these results we now prove Theorem~\ref{thm:theorem1}. For clarity, we restate the theorem below: 

\begin{theorem2}{1}
Suppose that Assumptions 1-3 hold, then given samples $X_a^{(n)}=X_{a,1},...,X_{a,n}$, the posterior distribution satisfies:

\[ \mathbb{P}_{\theta \sim \mu^{(n)}_{a}}[\gamma_a] \left(\|\theta_a-\theta_a^\ast\|_2 > \sqrt{\frac{2e}{m_an}\left(\frac{d_a}{\gamma_a} +\log B_a+\left(\frac{32}{\gamma_a}+\frac{8d_a\kappa_a L_a}{\nu_a}\right)\log\left(1/\delta\right)\right)}\right)<\delta. \]
\end{theorem2}

\begin{proof}
The proof makes use of the techniques used to prove Theorem 1 in \citet{mou2019diffusion}: analyzing how a carefully designed potential function evolves along trajectories of the s.d.e. By a careful accounting of terms and constants, however, we are able to keep explicit constants and derive tighter bounds which hold for any finite number of samples. Throughout the proof we drop the dependence on $a$ and condition on the high-probability event, $G_{a,n}(\delta_1)$, defined in Proposition \ref{prop:grad_conc}, which guarantees that the norm of the likelihood gradients concentrates with probability at least $1-\delta_1$. 

Consider the s.d.e.:
\[ d\theta_t= \frac{1}{2}\nabla_\theta F_{n}(\theta_t)dt+\frac{1}{2n}\nabla_\theta \log\pi(\theta_t))dt +\frac{1}{\sqrt{n\gamma}}dB_t,\]
and the potential function given by: 
\[ V(\theta)=\frac{1}{2}e^{\alpha t}\|\theta-\theta^\ast\|^2_2,\] 
for a choice of $\alpha>0$.
The idea is that bounds on the $p$-th moments of $V(\theta_t)$ can be translated into bounds on the $p$-th moments of $V(\theta)$ where $\theta \sim \mu^{(n)}$, due to the fact that $\lim_{t\rar \infty} \theta_t=\theta \sim \mu^{(n)}$. The square-root growth in $p$ of these moments will imply that $\|\theta-\theta^\ast\|_2$ has subgaussian tails with a rate that we make explicit.

We begin by using Ito's Lemma on $V(\theta_t)$:
\[  V(\theta_t) =T1+T2+T3+T4,\\\]
where:
\begin{align*}
    T1&=-\frac{1}{2}\int_{0}^t e^{\alpha s}\langle \theta^\ast-\theta_s,\nabla_\theta F_{n}(\theta_s) \rangle ds+\frac{\alpha}{2}\int_{0}^te^{\alpha s}\|\theta_s-\theta^\ast\|^2_2 ds\\
    T2&=\frac{1}{2n}\int_{0}^t e^{\alpha s}\langle \theta_s -\theta^\ast,\nabla_\theta \log \pi(\theta_s) \rangle ds\\
    T3&=\frac{d}{2n\gamma}\int_{0}^t e^{\alpha s}  ds\\ 
    T4&=\frac{1}{\sqrt{n\gamma}}\int_{0}^t e^{\alpha s} \langle \theta_s-\theta^\ast,dB_s \rangle 
\end{align*}

Let us first upper bound $T1$:
\begin{align*}
    T1&=-\frac{1}{2}\int_{0}^t e^{\alpha s}\langle \theta^\ast-\theta_s,\nabla_\theta F_{n}(\theta_s) \rangle ds+\frac{\alpha}{2}\int_{0}^te^{\alpha s}\|\theta_s-\theta^\ast\|^2_2 ds\\
    &= -\frac{1}{2}\int_{0}^t e^{\alpha s}\langle \theta^\ast-\theta_s,\nabla_\theta F_{n}(\theta_s)-\nabla_\theta F_{n}(\theta^\ast) \rangle ds+\frac{\alpha}{2}\int_{0}^te^{\alpha s}\|\theta_s-\theta^\ast\|^2_2 ds\\
    &\quad -\frac{1}{2}\int_{0}^t e^{\alpha s}\langle \theta^\ast-\theta_s,\nabla_\theta F_{n}(\theta^\ast) \rangle ds\\
    & \stackrel{(i)}{\le} \frac{\alpha-m}{2}\int_{0}^te^{\alpha s}\|\theta_s-\theta^\ast\|^2_2 ds-\frac{1}{2}\int_{0}^t e^{\alpha s}\langle \theta^\ast-\theta_s,\nabla_\theta F_{n}(\theta^\ast) \rangle ds\\
    &\stackrel{(ii)}{\le} \frac{\alpha-m}{2}\int_{0}^te^{\alpha s}\|\theta_s-\theta^\ast\|^2_2 ds+\frac{1}{2}\int_{0}^t e^{\alpha s} \|\theta^\ast-\theta_s\|\underbrace{\|\nabla_\theta F_{n}(\theta^\ast)\|}_{:=\epsilon(n)} ds\\
\end{align*}
where in $(i)$ we use the strong-concavity property from Assumption~\ref{assumption:family_assumption}, and in $(ii)$ we use Cauchy-Shwartz.

Using Young's inequality for products, where the constant is $m$, gives:
\begin{align*}
    T1&\le \frac{2\alpha-m}{4}\int_{0}^te^{\alpha s}\|\theta_s-\theta^\ast\|^2_2ds + \frac{\epsilon(n)^2}{4 m}\int_{0}^te^{\alpha s} ds
\end{align*}
Finally, choosing $\alpha=m/2$ the first term on the RHS vanishes. Evaluating the integral in the second term on the RHS gives:
\begin{align*}
    T1&\le  \frac{\epsilon(n)^2}{2m^2}\left( e^{\alpha t}-1\right)\le  \frac{\epsilon(n)^2}{m^2}e^{\alpha t}.
\end{align*}
Given our assumption on the prior, our choice of $\alpha=\frac{m}{2}$ and simple algebra, we can upper bound $T2$ and $T3$ as:
\[ T2 = \frac{1}{2n}\int_{0}^t e^{\alpha s}\langle \theta_s -\theta^\ast,\nabla_\theta \log \pi(\theta_s) \rangle ds\le \frac{\log B}{2\alpha n}(e^{\alpha t}-1) \le \frac{\log B}{n m}e^{\alpha t} \]
\[ T3=\frac{d}{2n\gamma}\int_{0}^t e^{\alpha s}  ds\le \frac{d}{\gamma n m}e^{\alpha t}. \]

We proceed to bound $T4$. Let's start by defining: 
\[ M_t=\int_{0}^t e^{\alpha s} \langle \theta_s-\theta^\ast,dB_s \rangle ,\]
so that:
\[ T4=\frac{M_t}{\sqrt{n\gamma}}.\]
Combining all the upper bounds of $T1, T2, T3$, and $T4$ we have that:
\[ V(\theta_t) \le  \left( \frac{\epsilon(n)^2}{m^2} +\frac{d}{\gamma n m}+\frac{\log B}{n m}\right) e^{\alpha t}+\frac{M_t}{\sqrt{\gamma n}}.\]

To find a bound for the $p-$th moments of $V$, we upper bound the $p$-th moments of the supremum of $M_t$ where $p\ge 1$: 
\begin{align*}
    \mb{E}\left[\sup_{0\le t\le T} |M_t|^{p}\right] 
    &\stackrel{(i)}{\le} (8p)^{\frac{p}{2}} \mb{E}\left[\langle M,M\rangle_T ^{\frac{p}{2}}\right]\\
    &=(8p)^{\frac{p}{2}} \mb{E}\left[\left( \int_{0}^T e^{2\alpha s} \|\theta_s-\theta^\ast\|_2^2\right)^{\frac{p}{2}}ds\right]\\
    &\stackrel{(ii)}{\le} (8p)^{\frac{p}{2}} \mb{E}\left[\left(\sup_{0\le t\le T} e^{\alpha t}\|\theta_t-\theta^\ast\|_2^2 \int_{0}^T e^{\alpha s}ds \right)^{\frac{p}{2}}\right]\\
     &= (8p)^{\frac{p}{2}} \mb{E}\left[\left(\sup_{0\le t\le T} e^{\alpha t}\|\theta_t-\theta^\ast\|_2^2 \frac{(e^{\alpha T}-1)}{\alpha}\right)^{\frac{p}{2}}\right]\\
     &\stackrel{(iii)}{\le} \left(\frac{8p e^{\alpha T}}{\alpha}\right)^{\frac{p}{2}} \mb{E}\left[\left(\sup_{0\le t\le T} e^{\alpha t}\|\theta_t-\theta^\ast\|_2^2 \right)^{\frac{p}{2}}\right]\\
\end{align*}

Inequality $(i)$ is a direct consequence of the Burkholder-Gundy-Davis inequality \citep{BurkholderWConstant}, $(ii)$ follows by pulling out the supremum out of the integral, $(iii)$ holds because $e^{\alpha T} - 1 \leq e^{\alpha T}$. 

Now, let us look at the moments of $V(\theta_t)$. 

\begin{align*}
    \mb{E}\left [\left( \sup_{0\le t\le T} V(\theta_t) \right)^p \right]^{\frac{1}{p}}&\le
    \mb{E}\left [\left( \sup_{0\le t\le T} \left( \frac{\epsilon(n)^2}{m^2} +\frac{d}{\gamma n m}+\frac{\log B}{n m}\right) e^{\alpha t} + \frac{|M_t|}{\sqrt{\gamma n}} \right)^p
    \right]^{\frac{1}{p}}\\
    &\le  \mb{E}\left [\left( \sup_{0\le t\le T} \left( \frac{\epsilon(n)^2}{m^2} +\frac{d}{\gamma n m}+\frac{\log B}{n m}\right) e^{\alpha t} + \sup_{0\le t\le T}\frac{|M_t|}{\sqrt{\gamma n}} \right)^p
    \right]^{\frac{1}{p}}
\end{align*}
Via the Minkowski Inequality, and the fact $\epsilon(n)$ is independent of $t$, we can expand the above as:

\begin{align*}
    \mb{E}\left [\left( \sup_{0\le t\le T} V(\theta_t) \right)^p \right]^{\frac{1}{p}}&\le  \underbrace{\left(\frac{d}{\gamma n m}+\frac{\log B}{n m}\right) e^{\alpha T}}_{:=U_T} +\frac{e^{\alpha T}}{m^2}\mb{E}\left[ \epsilon(n)^{2p}\right]^{\frac{1}{p}}+ \mb{E}\left [\left(\sup_{0\le t\le T}\frac{|M_t|}{\sqrt{n}} \right)^p
    \right]^{\frac{1}{p}}
\end{align*}
Since, from Proposition~\ref{prop:grad_conc_main}, we know that $\epsilon(n)$ is a $L\sqrt{\frac{ d }{n \nu}}$-sub-Gaussian vector, we know that:

\[\mb{E}\left[ \epsilon(n)^{2p}\right]^{\frac{1}{p}} \le \left(2L\sqrt{\frac{ 2d p }{n \nu}}\right)^2  \]

Using our upper bound on the supremum of $M_t$ gives:

\begin{align}\label{equation::v_upper_bound}
    \mb{E}\left [\left( \sup_{0\le t\le T} V(\theta_t) \right)^p \right]^{\frac{1}{p}}&\le  U_T +\frac{e^{\alpha T} 8dL^2}{\nu m^2n}p+ \mb{E}\left [\left(\frac{8 p e^{\alpha T}}{\gamma \alpha n}\right)^{\frac{p}{2}}  \left(\sup_{0\le t\le T} e^{\alpha t}\|\theta_t-\theta^\ast\|_2^2 \right)^{\frac{p}{2}}\right]^{\frac{1}{p}}
\end{align}
We proceed by bounding the second term on the RHS of the expression above:
\begin{align*}
    \mb{E}\left [\left(\frac{8p e^{\alpha T}}{\alpha n}\right)^{\frac{p}{2}}  \left(\sup_{0\le t\le T} e^{\alpha t}\|\theta_t-\theta^\ast\|_2^2 \right)^{\frac{p}{2}}\right]^{\frac{1}{p}} &\stackrel{(i)}{\le}  \mb{E}\left [\frac{2^{p-1}}{2}\left(\frac{8p e^{\alpha T}}{\alpha \gamma n}\right)^p  +\frac{1}{2^p}\left(\sup_{0\le t\le T} e^{\alpha t}\|\theta_t-\theta^\ast\|_2^2 \right)^{p}\right]^{\frac{1}{p}}\\
    &\stackrel{(ii)}{\le}  2^{\frac{p-2}{p}}\mb{E}\left [\left(\frac{8p e^{\alpha T}}{\alpha \gamma n}\right)^p\right]^{\frac{1}{p}}+ \frac{1}{2} \mb{E}\left [\left(\sup_{0\le t\le T} e^{\alpha t}\|\theta_t-\theta^\ast\|_2^2 \right)^{p}\right]^{\frac{1}{p}}\\
  &\stackrel{(iii)}{\le}16 \   \mb{E}\left [ \left(\frac{p e^{\alpha T}}{\alpha \gamma n}\right)^p\right]^{\frac{1}{p}}+\underbrace{ \frac{1}{2} \mb{E}\left [\left(\sup_{0\le t\le T} e^{\alpha t}\|\theta_t-\theta^\ast\|_2^2 \right)^{p}\right]^{\frac{1}{p}}  }_{I}
\end{align*}
Inequality $(i)$ follows from using Young's inequality  for products on the term inside the expectation with constant $2^{p-1}$, inequality $(ii)$ is a consequence of Minkowski Inequality and $(iii)$ because $2^{\frac{p-2}{p}} \leq 2$. We note now that the second term $I$ on the right hand side above is exactly:
\[  \frac{1}{2} \mb{E}\left [\left( \sup_{0\le t\le T} V(\theta_t) \right)^p \right]^{\frac{1}{p}} \]
Plugging this into Equation \ref{equation::v_upper_bound} and rearranging gives:

\begin{align*}
    \frac{1}{2}\mb{E}\left [\left( \sup_{0\le t\le T} V(\theta_t) \right)^p \right]^{\frac{1}{p}}&\le  U_T + \frac{16 e^{\alpha T}}{\alpha \gamma n}p+\frac{e^{\alpha T} 8dL^2}{\nu m^2n}p,
\end{align*}
which finally results in:
\begin{align}\label{equation::upper_bound_v}
    \mb{E}\left [\left( \sup_{0\le t\le T} V(\theta_t) \right)^p \right]^{\frac{1}{p}}&\le \frac{2}{mn} \left( \frac{d}{\gamma} +\log B+\left(\frac{32}{\gamma}+\frac{8dL^2}{\nu m}\right)p\right)e^{\alpha T}.
\end{align}

Given this control on the moments of the supremum of $V(\theta_t)$ (recall $ V(\theta)=\frac{1}{2}e^{\alpha t}\|\theta-\theta^\ast\|^2_2 $), we finally construct the bound on the moments of $\|\theta_T-\theta^\ast\|$:
\begin{align*}
    \mb{E}[\|\theta_T-\theta^\ast\|^p]^{\frac{1}{p}} &= \mb{E}\left [e^{-\frac{p \alpha T}{2} } V(\theta_T)^{\frac{p}{2}}\right ]^{\frac{1}{p}} \\
    &\stackrel{(i)}{\le}  \mb{E}\left [e^{-\frac{p \alpha T}{2} } \left(\sup_{0\le t\le T} V(\theta_t)\right)^{\frac{p}{2}}\right ]^{\frac{1}{p}}\\
    &=  e^{-\frac{\alpha T}{2} } \left( \mb{E}\left [\left(\sup_{0\le t\le T} V(\theta_t)\right)^{\frac{p}{2}}\right ]^{\frac{2}{p}}\right)^{\frac{1}{2}}\\
    &\stackrel{(ii)}{\le} e^{-\frac{\alpha T}{2}} \left(\frac{2}{mn} \left( \frac{d}{\gamma} +\log B+\left(\frac{16}{\gamma}+\frac{4dL^2}{\nu m}\right)p\right)e^{\alpha T}\right)^{\frac{1}{2}}\\
    &= \sqrt{\frac{2}{mn}}\left(\frac{d}{\gamma} +\log B+\left(\frac{16}{\gamma}+\frac{4dL^2}{\nu m}\right)p\right)^{\frac{1}{2}}\\
\end{align*}
Inequality $(i)$ follows from taking the supremum of $V(\theta_t)$, inequality $(ii) $ from plugging in the upper bound from Equation \ref{equation::upper_bound_v}.

Taking the limit as $T \rar \infty$ and using Fatou's Lemma, we therefore have that the moments of $\mb{E}[\|\theta -\theta^\ast\|^p]^{\frac{1}{p}}$, with probability at least $1-\delta_1$, grow at a rate of $\sqrt{p}$:
\begin{align}
    \mb{E}[\|\theta -\theta^\ast\|^p]^{\frac{1}{p}}&\le \lim\inf_{T\rar \infty}\mb{E}[\|\theta_T-\theta^\ast\|^p]^{\frac{1}{p}} \\
    &=  \sqrt{\frac{2}{mn}}\left(\frac{d}{\gamma} +\log B+\left(\frac{16}{\gamma}+\frac{4dL^2}{\nu m}\right)p\right)^{\frac{1}{2}}.
\end{align}

To simplify notation, let $D=\left(\frac{d}{\gamma}+\log B\right)$, and $\sigma=\left(\frac{16}{\gamma}+\frac{4dL^2}{\nu m}\right)$. Therefore we have:

\begin{align}
    \mb{E}[\|\theta -\theta^\ast\|^p]^{\frac{1}{p}}\le \sqrt{\frac{2}{mn}\left( D +\sigma p\right) }
    \label{eq:posterior_moment_bound}
\end{align}

The result \eqref{eq:posterior_moment_bound}, guarantees us that the norm of the uncentered random variable $\theta-\theta^*$  has subgaussian tails. We make the parameters explicit via Markov's inequality:
\begin{align*}
    \mb{P}_{\theta \sim \mu^{(n)}_{a}}\left(\|\theta-\theta^*\|>\epsilon \right)&\le \frac{\mb{E}[\|\theta -\theta^\ast\|^p]}{\epsilon^p}\\
    &\le \left( \frac{\sqrt{2\left(D +\sigma p\right)} }{\sqrt{mn}\epsilon}\right)^p.
\end{align*}

Choosing $p=2\log{1/\delta_1}$ and letting
\[ \epsilon= e^{\frac{1}{2}} \sqrt{\frac{2}{mn}\left( D +\sigma p\right) } \]
gives us our desired solution:

\[ \mathbb{P}_{\theta \sim \mu^{(n)}_{a}[\gamma_a]} \left(\|\theta-\theta^\ast\|_2 > \sqrt{\frac{2e}{mn}\left(\frac{d}{\gamma} +\log B+\left(\frac{32}{\gamma}+\frac{8dL^2}{\nu m}\right)\log\left(1/\delta\right)\right)}\right)<\delta. \]
\end{proof}

\section{Introduction to the Langevin Algorithms}
We refer to the stochastic process represented by the following
stochastic differential equation as \emph{continuous-time Langevin dynamics}:
\begin{align*}
\rd \theta_t & = - \nabla U(\theta_t) \ \rd t + \sqrt{2}
\ \rd B_t.
\end{align*}
We have first encountered this continuous time Langevin dynamics in Eq.~\eqref{eq:sde}, where we have set 
$U(\theta) = - \gamma_a\lrp{n F_{n,a}(\theta) + \log\pi_a(\theta)} 
= - \gamma_a \sum_{i=1}^n \log p_a\lrp{x_{a,i}|\theta} - \gamma_a \log \pi_a(\theta)$ 
to prove posterior concentration of $\mu_a^{(n)}[\gamma_a]$.

One important feature of the Langevin dynamics is that its invariant distribution is proportional to $e^{-U(\theta)}$.
We can therefore also use it to generate samples distributed according to the unscaled posterior distribution $\mu_a^{(n)}$.
Via letting $U(\theta) = - \sum_{i=1}^n \log p_a\lrp{x_{a,i}|\theta} - \log \pi_a(\theta)$, we obtain a continuous time dynamics which generates trajectories that converge towards the posterior distribution $\mu_a^{(n)}$ exponentially fast.
To obtain an implementable algorithm, we apply Euler-Maruyama discretization to the Langevin dynamics and arrive at the following ULA update:
\begin{align*}
\theta_{(i+1) h^{(n)}} \sim \mathcal{N}\lrp{ \theta_{i h^{(n)}} - h^{(n)} \nabla U(\theta_{i h^{(n)}}) , 2 h^{(n)} \mI }.
\end{align*}
Since $\nabla U(\theta) = - \sum_{i=1}^n \nabla \log p_a\lrp{x_{a,i}|\theta} - \nabla \log \pi_a(\theta)$ in the above update rule, the computation complexity within each iteration of the Langevin algorithm grows with the number of data being collected, $n$.
To cope with the growing number of terms in $\nabla U(\theta)$, we take a stochastic gradient approach and define
$\hpot(\theta) = - \frac{n}{|\mathcal{S}|} \sum_{x_k\in\mathcal{S} } \nabla \log p_a(x_k|\theta) - \nabla \log \pi_a(\theta)$, where $\mathcal{S}$ is a subset of the dataset $\{x_{a,1},\cdots,x_{a,n}\}$.
For simplicity, we form $\mathcal{S}$ via subsampling uniformly from $\{x_{a,1},\cdots,x_{a,n}\}$.
Substituting the stochastic gradient $\nabla \hpot$ for the full gradient $\nabla U$ in the above update rule results in the SGLD algorithm.

\section{Proofs for Approximate MCMC Sampling}\label{section::approximate_mcmc_sampling}

In this Appendix we supply the proofs of concentration for approximate samples from both the ULA and SGLD MCMC methods. 
We will quantify the computation complexity of generating samples which are distributed close enough to the posterior.
We restate the assumptions required of the likelihood for the MCMC sampling methods to converge.
\begin{assump}{1-Uniform}[Assumption on the family $p_a(X|\theta_a)$: strengthened for approximate sampling]
Assume that $\log p_a(x|\theta_a)$ is $L_a$-smooth and $m_a$-strongly concave over the parameter $\theta_a$:
\begin{multline*}
    -\log p_a(x|\theta_a') -\nabla_{\theta} \log p_a(x|\theta_a')^\top \left( \theta_a - \theta_a' \right) + \frac{m_a}{2}\| \theta_a - \theta_a' \|^2
    \leq -\log p_a(x|\theta_a) \\ 
    \leq -\log p_a(x|\theta_a') -\nabla_{\theta} \log p_a(x|\theta_a')^\top \left( \theta_a - \theta_a' \right) + \frac{L_a}{2}\| \theta_a - \theta_a'\|^2, 
    \quad \forall \theta_a, \theta_a' \in \mathbb{R}^{d_a}, x \in \mb{R}.
\end{multline*}
\end{assump}

\setcounter{assumption}{2}
\begin{assumption}[Assumptions on the prior distribution] 
For every $a \in \mathcal{A}$ assume that $\log\pi_a(\theta_a)$ is concave with $L$-Lipschitz gradients for all $\theta_a \in \mb{R}^{d_a}$:

\[ \|\nabla_\theta \pi_a(\theta)-\nabla_\theta\pi_a(\theta')\|\le L_a\|\theta-\theta' \| \ \ \ \forall \theta,\theta' \in \mb{R}^{d_a}\]
\end{assumption}

\begin{assumption}[Joint Lipschitz smoothness of the family $\log p_a(X|\theta_a)$: for SGLD]
\label{assumption:sg_lip}
Assume a joint Lipschitz smoothness condition, which strengthens Assumptions~\ref{assumption:family_assumption} and~\ref{assumption:true_reward_assumption}
to impose the Lipschitz smoothness on the entire bivariate function $\log p_a(x;\theta)$:
\begin{align*}
\lrn{\nabla_\theta \log p_a(x|\theta_a) - \nabla_\theta \log p_a(x'|\theta_a)}
\leq L_a\lrn{\theta_a - \theta_a'} + L_a^*\lrn{x-x'},
\quad \forall \theta_a, \theta_a' \in \mathbb{R}^{d_a}, x, x' \in \mb{R}.
\end{align*}
\end{assumption}

We now begin by presenting the result for ULA.

\subsection{Convergence of the unadjusted Langevin algorithm (ULA)}

If function $\log p_a(x;\theta)$ satisfies the Lipschitz smoothness condition in Assumption~\ref{assumption:family_assumption}, then we can leverage gradient based MCMC algorithms to generate samples with convergence guarantees in the $p$-Wasserstein distance.
As stated in Algorithm~\ref{alg:Approximate_sampling}, we initialize ULA in the $n$-th round from the last iterate in the $(n-1)$-th round.
\begin{theorem}[ULA Convergence]
\label{theorem_ULA}
Assume that the likelihood $\log p_a(x;\theta)$ and prior $\pi_a$ satisfy Assumption~\ref{assumption:Uniform_strong_concavity} and Assumption~\ref{assumption:prior_assumption}.
We take step size 
$h^{(n)} =
\frac{1}{32} \frac{m_a}{n \lrp{L_a+\frac{1}{n}L_a}^2} 
= \mathcal{O}\lrp{ \frac{1}{n  L_a \kappa_a} 
}$
and number of steps 
$N = 640 \frac{\lrp{L_a+\frac{1}{n}L_a}^2}{m_a^2} 
= \mathcal{O}\lrp{ \kappa_a^2 }$
in Algorithm~\ref{alg:Approximate_sampling}. 
If the posterior distribution satisfy the concentration inequality that $\Ep{\theta\sim\mu_a^{(n)}}{\|\theta -\theta^\ast\|^p}^{\frac{1}{p}}\le \frac{1}{\sqrt{n}} \widetilde{D}$,
then for any positive even integer $p$, we have convergence of the ULA algorithm in $W_p$ distance to the posterior $\mu_a^{(n)}$: 
$W_p\lrp{\hmu_a^{(n)}, \mu_a^{(n)}} \leq \frac{2}{\sqrt{n}} \widetilde{D}$, $\forall \widetilde{D}\geq \sqrt{\frac{32}{m_a} d_a p}$.
\end{theorem}

\begin{proof}[Proof of Theorem~\ref{theorem_ULA}]
We use induction to prove this theorem.
\begin{itemize}
\item 
For $n=1$, we initialize at $\theta_0$ which is within a $\sqrt{\frac{d_a}{m_a}}$-ball from the maximum of the target distribution, $\theta^*_p = \arg\max p_a(\theta|x_1)$, where $p_a(\theta|x_1) \propto p_a(x_1|\theta) \pi_a(\theta)$ and negative $\log p_a(\theta|x_1)$ is $m_a$-strongly convex and $\lrp{ L_a+L_a }$-Lipschitz smooth.
Invoking Lemma~\ref{lemma:W_p_fixed_point}, we obtain that for $\rd \mu_a^{(1)} = p_a(\theta|x_1) \rd \theta$, Wasserstein-$p$ distance between the target distribution and the point mass at its mode:
$
W_p \lrp{\mu_a^{(1)}, \delta\lrp{\theta_p^*}} 
\leq 5 \sqrt{\frac{1}{ m_a} d_a p}.
$
Therefore, $W_p \lrp{\mu_a^{(1)}, \delta\lrp{\theta_0}} \leq W_p \lrp{\mu_a^{(1)}, \delta\lrp{\theta_p^*}} + \lrn{\theta_0-\theta^*_p} \leq 6 \sqrt{\frac{1}{ m_a} d_a p}$.
We then invoke Lemma~\ref{lemma:ULA_cvg}, with initial condition $\mu_0 = \delta\lrp{\theta_p^*}$, to obtain the convergence in the $N$-th iteration of Algorithm~\ref{alg:Approximate_sampling} after the first pull to arm $a$:
\begin{align*}
W_p^p\lrp{\mu_{Nh^{(1)}}, \mu_a^{(1)}}
\leq \lrp{1 - \frac{ m_a}{8} h^{(1)}}^{p \cdot N} W_p^p\lrp{\delta\lrp{\theta_0}, \mu_a^{(1)}}
+ 2^{5p} \frac{\lrp{L_a+L_a}^{p}}{m_a^{p}} \lrp{d_a p}^{p/2} \lrp{h^{(1)}}^{p/2},
\end{align*}
where we have substituted in the strong convexity $ m_a$ for $\conv$ and the Lipschitz smoothness $ \lrp{L_a+L_a}$ for $\lip$.
Plugging in the step size $h^{(1)} = \frac{1}{32} \frac{m_a}{ \lrp{L_a+L_a}^2} 
\leq \min\lrbb{ \frac{m_a}{32  \lrp{L_a+L_a}^2}, \frac{1}{1024} \frac{m_a^2}{\lrp{L_a+L_a}^2} \frac{ {\widetilde{D}^2} }{d_a p} }$, and number of steps $N = \frac{20}{ m_a} \frac{1}{h^{(1)}} = 640 \frac{ \lrp{L_a+L_a}^2 }{ m_a^2 }$,
$
W_p^p\lrp{\hmu_a^{(1)}, \mu_a^{(1)}} 
= W_p^p\lrp{\mu_{N h^{(1)}}, \mu_a^{(1)}} 
\leq 2\widetilde{D}^p.
$
\item
Assume that after the $(n-1)$-th pull and before the $n$-th pull to the arm $a$, the ULA algorithm guarantees that $W_p\lrp{\hmu_a^{(n-1)}, \mu_a^{(n-1)}} \leq \frac{2}{\sqrt{n-1}} \widetilde{D}$.
We now prove that after the $n$-th pull and before the $(n+1)$-th pull, it is guaranteed that $W_p\lrp{\hmu_a^{(n)}, \mu_a^{(n)}} \leq \frac{2}{\sqrt{n}} \widetilde{D}$.
We first obtain from the assumed posterior concentration inequality:
\begin{align}
    W_p(\mu_a^{(n)}, \delta\lrp{\theta^*})
    \le \Ep{\theta\sim\mu_a^{(n)}}{\|\theta -\theta^\ast\|^p}^{\frac{1}{p}}\le \frac{1}{\sqrt{n}} \widetilde{D}.
\end{align}
Therefore, for $n\geq2$,
\begin{align*}
W_p\lrp{\mu_a^{(n)}, \mu_a^{(n-1)}} 
\le W_p(\mu_a^{(n)}, \delta\lrp{\theta^*}) + W_p(\mu_a^{(n-1)}, \delta\lrp{\theta^*})
\le \frac{3}{\sqrt{n}} \widetilde{D}.
\end{align*}
We combine this bound with the induction hypothesis and obtain that
$$
W_p\lrp{\mu_a^{(n)}, \hmu_a^{(n-1)}} 
\le W_p\lrp{\mu_a^{(n)}, \mu_a^{(n-1)}} + W_p\lrp{\mu_a^{(n-1)}, \hmu_a^{(n-1)}} 
\le \frac{8}{\sqrt{n}} \widetilde{D}.
$$

From Lemma~\ref{lemma:ULA_cvg}, we know that for $\conv = n \cdot  m_a$ and $\lip = n \cdot  L_a +  L_a$, with initial condition $\mu_0 = \hmu_a^{(n-1)}$, with accurate gradient,
\begin{align*}
W_p^p\lrp{\mu_{ih^{(n)}}, \mu_a^{(n)}}
\leq \lrp{1 - \frac{\conv}{8} h^{(n)}}^{p \cdot i} W_p^p\lrp{\hmu_a^{(n-1)}, \mu_a^{(n)}}
+ 2^{5p} \frac{\lip^{p}}{\conv^{p}} \lrp{d_a p}^{p/2} \lrp{h^{(n)}}^{p/2}.
\end{align*}
If we take step size 
$h^{(n)} = \frac{1}{32} \frac{\conv}{\lip^2}
\leq \min\lrbb{ \frac{\conv}{32\lip^2}, \frac{1}{1024} \frac{1}{n} \frac{\conv^2}{\lip^2} \frac{ {\widetilde{D}^2} }{d_a p} }$
and number of steps taken in the ULA algorithm from $(n-1)$-th pull till $n$-th pull to be: $\widehat{N} \geq \frac{20}{\conv} \frac{1}{h^{(n)}}$, 
\begin{align}
W_p^p\lrp{\hmu_a^{(n)}, \mu_a^{(n)}} 
= W_p^p\lrp{\mu_{\widehat{N} h^{(n)}}, \mu_a^{(n)}} 
\leq \lrp{1 - \frac{\conv}{8} h^{(n)}}^{p \cdot \widehat{N}} \frac{8^p \widetilde{D}^p}{n^{p/2}}
+ 2^{5p} \frac{\lip^{p}}{\conv^{p}} \lrp{d_a p}^{p/2} \lrp{h^{(n)}}^{p/2}
\leq \frac{2\widetilde{D}^p}{n^{p/2}},
\end{align}
leading to the result that $W_p\lrp{\hmu_a^{(n)}, \mu_a^{(n)}} \leq \frac{2}{\sqrt{n}} \widetilde{D}$.

Since at least one round would have past from the $(n-1)$-th pull to the $n$-th pull to arm $a$, taking number of steps in each round $t$ to be 
$N = \frac{20}{\conv} \frac{1}{h^{(n)}} = 640 \frac{\lrp{L_a+\frac{1}{n}L_a}^2}{m_a^2}$ suffices.
\end{itemize}
Therefore, 
$N = 640 \frac{\lrp{L_a+\frac{1}{n}L_a}^2}{m_a^2} 
= \mathcal{O}\lrp{ \frac{L_a^2}{m_a^2} }$.
\end{proof}

\subsection{Convergence of the stochastic gradient Langevin algorithm (SGLD)}
If $\log p_a(x;\theta)$ satisfies a stronger joint Lipschitz smoothness condition in Assumption~\ref{assumption:sg_lip}, similar guarantees can be obtained for stochastic gradient MCMC algorithms.

\begin{theorem}[SGLD Convergence]
\label{theorem_SG}
Assume that the family $\log p_a(x;\theta)$ and prior $\pi_a$ satisfy Assumption~\ref{assumption:Uniform_strong_concavity}, Assumption~\ref{assumption:prior_assumption}, and Assumption~\ref{assumption:sg_lip}.
We take number of data samples in the stochastic gradient estimate
$k = 32\frac{\lrp{L_a^*}^2}{m_a \nu_a} = 32 \kappa_a^2$,
step size 
$h^{(n)} =
\frac{1}{32} \frac{m_a}{n \lrp{L_a+\frac{1}{n}L_a}^2} 
= \mathcal{O}\lrp{ \frac{1}{n  L_a \kappa_a} 
}$
and number of steps 
$N = 1280 \frac{\lrp{L_a+\frac{1}{n}L_a}^2}{m_a^2} = \mathcal{O}\lrp{ \kappa_a^2 }$
in Algorithm~\ref{alg:Approximate_sampling}. 
If the posterior distribution satisfy the concentration inequality that $\Ep{\theta\sim\mu_a^{(n)}}{\|\theta -\theta^\ast\|^p}^{\frac{1}{p}}\le \frac{1}{\sqrt{n}} \widetilde{D}$,
then for any positive even integer $p$, we have convergence of the ULA algorithm in $W_p$ distance to the posterior $\mu_a^{(n)}$: 
$W_p\lrp{\hmu_a^{(n)}, \mu_a^{(n)}} \leq \frac{2}{\sqrt{n}} \widetilde{D}$, $\forall \widetilde{D}\geq \sqrt{\frac{32}{m_a} d_a p}$.
\end{theorem}

\begin{proof}[Proof of Theorem~\ref{theorem_SG}]
Similar to Theorem~\ref{theorem_ULA}, we use induction to prove this theorem.
After the first pull to arm $a$, we take the same 
$640 \frac{\lrp{L_a+\frac{1}{n}L_a}^2}{m_a^2}$ number of steps to converge to $
W_p^p\lrp{\hmu_a^{(1)}, \mu_a^{(1)}}
\leq 2\widetilde{D}^p.
$

Assume that after the $(n-1)$-th pull and before the $n$-th pull to the arm $a$, the SGLD algorithm guarantees that $W_p\lrp{\hmu_a^{(n-1)}, \mu_a^{(n-1)}} \leq \frac{2}{\sqrt{n-1}} \widetilde{D}$.
We prove that after the $n$-th pull and before the $(n+1)$-th pull, it is guaranteed that $W_p\lrp{\hmu_a^{(n)}, \mu_a^{(n)}} \leq \frac{2}{\sqrt{n}} \widetilde{D}$.
Following the proof of Theorem~\ref{theorem_ULA}, we combine the assumed posterior concentration inequality and the induction hypothesis to obtain:
$$
W_p\lrp{\mu_a^{(n)}, \hmu_a^{(n-1)}} 
\le W_p\lrp{\mu_a^{(n)}, \mu_a^{(n-1)}} + W_p\lrp{\mu_a^{(n-1)}, \hmu_a^{(n-1)}} 
\le \frac{8}{\sqrt{n}} \widetilde{D}.
$$
Denote function $U$ as the negative log-posterior density over parameter $\theta$.
From Lemma~\ref{lemma:ULA_cvg}, we know that for $\conv = n \cdot  m_a$ and $\lip = n \cdot  L_a +  L_a$, with initial condition that $\mu_0=\hmu_a^{(n-1)}$,
if the difference between the stochastic gradient $\nabla \hpot$ and the exact one $\nabla \pot$ is bounded as $\E{\lrn{ \nabla \pot(\theta) - \nabla \hpot(\theta) }^p \big| \theta} \leq \Delta_p$, then
\begin{align*}
W_p^p\lrp{\mu_{ih^{(n)}}, \mu_a^{(n)}}
\leq \lrp{1 - \frac{\conv}{8} h^{(n)}}^{p \cdot i} W_p^p\lrp{\hmu_a^{(n-1)}, \mu_a^{(n)}}
+ 2^{5p} \frac{\lip^{p}}{\conv^{p}} \lrp{d_a p}^{p/2} \lrp{h^{(n)}}^{p/2}
+ 2^{2p+3} \frac{\Delta_p}{\conv^p}.
\end{align*}

We demonstrate in the following Lemma~\ref{lemma:inaccurate_grad} that 
\begin{align*}
\Delta_p \leq 2 \frac{n^{p/2}  }{k^{p/2}} \lrp{ \frac{\sqrt{d_a p} L_a^* }{\sqrt{\nu_a}} }^p.
\end{align*}

\begin{lemma}
\label{lemma:inaccurate_grad}
Denote $\hpot$ as the stochastic estimator of $\pot$.
Then for stochastic gradient estimate with $k$ data points,
\begin{align*}
\E{\lrn{ \nabla \hpot(\theta) - \nabla \pot(\theta) }^p \big| \theta} \leq 2 \frac{n^{p/2}  }{k^{p/2}} \lrp{ \frac{\sqrt{d_a p} L_a^* }{\sqrt{\nu_a}} }^p.
\end{align*}
\end{lemma}
If we take the number of samples in the stochastic gradient estimator $k=32\frac{\lrp{L_a^*}^2}{m_a \nu_a}$, then 
$\Delta_p \leq \frac{2}{32^{p/2}} \lrp{n \cdot m_a}^{p/2} \cdot \lrp{p \cdot d_a}^{p/2} 
\leq 2^{-2p-5} \frac{\conv^p \widetilde{D}^p }{n^{p/2}}$ for any $p\geq2$.
Consequently, $2^{2p+3} \frac{\Delta_p}{\conv^p} \leq \frac{1}{4} \frac{\widetilde{D}^p}{n^{p/2}}$.

If we take step size 
$h^{(n)} = \frac{1}{32} \frac{\conv}{\lip^2}
\leq \min\lrbb{ \frac{\conv}{32\lip^2}, \frac{1}{1024} \frac{1}{n} \frac{\conv^2}{\lip^2} \frac{ {\widetilde{D}^2} }{d_a p} }$
and number of steps taken in the SGLD algorithm from $(n-1)$-th pull till $n$-th pull to be: $\widehat{N} \geq \frac{40}{\conv} \frac{1}{h^{(n)}}$, 
\begin{align*}
W_p^p\lrp{\hmu_a^{(n)}, \mu_a^{(n)}} 
= W_p^p\lrp{\mu_{\widehat{N} h^{(n)}}, \mu_a^{(n)}} 
&\leq \lrp{1 - \frac{\conv}{8} h^{(n)}}^{p\cdot \widehat{N}} \frac{8^p \widetilde{D}^p}{n^{p/2}}
+ 2^{5p} \frac{\lip^{p}}{\conv^{p}} \lrp{d_a p}^{p/2} \lrp{h^{(n)}}^{p/2}
+ 2^{2p+3} \frac{\Delta_p}{\conv^p} \\
&\leq \frac{2\widetilde{D}^p}{n^{p/2}},
\end{align*}
leading to the result that $W_p\lrp{\hmu_a^{(n)}, \mu_a^{(n)}} \leq \frac{2}{\sqrt{n}} \widetilde{D}$.
Since at least one round would have past from the $(n-1)$-th pull to the $n$-th pull to arm $a$, taking number of steps in each round $t$ to be $N = \frac{40}{\conv} \frac{1}{h^{(n)}}$ suffices.

Therefore, $N = 1280 \frac{\lrp{L_a+\frac{1}{n}L_a}^2}{m_a^2} = \mathcal{O}\lrp{ \frac{L_a^2}{m_a^2} }$.
\end{proof}

\begin{proof}[Proof of Lemma~\ref{lemma:inaccurate_grad}]
We first develop the expression:
\begin{align*}
\E{\lrn{ \nabla \pot(\theta)  - \nabla \hpot(\theta) }^p }
&= n^p  \E{\lrn{ \frac{1}{n} \sum_{i=1}^n \nabla \log p(x_i|\theta_a) - \frac{1}{k} \sum_{j=1}^k \nabla \log p(x_j|\theta_a) }^p } \\
&= \frac{n^p  }{k^p} \E{\lrn{ \sum_{j=1}^k \lrp{ \frac{1}{n} \sum_{i=1}^n \nabla \log p(x_i|\theta_a) - \nabla \log p(x_j|\theta_a) } }^p }.
\end{align*}
We note that 
\begin{align*}
\nabla \log p(x_j|\theta_a) - \frac{1}{n} \sum_{i=1}^n \nabla \log p(x_i|\theta_a) 
= \frac{1}{n} \sum_{i\neq j} \lrp{\nabla \log p(x_j|\theta_a) - \nabla \log p(x_i|\theta_a)}.
\end{align*}
By the joint Lipschitz smoothness Assumption~\ref{assumption:sg_lip}, we know that $\nabla \log p(x|\theta_a)$ is a Lipschitz function of $x$:
\begin{align*}
\lrn{\nabla \log p(x_j|\theta_a) - \nabla \log p(x_i|\theta_a)}
\leq L_a^* \lrn{x_j-x_i}.
\end{align*}
On the other hand, the data $x$ follows the true distribution $p(x;\theta^*)$, which by Assumption~\ref{assumption:true_reward_assumption} is $\nu_a$-strongly log-concave.
Applying Theorem $3.16$ in~\citep{WainwrightBook}, we obtain that $\lrp{\nabla \log p(x_j|\theta_a) - \nabla \log p(x_i|\theta_a)}$ is $\frac{2 L_a^*}{\sqrt{\nu_a}}$-sub-Gaussian.
Leveraging the Azuma-Hoeffding inequality for martingale difference sequences~\citep{WainwrightBook}, we obtain that sum of the $(n-1)$ sub-Gaussian random variables:
\[ \lrp{ \nabla \log p(x_j|\theta_a) - \frac{1}{n} \sum_{i=1}^n \nabla \log p(x_i|\theta_a) },\]
is $\frac{2 \sqrt{n-1} L_a^*}{n\sqrt{ \nu_a} }$-sub-Gaussian.
In the same vein, $\lrp{ \sum_{j=1}^k \lrp{ \frac{1}{n} \sum_{i=1}^n \nabla \log p(x_i|\theta_a) - \nabla \log p(x_j|\theta_a) } }$ is $\frac{2 \sqrt{k (n-1)} L_a^*}{n \sqrt{\nu_a}}$-sub-Gaussian.
We then invoke the $\frac{2 \sqrt{d_a k (n-1)} L_a^*}{n \sqrt{\nu_a}}$-sub-Gaussianity of \[\lrn{ \sum_{j=1}^k \lrp{ \frac{1}{n} \sum_{i=1}^n \nabla \log p(x_i|\theta_a) - \nabla \log p(x_j|\theta_a) } }\] and have
\begin{align*}
\E{\lrn{ \sum_{j=1}^k \lrp{ \frac{1}{n} \sum_{i=1}^n \nabla \log p(x_i|\theta_a) - \nabla \log p(x_j|\theta_a) } }^p }
\leq 2\lrp{ \frac{2 \sqrt{d_a k (n-1) p} L_a^* }{e n \sqrt{\nu_a}} }^p.
\end{align*}
Therefore, 
\begin{align*}
\E{\lrn{ \nabla \pot(\theta)  - \nabla \hpot(\theta) }^p }
&= \frac{n^p  }{k^p} \E{\lrn{ \sum_{j=1}^k \lrp{ \frac{1}{n} \sum_{i=1}^n \nabla \log p(x_i|\theta_a) - \nabla \log p(x_j|\theta_a) } }^p } \\
&\leq 2 \frac{n^{p/2}  }{k^{p/2}} \lrp{ \frac{2 \sqrt{d_a p} L_a^* }{e \sqrt{\nu_a}} }^p 
\leq 2 \frac{n^{p/2}  }{k^{p/2}} \lrp{ \frac{\sqrt{d_a p} L_a^* }{\sqrt{\nu_a}} }^p.
\end{align*}

\end{proof}

\subsection{Convergence of (Stochastic Gradient) Langevin Algorithm within Each Round}
In this section, we examine convergence of the (stochastic gradient) Langevin algorithm to the posterior distribution over $a$-th arm at the $n$-th round.
Since only the $a$-th arm and $n$-th round are considered, we drop these two indices in the notation whenever suitable.
We also define some notation that will only be used within this subsection.
For example, we focus on the $\theta$ parameter and denote the posterior measure $\rd \mu_a^{(n)} (x;\theta) = \rd \mu^*(\theta) = \exp\lrp{ - \pot(\theta) } \rd \theta$ as the target distribution.
\begin{center}
 \begin{tabular}{|c | c |} 
 \hline
 Symbol & Meaning\\ [0.5ex] 
 \hline\hline
 $\mu^*$ & posterior distribution, $\mu_a^{n}$ \\
 \hline
 $\pot$ & potential (i.e., negative log posterior density) \\
 \hline
 $\theta_U^*$ & minimum of the potential $U$ (or mode of the posterior $\mu^*$) \\
 \hline
 $\theta_t$ &  interpolation between $\theta_{ih^{(n)}}$ and $\theta_{(i+1)h^{(n)}}$, for $t\in[{ih^{(n)}},{(i+1)h^{(n)}}]$\\
 \hline
 $\mu_t$ &  measure associated with $\theta_t$\\
 \hline
 $\theta_t^*$ &  an auxiliary stochastic process with initial distribution $\mu^*$ and follows dynamics~\eqref{eq:LD_dyn}\\
 \hline
 $\conv$ & strong convexity of the potential $U$, $n m_a$ \\
 \hline
 $\lip$ & Lipschitz smoothness of the potential $U$, $n L_a + L_a$ \\
 \hline
\end{tabular}
\end{center}
We also formally define the Wasserstein-$p$ distance used in the main text.
Given a pair of distributions $\mu$ and $\nu$ on $\real^d$, a \emph{coupling} $\gamma$ is a joint distribution over the product space $\real^d \times \real^d$ that has $\mu$ and $\nu$ as its marginal distributions.  
We let $\Gamma(\mu, \nu)$ denote the space of all possible couplings of $\mu$ and $\nu$.  
With this notation, the Wasserstein-$p$ distance is given by
\begin{align}
  \label{EqnWasserstein}
{W}^p (\mu, \nu) & = \inf_{\gamma\in\Gamma(\mu,
  \nu)} \int_{\mathbb{R}^d\times\mathbb{R}^d} \lrn{x-y}^p \rd
\gamma(x,y).
\end{align}

We use the following (stochastic gradient) Langevin algorithm to generate approximate samples from the posterior distribution $\mu_a^{(n)}(\theta)$ at $n$-th round. 
For $i=0,\cdots,T$,
\begin{align}
\theta_{(i+1)h^{(n)}} \sim \mathcal{N} \lrp{\theta_{i h^{(n)}} - h^{(n)} \nabla \hpot(\theta_{ih^{(n)}}), 2 h^{(n)}\mI },
\label{eq:LD_iter}
\end{align}
where $\nabla \hpot(\theta_{ih^{(n)}})$ is a stochastic estimate of $\nabla \pot(\theta_{ih^{(n)}})$.
We prove in the following Lemma~\ref{lemma:ULA_cvg} the convergence of this algorithm within $n$-th round.
\begin{lemma}
\label{lemma:ULA_cvg}
Assume that the potential $\pot$ is $\conv$-strongly convex and $\lip$-Lipschitz smooth.
Further assume that the $p$-th moment between the true gradient and the stochastic one satisfies: \begin{align*}
\E{\lrn{ \nabla \pot(\theta_{ih^{(n)}}) - \nabla \hpot(\theta_{ih^{(n)}}) }^p \Big| \theta_{ih^{(n)}}} \leq \Delta_p.
\end{align*}
Then at $i$-th step, for $\mu_{ih^{(n)}}$ following the (stochastic gradient) Langevin algorithm with $h\leq\frac{\conv}{32\lip^2}$,
\begin{align}
W_p^p\lrp{\mu_{ih^{(n)}}, \mu^*}
\leq \lrp{1 - \frac{\conv}{8} h^{(n)}}^{p \cdot i} W_p^p\lrp{\mu_{0}, \mu^*}
+ 2^{5p} \frac{\lip^{p}}{\conv^{p}} \lrp{dp}^{p/2} \lrp{h^{(n)}}^{p/2}
+ 2^{2p+3} \frac{\Delta_p}{\conv^p}.
\end{align}
\begin{remark}
When $\Delta_p=0$, Lemma~\ref{lemma:ULA_cvg} provides convergence rate of the unadjusted Langevin algorithm (ULA) with the exact gradient.
\end{remark}
\end{lemma}

\begin{proof}[Proof of Lemma~\ref{lemma:ULA_cvg}]
We first interpolate a continuous time stochastic process, $\theta_t$, between $\theta_{ih^{(n)}}$ and $\theta_{(i+1)h^{(n)}}$.
For $t\in[{ih^{(n)}},{(i+1)h^{(n)}}]$,
\begin{align}
\rd \theta_t = \nabla \hpot(\theta_{ih^{(n)}}) \rd t + \sqrt{2} \rd B_t,
\label{eq:LD_interp}
\end{align}
where $B_t$ is standard Brownian motion.
This process connects $\theta_{ih^{(n)}}$ and $\theta_{(i+1)h^{(n)}}$ and approximates the following stochastic differential equation which maintains the exact posterior distribution:
\begin{align}
\rd \theta^*_t = \nabla \pot(\theta^*_{t}) \rd t + \sqrt{2} \rd B_t.
\label{eq:LD_dyn}
\end{align}
For a $\theta_t^*$ initialized from $\mu^*$ and following equation~\eqref{eq:LD_dyn}, $\theta_t^*$ will always have distribution $\mu^*$.

We therefore design a coupling between the two processes: $\theta_t$ and $\theta^*_t$, where $\theta_t$ follows equation~\eqref{eq:LD_interp} (and thereby interpolates Algorithm~\ref{alg:Approximate_sampling}) and $\theta^*_t$ initializes from $\mu^*$ and follows equation~\eqref{eq:LD_dyn} (and thereby preserves $\mu^*$).
By studying the difference between the two processes, we will obtain the convergence rate in terms of the Wasserstein-$p$ distance.

For $t=ih^{(n)}$, we let $\theta_{ih^{(n)}}$ to couple optimally with $\theta^*_{ih^{(n)}}$, so that for 
\[{\lrp{\theta_{ih^{(n)}},\theta^*_{ih^{(n)}}} \sim \gamma^*\in\Gamma_{opt} \lrp{ \mu_{ih^{(n)}},\mu^*_{ih^{(n)}}} },\]
$\E{\lrn{ \theta_{ih^{(n)}}-\theta^*_{ih^{(n)}} }^p} = W_p^p\lrp{\mu_{ih^{(n)}}, \mu^*}$.
For $t\in[ih^{(n)},(i+1)h^{(n)}]$, we choose a synchronous coupling $\bar{\gamma}\lrp{\theta_t,\theta^*_t | \theta_{ih^{(n)}},\theta^*_{ih^{(n)}}} \in\Gamma\lrp{\mu_t(\theta_t|\theta_{ih^{(n)}}), \mu^*_t(\theta_t^*|\theta_{ih^{(n)}})}$ for the laws of $\theta_t$ and $\theta_t^*$. (A synchonous coupling simply means that we use the same Brownian motion $B_t$ in defining $\theta_t$ and $\theta_t^*$.) We then obtain that for any pair $(\theta_t, \theta_t^*)\sim\bar{\gamma}$,
\begin{align}
\frac{\rd \|\theta_t-\theta^*_t\|^p}{\rd t}
&= \|\theta_t-\theta^*_t\|^{p-2} \left\langle \theta_t-\theta^*_t, \frac{\rd \theta_t}{\rd t} - \frac{\rd \theta^*_t}{\rd t} \right\rangle \nonumber\\
&= p \|\theta_t-\theta^*_t\|^{p-2} \left\langle \theta_t-\theta^*_t, - \nabla \pot(\theta_{t}) + \nabla \pot(\theta^*_{t}) \right\rangle \nonumber\\
&+ p \|\theta_t-\theta^*_t\|^{p-2} \left\langle \theta_t-\theta^*_t, \nabla \pot(\theta_{t}) - \nabla \hpot(\theta_{ih^{(n)}}) \right\rangle \nonumber\\
&\leq - p \conv \lrn{\theta_t-\theta^*_t}^p
+ p \lrn{\theta_t-\theta^*_t}^{p-1} \lrn{\nabla \pot(\theta_{t}) - \nabla \hpot(\theta_{ih^{(n)}})}  \\
&\leq - p \conv \lrn{\theta_t-\theta^*_t}^p\\
& \qquad+ p \lrp{ \frac{p-1}{p} \lrp{\frac{p\conv}{2(p-1)}} \lrn{\theta_t-\theta^*_t}^p 
+ \frac{1}{p} \frac{1}{\lrp{\frac{p\conv}{2(p-1)}}^{p-1}} \lrn{\nabla \pot(\theta_{t}) - \nabla \hpot(\theta_{ih^{(n)}})}^p } \label{eq:Young}\\
&\leq - \frac{p\conv}{2} \lrn{\theta_t-\theta^*_t}^p
+ \frac{2^{p-1}}{\conv^{p-1}} \lrn{\nabla \pot(\theta_{t}) - \nabla \hpot(\theta_{ih^{(n)}})}^p,
\end{align}
where equation~\eqref{eq:Young} follows from Young's inequality.

Equivalently, we can obtain
\begin{align*}
\frac{\rd e^{\frac{p \conv}{2} t} \|\theta_t-\theta^*_t\|^p}{\rd t}
\leq e^{\frac{p \conv}{2} t}\frac{2^{p-1}}{\conv^{p-1}} \lrn{\nabla \pot(\theta_{t}) - \nabla \hpot(\theta_{ih^{(n)}})}^p.
\end{align*}
By the fundamental theorem of calculus,
\begin{align}
\|\theta_t-\theta^*_t\|^p
\leq e^{-\frac{p\conv}{2} \lrp{t-ih^{(n)}}} \lrn{\theta_{ih^{(n)}}-\theta^*_{ih^{(n)}}}^p
+ \frac{2^{p-1}}{\conv^{p-1}} \int_{i h^{(n)}}^t e^{ -\frac{p\conv}{2} (t-s) } \lrn{\nabla \pot(\theta_{s}) - \nabla \hpot(\theta_{ih^{(n)}})}^p \rd s.
\label{eq:int_bound}
\end{align}
Taking expectation on both sides, we obtain that
\begin{align}
\E{\|\theta_t-\theta^*_t\|^p} 
&= \E{ \E{\|\theta_t-\theta^*_t\|^p 
\ | \ \theta_{ih^{(n)}}, \theta^*_{ih^{(n)}}
} } \nonumber\\
&\leq e^{-\frac{p\conv}{2}\lrp{t-ih^{(n)}}} \E{\lrn{\theta_{ih^{(n)}}-\theta^*_{ih^{(n)}}}^p} \nonumber\\
&+ \frac{2^{p-1}}{\conv^{p-1}} \int_{i h^{(n)}}^t e^{ -\frac{p\conv}{2} (t-s) } \E{\lrn{\nabla \pot(\theta_{s}) - \nabla \hpot(\theta_{ih^{(n)}})}^p} \rd s.
\end{align}
In the above expression, the integral and expectation are exchanged using Tonelli's theorem, since
\[\lrn{\nabla \pot(\theta_{s}) - \nabla \hpot(\theta_{ih^{(n)}})}^p\]
is positive measurable.

We further expand the expected error $\E{\lrn{\nabla \pot(\theta_{s}) - \nabla \hpot(\theta_{ih^{(n)}})}^p}$:
\begin{align}
\lefteqn{ \E{\lrn{\nabla \pot(\theta_{s}) - \nabla \hpot(\theta_{ih^{(n)}})}^p} } \nonumber\\
&= \E{\lrn{\nabla \pot(\theta_{s}) - \nabla \pot(\theta_{ih^{(n)}})
+ \nabla \pot(\theta_{ih^{(n)}}) - \nabla \hpot(\theta_{ih^{(n)}})}^p} \nonumber\\
&\leq \frac{1}{2} \E{\lrn{2\lrp{ \nabla \pot(\theta_{s}) - \nabla \pot(\theta_{ih^{(n)}}) }}^p}
+ \frac{1}{2} \E{\lrn{2\lrp{ \nabla \pot(\theta_{ih^{(n)}}) - \nabla \hpot(\theta_{ih^{(n)}}) }}^p} \nonumber\\
&= 2^{p-1} \E{\lrn{ \nabla \pot(\theta_{s}) - \nabla \pot(\theta_{ih^{(n)}}) }^p}
+ 2^{p-1} \E{ \E{\lrn{ \nabla \pot(\theta_{ih^{(n)}}) - \nabla \hpot(\theta_{ih^{(n)}}) }^p \Big| \theta_{ih^{(n)}}} } \nonumber\\
&\leq 2^{p-1} \lip^p \cdot \E{\lrn{ \theta_{s} - \theta_{ih^{(n)}} }^p}
+ 2^{p-1} \Delta_p.
\end{align}
Plugging into equation~\eqref{eq:int_bound}, we have that
\begin{align}
\lefteqn{ \E{\|\theta_t-\theta^*_t\|^p} } \nonumber\\
&\leq e^{-\frac{p\conv}{2}\lrp{t-ih^{(n)}}} \E{\lrn{\theta_{ih^{(n)}}-\theta^*_{ih^{(n)}}}^p} \nonumber\\
&+ 2^{2p-2} \frac{\lip^p}{\conv^{p-1}} \int_{i h^{(n)}}^t e^{ -\frac{p\conv}{2} (t-s) } \E{\lrn{ \theta_{s} - \theta_{ih^{(n)}} }^p} \rd s
+ 2^{2p-2} (t-ih^{(n)}) \frac{\Delta_p}{\conv^{p-1}}.
\label{eq:dev_int_bound}
\end{align}
We provide an upper bound for $\int_{i h^{(n)}}^t e^{ -\frac{p\conv}{2} (t-s) } \E{\lrn{ \theta_{s} - \theta_{ih^{(n)}} }^p} \rd s$ in the following lemma.
\begin{lemma}
\label{lemma:dev_int_bound}
For $h^{(n)} \leq \frac{\conv}{32\lip^2}$, and for $t\in[ih^{(n)}, (i+1)h^{(n)}]$,
\begin{align}
\lefteqn{ \int_{i h^{(n)}}^t e^{ -\frac{p\conv}{2} (t-s) } \E{ \lrn{\theta_s-\theta_{ih^{(n)}}}^p \rd s } } \nonumber\\
&\leq 2^{3p-3} \lip^p \lrp{t-ih^{(n)}}^{p+1} W_p^p\lrp{\mu_{ih^{(n)}}, \mu^*} 
+ \frac{8^p}{2} \lrp{t-ih^{(n)}}^{p/2+1} \lrp{dp}^{p/2}
+ 2^{2p-2} (t-ih^{(n)})^{p+1} \cdot \Delta_p.
\end{align}
\end{lemma}
Applying this upper bound to equation~\eqref{eq:dev_int_bound}, we obtain that for $h^{(n)} \leq \frac{\conv}{32\lip^2}$, and for $t\in[ih^{(n)}, (i+1)h^{(n)}]$,
\begin{align}
\E{\|\theta_t-\theta^*_t\|^p}
&\leq e^{-\frac{p\conv}{2} \lrp{t-ih^{(n)}}} \E{\lrn{\theta_{ih^{(n)}}-\theta^*_{ih^{(n)}}}^p}
+ 2^{5p-5} \frac{\lip^{2p}}{\conv^{p-1}} \lrp{t-ih^{(n)}}^{p+1} W_p^p\lrp{\mu_{ih^{(n)}}, \mu^*} \nonumber\\
&+ 2^{5p-3} \frac{\lip^p}{\conv^{p-1}} \lrp{t-ih^{(n)}}^{p/2+1} \lrp{d p}^{p/2}
+ 2^{4p-4} \frac{\lip^p}{\conv^{p-1}} (t-ih^{(n)})^{p+1} \cdot \Delta_p \nonumber\\
& \qquad \qquad + 2^{2p-2} (t-ih^{(n)}) \frac{\Delta_p}{\conv^{p-1}} \nonumber\\
&\leq \lrp{1-\frac{\conv}{4} \lrp{t-ih^{(n)}}}^p \E{\lrn{\theta_{ih^{(n)}}-\theta^*_{ih^{(n)}}}^p}
+ 2^{5p-5} \frac{\lip^{2p}}{\conv^{p-1}} \lrp{t-ih^{(n)}}^{p+1} W_p^p\lrp{\mu_{ih^{(n)}}, \mu^*} \nonumber\\
&+ 2^{5p-3} \frac{\lip^p}{\conv^{p-1}} \lrp{t-ih^{(n)}}^{p/2+1} \lrp{d p}^{p/2}
+ 2^{2p} (t-ih^{(n)}) \frac{\Delta_p}{\conv^{p-1}}. \nonumber
\end{align}
Recognizing that $\displaystyle\widehat{\gamma}\lrp{\theta_t,\theta^*_t} = \Ep{\lrp{ \theta_{ih^{(n)}},\theta^*_{ih^{(n)}} }\sim\gamma^*}{\bar{\gamma}\lrp{ \theta_t,\theta^*_t | \theta_{ih^{(n)}},\theta^*_{ih^{(n)}} }}$ is a coupling, we achieve the upper bound for $W^p_p\lrp{\mu_t,\mu^*}$:
\begin{align}
W_p^p\lrp{\mu_{t}, \mu^*}
&\leq
\Ep{\lrp{\theta_t,\theta^*_t}\sim\widehat{\gamma}}{\|\theta_t-\theta^*_t\|^p} \nonumber\\
&\leq \lrp{1-\frac{\conv}{4} \lrp{t-ih^{(n)}}}^p \Ep{\lrp{ \theta_{ih^{(n)}},\theta^*_{ih^{(n)}} }\sim\gamma^*}{\lrn{\theta_{ih^{(n)}}-\theta^*_{ih^{(n)}}}^p}\nonumber\\
&\qquad +2^{5p-5} \frac{\lip^{2p}}{\conv^{p-1}} \lrp{t-ih^{(n)}}^{p+1} W_p^p\lrp{\mu_{ih^{(n)}}, \mu^*}+ 2^{5p-3} \frac{\lip^p}{\conv^{p-1}} \lrp{t-ih^{(n)}}^{p/2+1} \lrp{d p}^{p/2}\nonumber \\
& \qquad + 2^{2p} (t-ih^{(n)}) \frac{\Delta_p}{\conv^{p-1}}. \nonumber\\
&\leq \lrp{1-\frac{\conv}{8} \lrp{t-ih^{(n)}}}^p W_p^p\lrp{\mu_{ih^{(n)}}, \mu^*} 
+ 2^{5p-3} \frac{\lip^p}{\conv^{p-1}} \lrp{t-ih^{(n)}}^{p/2+1} \lrp{d p}^{p/2}\\
& \qquad + 2^{2p} (t-ih^{(n)}) \frac{\Delta_p}{\conv^{p-1}}.
\end{align}

Taking $t=(i+1)h^{(n)}$, the recurring bound reads
\begin{align*}
W_p^p\lrp{\mu_{(i+1)h^{(n)}}, \mu^*}
\leq \lrp{1 - \frac{\conv}{8} h^{(n)}}^p W_p^p\lrp{\mu_{ih^{(n)}}, \mu^*}
+ 2^{5p-3} \frac{\lip^p}{\conv^{p-1}} \lrp{dp}^{p/2} \lrp{h^{(n)}}^{p/2+1}
+ \frac{4^p}{\conv^{p-1}} h^{(n)} \Delta_p.
\end{align*}
We finish the proof by invoking the recursion $i$ times:
\begin{align}
W_p^p\lrp{\mu_{ih^{(n)}}, \mu^*}
&\leq \lrp{1 - \frac{\conv}{8} h^{(n)}}^p W_p^p\lrp{\mu_{(i-1)h^{(n)}}, \mu^*}
+ 2^{5p-3} \frac{\lip^p}{\conv^{p-1}} \lrp{dp}^{p/2} \lrp{h^{(n)}}^{p/2+1}
+ \frac{4^p}{\conv^{p-1}} h^{(n)} \Delta_p
\nonumber\\
&\leq \lrp{1 - \frac{\conv}{8} h^{(n)}}^{p \cdot i} W_p^p\lrp{\mu_{0}, \mu^*} \nonumber\\
&+ \sum_{k=0}^{i-1} \lrp{1 - \frac{\conv}{8} h^{(n)}}^{p \cdot k} 
\cdot \lrp{2^{5p-3} \frac{\lip^p}{\conv^{p-1}} \lrp{dp}^{p/2} \lrp{h^{(n)}}^{p/2+1}
+ \frac{4^p}{\conv^{p-1}} h^{(n)} \Delta_p} \nonumber\\
&\leq \lrp{1 - \frac{\conv}{8} h^{(n)}}^{p \cdot i} W_p^p\lrp{\mu_{0}, \mu^*}
+ 2^{5p} \frac{\lip^{p}}{\conv^{p}} \lrp{dp}^{p/2} \lrp{h^{(n)}}^{p/2}
+ 2^{2p+3} \frac{\Delta_p}{\conv^p}.
\end{align}

\end{proof}

\subsubsection{Supporting proofs for Lemma~\ref{lemma:ULA_cvg}}
\begin{proof}[Proof of Lemma~\ref{lemma:dev_int_bound}]
We use the update rule of ULA to develop $\int_{i h^{(n)}}^t e^{ -\frac{p\conv}{2} (t-s) } \E{\lrn{ \theta_{s} - \theta_{ih^{(n)}} }^p} \rd s$: 
\begin{align}
\lefteqn{ \int_{i h^{(n)}}^t e^{ -\frac{p\conv}{2} (t-s) } \E{ \lrn{\theta_s - \theta_{ih^{(n)}}}^p \rd s } } \nonumber\\
&= \int_{i h^{(n)}}^t e^{ -\frac{p\conv}{2} (t-s) } \E{ \lrn{ -  (s-ih^{(n)})\left( \nabla \pot(\theta_{ih^{(n)}}) - \lrp{ \nabla \pot(\theta_{ih^{(n)}}) - \nabla \hpot(\theta_{ih^{(n)}}) }\right)+ \sqrt{2}(B_s-B_{ih^{(n)}}) }^p } \rd s \nonumber\\
&\leq 2^{2p-2} (t-ih^{(n)})^p \int_{i h^{(n)}}^t e^{ -\frac{p\conv}{2} (t-s) } \E{ \lrn{ \nabla \pot(\theta_{ih^{(n)}})}^p } \rd s \nonumber\\
&+ 2^{3p/2-1} \int_{i h^{(n)}}^t e^{ -\frac{p\conv}{2} (t-s) } \E{ \lrn{ B_s-B_{ih^{(n)}} }^p } \rd s \nonumber\\
&+ 2^{2p-2} (t-ih^{(n)})^p \int_{i h^{(n)}}^t e^{ -\frac{p\conv}{2} (t-s) } \E{ \lrn{ \nabla \pot(\theta_{ih^{(n)}}) - \nabla \hpot(\theta_{ih^{(n)}}) }^p } \rd s \nonumber\\
&\leq 2^{2p-2} \lip^p \lrp{t-ih^{(n)}}^{p+1} \E{ \lrn{ \theta_{ih^{(n)}} - \theta_U^* }^p }
+ 2^{3p/2-1} \int_{i h^{(n)}}^t \E{ \lrn{ B_s-B_{ih^{(n)}} }^p } \rd s\nonumber\\
& \qquad \qquad+ 2^{2p-2} \lrp{t-ih^{(n)}}^{p+1} \Delta_p. 
\label{eq:ULA_dev_bound} 
\end{align}
where $\theta_U^*$ is the fixed point of $\pot$.
We then use the following lemma to simplify the above expression.
\begin{lemma}
\label{lemma:Brownian_moment_bound}
The integrated $p$-th moment of the Brownian motion can be bounded as:
\begin{align}
\int_{i h^{(n)}}^t 
\mathbb{E} \lrn{ B_s-B_{ih^{(n)}} }^p \rd s
\leq 2 \lrp{\frac{dp}{e}}^{p/2} \lrp{t-ih^{(n)}}^{p/2+1}.
\end{align}
\end{lemma}
We also provide bound for the $p$-th moment of $\lrn{ \theta_{ih^{(n)}} - \theta_U^* }$.
\begin{lemma}
\label{lemma:theta_diff_bound}
For $\theta_{ih^{(n)}} \sim \mu_{ih^{(n)}}$, 
\begin{align}
\mathbb{E} \lrn{ \theta_{ih^{(n)}} - \theta_U^* }^p
\leq 2^{p-1} W_p^p\lrp{\mu_{ih^{n}}, \mu^*} 
+ \frac{10^p}{2} \lrp{\frac{dp}{\conv}}^{p/2}.
\end{align}
\end{lemma}
Plugging the results into equation~\eqref{eq:ULA_dev_bound}, we obtain that for $h^{(n)} \leq \frac{\conv}{32\lip^2}$, and for $t\in[ih^{(n)}, (i+1)h^{(n)}]$,
\begin{align}
\lefteqn{ \int_{i h^{(n)}}^t e^{ -\frac{p\conv}{2} (t-s) } \E{ \lrn{\theta_s-\theta_{ih^{(n)}}}^p \rd s } } \nonumber\\
&\leq 2^{3p-3} \lip^p \lrp{t-ih^{(n)}}^{p+1} W_p^p\lrp{\mu_{ih^{n}}, \mu^*} 
+ \frac{40^p}{8} \lip^p \lrp{t-ih^{(n)}}^{p+1} \lrp{\frac{dp}{\conv}}^{p/2} \nonumber\\
&+ \lrp{ \frac{8}{e} }^{p/2} \lrp{d p}^{p/2} \lrp{t-ih^{(n)}}^{p/2+1}
+ 2^{2p-2} (t-ih^{(n)})^{p+1} \cdot \Delta_p \nonumber\\
&\leq 2^{3p-3} \lip^p \lrp{t-ih^{(n)}}^{p+1} W_p^p\lrp{\mu_{ih^{n}}, \mu^*} 
+ \frac{8^p}{2} \lrp{t-ih^{(n)}}^{p/2+1} \lrp{dp}^{p/2}
+ 2^{2p-2} (t-ih^{(n)})^{p+1} \Delta_p.
\end{align}

\end{proof}

\begin{proof}[Proof of Lemma~\ref{lemma:Brownian_moment_bound}]
The Brownian motion term can be upper bounded by higher moments of a normal random variable: 
\[
\int_{i h^{(n)}}^t 
\mathbb{E} \lrn{ B_s-B_{ih^{(n)}} }^p \rd s
\leq \lrp{t-ih^{(n)}} \mathbb{E} \lrn{ B_t-B_{ih^{(n)}} }^p
= \lrp{t-ih^{(n)}}^{p/2+1} \mathbb{E} \lrn{ v }^p,
\]
where $v$ is a standard $d$-dimensional normal random variable.
We then invoke the $\sqrt{d}$ sub-Gaussianity of $\lrn{ v }$ and have (assuming $p$ to be an even integer):
$$
\mathbb{E} \lrn{ v }^p
\leq \frac{p!}{2^{p/2}\lrp{p/2}!} d^{p/2} %
\leq \frac{ e^{1/12p} \sqrt{2\pi p} (p/e)^p }{ 2^{p/2} \sqrt{\pi p} (p/2e)^{p/2} } d^{p/2}
\leq 2 \lrp{\frac{dp}{e}}^{p/2}.
$$
\end{proof}

\begin{proof}[Proof of Lemma~\ref{lemma:theta_diff_bound}]
For the $\mathbb{E} \lrn{ \theta_{ih^{(n)}} - \theta_U^* }^p$ term, we note that any coupling of a distribution with a delta measure is their product measure.
Therefore, $\mathbb{E} \lrn{ \theta_{ih^{(n)}} - \theta_U^* }^p$ relates to the $p$-Wasserstein distance between $\mu_{ih^{(n)}}$ and the delta measure at the fixed point $\theta_U^*$, $\delta\lrp{\theta_U^*}$:
\begin{align*}
\mathbb{E} \lrn{ \theta_{ih^{(n)}} - \theta_U^* }^p
= W_p^p \lrp{\mu_{ih^{(n)}}, \delta\lrp{\theta_U^*}}
&\leq \lrp{W_p\lrp{\mu_{ih^{(n)}}, \mu^*} + W_p \lrp{\mu^*, \delta\lrp{\theta_U^*}} }^p \\
&\leq 2^{p-1} W_p^p\lrp{\mu_{ih^{(n)}}, \mu^*} + 2^{p-1} W_p^p\lrp{\mu^*, \delta\lrp{\theta_U^*}}.
\end{align*}
We then bound $W_p^p\lrp{\mu^*, \delta\lrp{\theta_U^*}}$ in the following lemma.
\begin{lemma}
\label{lemma:W_p_fixed_point}
Assume the posterior $\mu^*$ is $\conv$-strongly log-concave.
Then for $\theta_U^* = \arg\max \mu^*$,
\begin{align}
W_p^p \lrp{\mu^*, \delta\lrp{\theta_U^*}} 
\leq 5^p \lrp{\frac{dp}{\conv}}^{p/2}.
\end{align}
\end{lemma}
Therefore, 
\begin{align*}
\mathbb{E} \lrn{ \theta_{ih^{(n)}}^{(n)} - \theta^*_n }^p
\leq 2^{p-1} W_p^p\lrp{\mu_{ih^{(n)}}, \mu^*} 
+ \frac{10^p}{2} \lrp{\frac{dp}{\conv}}^{p/2}.
\end{align*}
\end{proof}

\begin{proof}[Proof of Lemma~\ref{lemma:W_p_fixed_point}]
We first decompose $W_p\lrp{\mu^*, \delta\lrp{\theta_U^*}}$ into two terms:
$$W_p\lrp{\mu^*, \delta\lrp{\theta_U^*}}
\leq W_p\lrp{\mu^*, \delta\lrp{\Ep{\theta\sim\mu^*}{\theta}}} + \lrn{\theta_U^*-\Ep{\theta\sim\mu^*}{\theta}}.$$
By the celebrate relation between mean and mode for $1$-unimodal distributions~\citep[see, e.g.,][Theorem $7$]{Basu_96}, we can first bound the difference between mean and mode: 
$$
\lrp{\theta_U^*-\Ep{\theta\sim\mu^*}{\theta}}^\rT \Sigma^{-1} \lrp{\theta_U^*-\Ep{\theta\sim\mu^*}{\theta}}
\leq 3.
$$
where $\Sigma$ is the covariance matrix of $\mu^*$.
Therefore, 
\begin{align}
\lrn{\theta_U^*-\Ep{\theta\sim\mu^*}{\theta}}^2 \leq \frac{3}{\conv}.
\label{eq:mean_mode}
\end{align}

We then bound $W_p\lrp{\mu^*, \delta\lrp{\Ep{\theta\sim\mu^*}{\theta}}}$. 
Since the coupling between $\mu^*$ and the delta measure $\delta\lrp{\Ep{\theta\sim\mu^*}{\theta}}$ is their product measure, we can directly obtain that the $p$-Wasserstein distance is the $p$-th moments of $\mu^*$:
$$
W_p^p\lrp{\mu^*, \delta\lrp{\Ep{\theta\sim\mu^*}{\theta}}} 
= \int \lrn{\theta-\Ep{\theta\sim\mu^*}{\theta}}^p \rd \mu^*(\theta).
$$
We invoke the Herbst argument~\citep[see, e.g.,][]{ledoux1999concentration} to obtain the $p$-th moment bound.
We first note that for an $\conv$-strongly log-concave distribution, it has a log Sobolev constant of $\conv$.
Then using the Herbst argument, we know that $x\sim\mu^*$ is a sub-Gaussian random vector with parameter $\sigma^2=\frac{1}{2\conv}$:
\[
\int e^{\lambda u^\rT\lrp{\theta-\Ep{\theta\sim\mu^*}{\theta}}} \rd \mu^*(\theta)
\leq e^{\frac{\lambda^2}{4\conv}},
\quad \forall \lrn{u}=1.
\]
Hence $\theta$ is $2\sqrt{\frac{d}{\conv}}$ norm-sub-Gaussian, which implies that 
\begin{align}
\lrp{\Ep{\theta\sim\mu^*}{\lrn{\theta-\Ep{\theta\sim\mu^*}{\theta}}^p}}^{1/p}
\leq 2 e^{1/e} \sqrt{\frac{dp}{\conv}}.
\label{eq:mean_var}
\end{align}

Combining equations~\eqref{eq:mean_mode} and~\eqref{eq:mean_var}, we obtain the final result that
\begin{align*}
W_p^p\lrp{\mu^*, \delta\lrp{\theta_U^*}} 
&\leq \lrp{ 2 e^{1/e} \sqrt{\frac{dp}{\conv}} + \sqrt{\frac{3}{\conv}} }^p \\
&\leq 5^p \lrp{\frac{dp}{\conv}}^{p/2}.
\end{align*}

\end{proof}

\begin{lemma}
\label{lemma:W_p_delta_ULA}
Assume that the likelihood $\log p_a(x;\theta)$, prior distribution, and true distributions satisfy Assumptions~1-3, and that arm $a$ has been chosen $n=T_a(t)$ times up to iteration $t$ of the Thompson sampling algorithm. Further, assume that we choose the stepsize step size 
$h^{(n)} =
\frac{1}{32} \frac{m_a}{n \lrp{L_a+\frac{1}{n}L_a}^2} 
= \mathcal{O}\lrp{ \frac{m_a}{n  L_a^2} 
}$
, and number of steps 
$N = 640 \frac{\lrp{L_a+\frac{1}{n}L_a}^2}{m_a^2} 
= \mathcal{O}\lrp{ \frac{L_a^2}{m_a^2} }$
in Algorithm~\ref{alg:Approximate_sampling} then:

\[ 
\mathbb{P}_{\theta_{a,t} \sim \bar \mu^{(n)}_{a} [\gamma_a]} \left(\|\theta_{a,t}-\theta_a^\ast\|_2 > \sqrt{\frac{36e}{m_a n} \left( d_a+\log B_a +2\sigma\log{1/\delta_1}+2\left(\sigma_a+\frac{m_a d_a}{18L_a \gamma_a}\right)\log{1/\delta_2} \right)} \bigg| Z_{n-1} \right)<\delta_2. 
\]

where $Z_{t-1}=\{ \| \theta_{a,t-1}-\theta^*_a\| \le C(n) \}$ for:

\[ C(n)= \sqrt{\frac{18e}{n m_a}} \left( d_a+\log B_a+2 \sigma \log1/\delta_1\right)^{\frac{1}{2}}, \]

$\sigma=16+\frac{4d_aL_a^2}{\nu_a m_a}$, and  where $\theta_{a,t-1}$ is the sample from the previous round of the Thompson sampling algorithm for arm $a$.

\end{lemma}

\begin{proof}

We begin as in the proof of Theorem~\ref{theorem3}, except that we now take $\mu_0=\delta_{\theta_{a,t-1}}$, where $\theta_{a,t-1}$ is the sample from the previous step of the algorithm:

\begin{align*}
W_p^p\lrp{\mu_{ih^{(n)}}, \mu_a^{(n)}}
\leq \lrp{1 - \frac{\conv}{8} h^{(n)}}^{p \cdot i} W_p^p\lrp{\delta(\theta_{a,t-1}), \mu_a^{(n)}}
+ \frac{80^{p}}{2} \frac{\lip^{p}}{\conv^{p}} \lrp{dp}^{p/2} \lrp{h^{(n)}}^{p/2}.
\end{align*}

We first use the triangle inequality on the first term on the RHS:

\begin{align*}
    W_p\lrp{\delta(\theta_{a,t-1}), \mu_a^{(n)}} &\le W_p\lrp{\delta(\theta_{a,t-1}), \delta_{\theta^*_a}}+W_p\lrp{\delta(\theta^*_a),\mu_a^{(n)}}\\
    &= \| \theta_a^*-\theta_{a,t-1}\| + +W_p\lrp{\delta(\theta^*_a),\mu_a^{(n)}}\\
    &\le C(n) + \frac{\tilde D}{\sqrt{n}}
\end{align*}

where we have used the fact that  $\| \theta_a^*-\theta_{a,t-1}\| \le C(n)$ by assumption, and the definition of $\tilde D$ from the proof of Theorem~\ref{theorem_ULA}: $\widetilde{D} = \sqrt{\frac{2}{m_a}} \left( d_a+\log B_a+\sigma p\right)^{\frac{1}{2}}$.

Since:
\[ C(n)=  \sqrt{\frac{18e}{m_a}} \left( d_a+\log B_a+2\sigma\log1/\delta_1\right)^{\frac{1}{2}}, \]

We can further develop this upper bound:
\begin{align*}
    W_p\lrp{\delta_{\theta_{a,t-1}}, \mu_a^{(n)}}&\le \frac{\tilde D}{\sqrt{n}}+C(n)\\
    & \le 8\sqrt{\frac{2}{m_a n}}\left( d_a+\log B_a+2\sigma \log 1/\delta_1 +\sigma p \right)^{\frac{1}{2}},
\end{align*}

where to derive this result we have used the fact that $\sqrt{2(x+y)}\ge \sqrt{x}+\sqrt{y}$.

Letting $\bar D =  \sqrt{\frac{2}{m_a n}}\left( d_a+\log B_a+2\sigma \log 1/\delta_1 +\sigma p \right)^{\frac{1}{2}}$, we see that our final result is:

\[W_p\lrp{\delta_{\theta_{a,t-1}}, \mu_a^{(n)}} \le \frac{8}{\sqrt{n}}\bar D,\]

where $\tilde D<\bar D$. Using the same choice of $h^{(n)}$ and number of steps $N$ as in  the proof or Theorem~\ref{theorem_ULA} guarantees us that:

\[ W_p^p\lrp{\mu_{ih^{(n)}}, \mu_a^{(n)}} \le  2\left( \frac{\bar D}{\sqrt{n}} \right)^p \]

Further combining this with the triangle inequality, and the fact that $\tilde D<\bar D$ gives us that:

\[ W_p\lrp{\mu_{ih^{(n)}}, \delta_{\theta^*}} \le \frac{\tilde D}{\sqrt{n}}+ \frac{\bar D}{\sqrt{n}} \le 3\frac{\bar D}{\sqrt{n}}, \]

Now, since the sample returned by the Langevin algorithm is given by:
\begin{align} 
\theta_{a}=\theta_N+Z,
\end{align}
where $Z\sim \mc{N}\left(0,\frac{1}{nL_a\gamma_a} I \right)$, it remains to bound the distance between the approximate posterior $\hat \mu_a^{(n)}$ of $\theta_a$ and the distribution of $\theta_{N h^{(n)} }$. 
Since $\theta_{a} - \theta_{N h^{(n)} } = Z$, for any even integer $p$,
\begin{align*}
W_p^p\lrp{\bar \mu_{a}^{(n)} , \bar \mu_a^{(n)} [\gamma_a]}
= \lrp{ \inf_{ \gamma\in\Gamma\lrp{ \bar \mu_{a}^{(n)} , \bar \mu_a^{(n)} [\gamma_a]}} \int \lrn{\theta_a - \theta_N}^p \rd \theta_a \rd \theta_N }^{1/p} 
&\le \mb{E}[\lrn{Z}^p]^{\frac{1}{p}}\\
&\le \sqrt{\frac{d}{nL_a\gamma_a}} \lrp{ \frac{2^{p/2}\mathbf{\Gamma}(\frac{p+1}{2})}{\sqrt{\pi}} }^{1/p} \\
&\le \sqrt{\frac{d}{nL_a\gamma_a}} \lrp{ 2^{p/2} \lrp{\frac{p}{2}}^{p/2}  }^{1/p}\\
&\le \sqrt{\frac{dp}{nL_a\gamma_a}},
\end{align*}                 
where we have used upper bound of the Stirling type for the Gamma function $\mathbf{\Gamma}(\cdot)$ in the second last inequality.

Thus, we have, via the triangle inequality once again, that:
\begin{align*}
     W_p\lrp{\bar \mu_a^{(n) [\gamma_a]}, \delta_{\theta^*}} &\le  3\frac{\bar D}{\sqrt{n}}+\sqrt{\frac{dp}{nL_a\gamma_a}}\\
     &\le \sqrt{\frac{36}{m_a n}}\left( d_a+\log B_a+2\sigma_a \log 1/\delta_1 +\left(\sigma_a+\frac{d_a}{18L_a \gamma_a}\right) p \right)^{\frac{1}{2}},
\end{align*}
which, by the same derivation as in the proof of Theorem~\ref{thm:theorem1}, gives us that:
\[ 
\mathbb{P}_{\theta_{a,t} \sim \bar \mu^{(n)}_{a} [\gamma_a]} \left(\|\theta_{a,t}-\theta_a^\ast\|_2 > \sqrt{\frac{36e}{m_a n} \left( d_a+\log B_a +2\sigma\log{1/\delta_1}+2\left(\sigma_a+\frac{m_a d_a}{18L_a \gamma_a}\right)\log{1/\delta_2} \right)} \bigg| Z_{n-1} \right)<\delta_2. 
\]

\end{proof}

We remark that via an identical argument, the following Lemma holds as well:

\begin{lemma}
\label{lemma:W_p_delta_SGLD}
Assume that the family $\log p_a(x;\theta)$ and the prior $\pi_a$ satisfy Assumptions~1-\ref{assumption:sg_lip} and that arm $a$ has been chosen $n=T_a(t)$ times up to iteration $t$ of the Thompson sampling algorithm.
If we take number of data samples in the stochastic gradient estimate
$k=32\frac{\lrp{L_a^*}^2}{m_a \nu_a}$,
step size 
$h^{(n)} =
\frac{1}{32} \frac{m_a}{n \lrp{L_a+\frac{1}{n}L_a}^2} 
= \mathcal{O}\lrp{ \frac{m_a}{n  L_a^2} 
}$
and number of steps 
$N = 1280 \frac{\lrp{L_a+\frac{1}{n}L_a}^2}{m_a^2} = \mathcal{O}\lrp{ \frac{L_a^2}{m_a^2} }$ in Algorithm~\ref{alg:Approximate_sampling}, then:
\[ \mathbb{P}_{\theta_{a,t} \sim \bar \mu^{(n)}_{a} [\gamma_a]} \left(\|\theta_{a,t}-\theta_a^\ast\|_2 > \sqrt{\frac{36e}{m_a n} \left( d_a+\log B_a +2\sigma\log{1/\delta_1}+2\left(\sigma_a+\frac{m_a d_a}{18L_a \gamma_a}\right)\log{1/\delta_2} \right)} \bigg| Z_{n-1} \right)<\delta_2. , \]
where $Z_{t-1}=\{ \| \theta_{a,t-1}-\theta^*_a\| \le C(n) \}$ for the parameters:
\[ C(n)= \sqrt{\frac{18e}{n m_a}} \left( d_a+\log B_a+2 \sigma \log1/\delta_1\right)^{\frac{1}{2}}, 
\qquad
\sigma=16+\frac{4d_aL_a^2}{\nu_a m_a},
\]
and $\theta_{a,t-1}$ being the sample from the previous round of the Thompson sampling algorithm over arm $a$.

\end{lemma} %
\section{Regret Proofs}\label{section::appendix_regret_proofs}

We now present the proof of logarithmic regret of Thompson sampling under our assumptions with samples from the true posterior and from the approximate sampling schemes discussed in Section~\ref{sec:approx}. To provide the regret guarantees for Thompson sampling with samples from the \emph{true} posterior and from approximations to the posterior, we proceed as is common in regret proofs for multi-armed bandits by upper-bounding the number of times a sub-optimal arm $a\in \mathcal{A}$ is pulled up to time $T$, denoted $T_a(T)$. Without loss of generality we assume throughout this section that arm $1$ is the optimal arm, and define the filtration associated with a run of the algorithm as $\mc{F}_{t}=\{A_1,X_1,A_2,X_2,...,A_t,X_t\}$.

To upper bound the expected number of times a sub-optimal arm is pulled up to time $T$, we first define the event $E_a(t)=\{ r_{a,t}(T_a(t))\ge \bar r_1-\epsilon \}$ for some $\epsilon>0$. This captures the event that the mean calculated from the value of $\theta_a$ sampled from the posterior at time $t\le T$, $r_{a,t}(T_a(t))$, is greater than $\bar r_1-\epsilon$ (recall $\bar r_1$ is the optimal arm's mean). Given these events, we proceed to decompose the expected number of pulls of a sub-optimal arm $a \in \mathcal{A}$ as:

\begin{align}
    \mb{E}[T_a(T)]&=\mb{E}\left[\sum_{t=1}^T\mb{I}(A_t=a)\right]= \underbrace{\mb{E}\left [\sum_{t=1}^T\mb{I}(A_t=a, E^c_a(t))\right]}_{I}+\underbrace{\mb{E}\left [\sum_{t=1}^T\mb{I}(A_t=a, E_a(t))\right]}_{II}.
\end{align}

In Lemma~\ref{lemma:termI} we upper bound $(I)$, and then bound  term $(II)$ in Lemmas ~\ref{lemma:termII_exact}.

We note that this proof follows a similar structure to that of the regret bound for Thompson sampling for Bernoulli bandits and bounded rewards in \citep{bernoulliBandits}. However, to give the regret guarantees that incorporate the quality of the priors as well as the potential errors and lack of independence resulting from the approximate sampling methods we discuss in Section~\ref{sec:approx} the proof is more complex.

\begin{lemma}[Bounding I]\label{lemma:termI}
    For a sub-optimal arm $a\in \mathcal{A}$, we have that:
    
    \[I = \mb{E}\left [\sum_{t=1}^T\mb{I}(A_t=a, E^c_a(t))\right] \le \mb{E}\left [\sum_{s=\ell}^{T-1} \frac{1}{p_{1,s}}-1\right].  \]
    
    where $p_{a,s}=\mb{P}(r_{a,t}(s)>\bar r_1-\epsilon |\mc{F}_{t-1})$, for some $\epsilon>0$. 
\end{lemma}

\begin{proof}
To bound term I of \eqref{eq:regret_decomp1}, we first recall $A_t$ is the arm achieving the largest sample reward mean at round $t$.  Further, we define $A_t'$ to be the arm achieving the maximum sample mean value among all the suboptimal arms:
\[ A'_t=\argmax_{a \in \mc{A}, a\ne1} r_{a}(t,T_a(t)).\]
Since $\mb{E}\left[  \mb{I}( A_t = a, E^c_a(t))  \right] = \mb{P}\left(  A_t=a, E^c_a(t)    \right)$, we aim to bound $\mathbb{P}(A_t = a, E^c_a(t) | \mc{F}_{t-1})$. We note that the following inequality holds:
\begin{align}
\mb{P}(A_t=a,E^c_a(t)|\mc{F}_{t-1}) &\le \mb{P}(A'_t=a,E^c_a(t)|\mc{F}_{t-1})(\mb{P}(r_1(t, T_1(t))\le \bar r_1-\epsilon |\mc{F}_{t-1}))  \notag \\
&= \mb{P}(A'_t=a,E^c_a(t)|\mc{F}_{t-1})(1-\mb{P}(E_1(t) |\mc{F}_{t-1})). \label{eq:thomson_regret_helper_bound1}
\end{align}
We also note that the term $\mb{P}(A'_t=a,E^c_a(t)|\mc{F}_{t-1})$ can be bounded as follows:
\begin{align}
    \mb{P}(A_t=1,E^c_a(t)|\mc{F}_{t-1})&\stackrel{(i)}{\ge} \mb{P}(A'_t=a,E^c_a(t),E_1(t) |\mc{F}_{t-1})\notag\\
    &=\mb{P}(A'_t=a,E^c_a(t)|\mc{F}_{t-1})\mb{P}(E_1(t). |\mc{F}_{t-1})\label{eq:thomson_regret_helper_bound2}
\end{align}

Inequality $(i)$ holds because $ \{ A_t' = a, E^c_a(t), E_1(t) \} \subseteq \{A_t=1, E^c_a(t), E_1(t) \}$. The equality is a consequence of the conditional independence of $E_1(t)$ and $\{ A_t'=a, E^c_a(t)\}$ (conditioned on $\mathcal{F}_{t-1}$). \footnote{The conditional independence property holds for all of our sampling mechanisms because the sample distributions for the two distinct arms $(a,1)$ are always conditionally independent on $\mathcal{F}_{t-1}$}

Assuming $\mathbb{P}(E_1(t) | \mathcal{F}_{t-1}) >0$ and\footnote{In all the cases we consider, including approximate sampling schemes, this property holds. In that case, since the Gaussian noise in the Langevin diffusion ensures all sets of the form $(a,b)$ have nonzero probability mass.} putting inequalities \ref{eq:thomson_regret_helper_bound1} and \ref{eq:thomson_regret_helper_bound2} together gives the following upper bound for $\mathbb{P}(A_t = a, E^c_a(t)| \mc{F}_{t-1})$:
\begin{align*}
    \mb{P}(A_t = a, E^c_a(t) | \mc{F}_{t-1}) \leq \mb{P}(A_t=1, E^c_a(t) | \mc{F}_{t-1})\left( \frac{1- \mb{P}( E_1(t) | \mc{F}_{t-1}) }{\mb{P}(E_1(t) | \mc{F}_{t-1}) } \right).
\end{align*}

Letting $P(E_1(t)|\mc{F}_{t-1}):=p_{1,T_1(t)}$ and noting that$\{A_t = 1, E^c_a(t)\} \subseteq \{ A_t = 1 \}$ :
\begin{align}\label{eq:upperbound0}
    \mb{P}(A_t=a,E^c_a(t)|\mc{F}_{t-1}) &\le \mb{P}(A_t=1|\mc{F}_{t-1}) \left(\frac{1}{p_{1,T_1(t)}}-1\right).
\end{align}

Now, we use this to give an upper bound on the term of interest:

\begin{align*}
    \mb{E}\left [\sum_{t=1}^T\mb{I}(A_t=a, E^c_a(t))\right]
    &\stackrel{(i)}{=}  \mb{E}\left [\sum_{t=1}^T\mb{E}\left[\mb{I}\left(A_t=a, E^c_a(t)\right) |\mathcal{F}_{t-1}\right]\right] \\
    &\stackrel{(ii)}{=}  \mb{E}\left [\sum_{t=1}^T\mb{P}\left(A_t=a, E^c_a(t)|\mathcal{F}_{t-1}\right) \right] \\
    &\stackrel{(iii)}{\le} \mb{E}\left [\sum_{t=1 }^T  \mb{P}(A_t=1|\mc{F}_{t-1}) \left(\frac{1}{p_{1,T_1(t)}}-1\right) \right]\\
    &\stackrel{(iv)}{=} \mb{E}\left [\sum_{t=1 }^T \mb{E}\left[\mb{I}(A_t=1)|\mc{F}_{t-1}\right] \left(\frac{1}{p_{1,T_1(t)}}-1\right)\right]\\
    &\stackrel{(v)}{=}  \mb{E}\left [\sum_{t=1}^T \mb{I}(A_t=1) \left(\frac{1}{p_{1,T_1(t)}}-1\right)\right]\\
    &\stackrel{(vi)}{\le} \mb{E}\left [\sum_{s=1}^{T-1} \frac{1}{p_{1,s}}-1\right].
\end{align*}

Here the equality $(i)$ is a consequence of the tower property, and equality $(ii)$ by noting that $\mb{E}\left[\mb{I}\left(A_t=a, E^c_a(t)\right) |\mathcal{F}_{t-1}\right] = \mb{P}\left(A_t=a, E^c_a(t) |\mathcal{F}_{t-1}\right)$. Inequality $(iii)$ follows by from Equation \ref{eq:upperbound0}, and equality $(iv)$ follows by definition. Finally, equality $(v)$ follows by the tower property and the last line each the fact that  $T_1(t)=s$ and $A_t=1$ can only happen once for every $s=1,...,T$. This completes the proof.
\end{proof}

Given the bound on $(I)$ from \eqref{eq:regret_decomp1}, we now present the tighter of two bounds on $(II)$ which is used to provide regret guarantees for Thompson sampling with exact samples from the posteriors.

\begin{lemma}[Bounding II - exact posterior]\label{lemma:termII_exact}
    For a sub-optimal arm $a\in \mathcal{A}$, we have that:
    
    \[II = \mb{E}\left [\sum_{t=1}^T\mb{I}(A_t=a, E_a(t))\right]\le 1+\mb{E}\left [\sum_{s=1}^T\mb{I}\left (p_{a,s}>\frac{1}{T}\right )\right].  \]
    
    where  $p_{a,s}=\mb{P}(r_{a,t}(s)>\bar r_1-\epsilon |\mc{F}_{t-1})$, for some $\epsilon>0$. 
\end{lemma}

\begin{proof}
The upper bound for term $II$ in \eqref{eq:regret_decomp1} follows the exact same proof as in \citep{bernoulliBandits}, and we recreate it for completeness below. Let $\mathcal{T}=\{t: p_{a,T_a(t)}>\frac{1}{T}\}$, then:

\begin{align}
\label{eq:regretlemmat2}
    \mb{E}\left [\sum_{t=1}^T\mb{I}(A_t=a, E_a(t))\right]\le   \underbrace{\mb{E}\left [\sum_{t \in \mathcal{T}}\mb{I}(A_t=a)\right]}_{I}+\underbrace{\mb{E}\left [\sum_{t \notin \mathcal{T}}\mb{I}(E_a(t))\right]}_{II}
\end{align}

By definition, term $I$ in \eqref{eq:regretlemmat2} satisfies:
\begin{align*}
 \sum_{t \in \mathcal{T}}\mb{I}(A_t=a) &=  \sum_{t \in \mathcal{T}}\mb{I}\left(A_t=a, p_{a,T_a(t)}>\frac{1}{T}\right)\le \sum_{s=1}^T \mb{I}\left(p_{a,s}>\frac{1}{T}\right)
\end{align*}

To address term $II$ in \eqref{eq:regretlemmat2}, we note that, by definition: $\mb{E}[\mb{I}(E_a(t))|\mathcal{F}_{t-1}]=p_{a,T_a(t)}$. Therefore, using the definition of the set of times $\mathcal{T}$, we can construct this simple upper bound:
\begin{align*}
    \mb{E}\left[\sum_{t \notin \mathcal{T}}\mb{I}(E_a(t))\right]&=\mb{E}\left[\sum_{t \notin \mathcal{T}}\mb{E}\left[\mb{I}\left(E_a(t)\right)| \mc{F}_{t-1}\right]\right]\\
    &=\mb{E}\left[ \sum_{t \notin \mathcal{T}} p_{a,t}\right]\\
    &\le \sum_{t \notin \mathcal{T}} \frac{1}{T} \\
    &\le 1
\end{align*}
Using the two upper bounds for terms $I$ and $II$ in \eqref{eq:regretlemmat2} gives out desired result:

\[ \mb{E}\left [\sum_{t=1}^T\mb{I}(A_t=a, E_a(t))\right]\le 1+\mb{E}\left [\sum_{s=1}^T\mb{I}\left (p_{a,s}>\frac{1}{T}\right )\right] \]

\end{proof}

\subsection{Regret of Exact Thompson Sampling}

 We now present two  technical lemmas for use in the proof of the regret of exact Thompson sampling. The first technical lemma, provides a lower bound on the probability of an arm begin optimistic in terms of the quality of the prior: 










\begin{lemma}
Suppose the likelihood and reward distributions satisfy Assumptions 1-3, then for all $n=1,...,T$ and $\gamma_1=\frac{\nu_1 m_1^2}{8d_1L_1^3}$:
\[ \mb{E}\left[\frac{1}{p_{1,n}}\right]\le 64\sqrt{\frac{L_1}{m_1}B_1} \]
\label{lemma:anti_conc_exact}
\end{lemma}

\begin{proof}
Throughout this proof we drop the dependence on the arm to simplify notation (unless necessary). We first analyze $\|\theta^*-\theta_u\|^2$ where $\theta_u$ is the mode of the posterior of arm $1$ after having received $n$ samples from the arm which satisfies:
\[ \frac{1}{n}\nabla \log \pi_1(\theta_u)+\nabla F_{1,n}(\theta_u)=0\]

Given this definition, and letting $\hat \theta = \theta_u-\theta^*$ we have that:
\begin{align*}
    \hat \theta^T \left(\nabla F_n(\theta^*)-\nabla F_n(\theta_u)\right)- \frac{1}{n}\hat \theta^T \nabla \log \pi(\theta_u)&=\hat \theta^T \nabla F_n(\theta^*)\\
    m\|\hat \theta\|^2&\le \frac{m}{2}\|\hat \theta\|^2 +\frac{1}{2m}\|\nabla F_n(\theta^*)\|^2 +\frac{\log B_1}{n}\\
    \|\hat \theta\|^2 &\le \frac{1}{m^2}\|\nabla F_n(\theta^*)\|^2+\frac{2\log B_1}{mn}
\end{align*}

Noting that $|a^T(\theta^*-\theta_u)|\le \sqrt{A^2\|\hat \theta\|^2}$ we find that:

\begin{align*}
p_{1,s}&=Pr\left(\alpha^T\left(\theta-\theta_u  \right)\ge \alpha^T\left(\theta^*-\theta_u  \right) -\epsilon\right)\\
&\ge Pr\left  (\alpha^T(\theta-\theta_u)\ge \underbrace{\sqrt{ \frac{2A^2\log B_1}{nm}+\frac{A^2}{m^2}\|\nabla F_n(\theta^*)\|^2} }_{=t}\right ),
\end{align*}
where we note that $\|F_n(\theta^*)\|$ in Proposition $1$ is a 1-dimensional $\frac{dL_a}{\sqrt{n\nu}}$ subgaussian random variable.

Now, since we know that the posterior over $\theta$ is $\gamma(n+1)L$-smooth and $\gamma mn$-strongly log concave, with mode $\theta_u$, we know from  e.g \cite{LogConcaveReview} Theorem 3.8 that the marginal density of $\alpha^T\theta$ is  $\frac{\gamma (n+1) L}{A^2}$-smooth and $\frac{\gamma mn}{A^2}$-strongly log-concave.

Thus we have that: 
\[ Pr\left(\alpha^T\left(\theta-\theta_u  \right)\ge t\right) \ge \sqrt{\frac{nm}{(n+1)L}} Pr(Z\ge t)  \]
where $Z\sim \mc{N}\left(0,\frac{A^2}{\gamma (n+1)L}\right)$.

Now using a lower bound on the cumulative density function of a Gaussian random variable, we find that, for $\sigma^2=\frac{A^2}{\gamma (n+1)L}$:

\[  p_{1,s}\ge \sqrt{\frac{nm}{2 \pi (n+1)L}}  \begin{cases} 
\frac{ \sigma t}{t^2+\sigma^2} e^{-\frac{t^2}{2\sigma^2}}\ \ \ \ &:  t>\frac{A}{\sqrt{\gamma (n+1)L}} \\ 0.34  \ \ \ \ &:  t\le \frac{A}{\sqrt{\gamma (n+1)L}} \end{cases}\]

Thus we have that:

\begin{align*}
    \frac{1}{p_{1,s}}&\le \sqrt{\frac{2 \pi (n+1)L}{nm}} \begin{cases} 
\frac{t^2+\sigma^2}{ \sigma t} e^{\frac{t^2}{2\sigma^2}}\ \ \ \ &:  t>\frac{A}{\sqrt{\gamma (n+1)L}} \\ \frac{1}{0.34}  \ \ \ \ &:  t\le \frac{A}{\sqrt{\gamma (n+1)L}} \end{cases}\\
&\le \sqrt{\frac{2 \pi (n+1)L}{nm}} \begin{cases} 
\left(\frac{t}{ \sigma}+1\right) e^{\frac{t^2}{2\sigma^2}}\ \ \ \ &:  t>\frac{A}{\sqrt{\gamma (n+1)L}} \\ 3  \ \ \ \ &:  t\le \frac{A}{\sqrt{\gamma (n+1)L}} \end{cases}
\end{align*}

Taking the expectation of both sides with respect to the samples $X_1,...,X_n$, letting  $\kappa=L/m$, and using the fact that $\frac{n+1}{n}\le 2$ for $n\ge 1$ we find that:

\begin{align*}
    \mb{E}\left[\frac{1}{p_{1,s}}\right]&\le 6\sqrt{\pi \kappa}+2\sqrt{\pi\kappa}\mb{E}\left[ \left(\frac{\sqrt{ \frac{2A^2\log B_1}{nm}+\frac{A^2}{m^2}\|\nabla F_n(\theta^*)\|^2}}{ \sigma}+1\right) e^{\frac{t^2}{2\sigma^2}} \right]\\
\end{align*}

Noting that $\sqrt{ \frac{2A^2\log B_1}{nm}+\frac{A^2}{m^2}\|\nabla F_n(\theta^*)^T\|^2} \le A\sqrt{\frac{2\log B_1}{nm}}+\frac{A}{m}\|\nabla F_n(\theta^*)\|$, and letting $Y=\|\nabla F_n(\theta^*) \|$ to simplify notation, this further simplifies:

\begin{align*}
    \mb{E}\left[\frac{1}{p_{1,s}}\right]&\le 6\sqrt{\pi \kappa}+2\sqrt{\pi\kappa}\mb{E}\left[ \left(\sqrt{4\gamma \kappa \log B_1}+\frac{A}{m\sigma}Y \right) e^{2\gamma\kappa \log B_1+\frac{(n+1) \gamma L}{2 m^2}Y^2} \right]\\
\end{align*}

Via Cauchy-Schwartz we can further develop this upper bound and find that:

\begin{align*}
    \mb{E}\left[\frac{1}{p_{1,s}}\right]&\le 6\sqrt{\pi \kappa}+2\sqrt{\pi\kappa}e^{2\gamma\kappa \log B_1} \left(\sqrt{4\gamma \kappa \log B_1}\mb{E}\left[ e^{\frac{(n+1) \gamma L}{2 m^2}Y^2} \right] + \frac{A}{m\sigma}\sqrt{\mb{E}\left[Y^2\right]} \sqrt{\mb{E}\left[e^{\frac{(n+1) \gamma L}{ m^2}Y^2}\right]}\right) \\
\end{align*}

Since $Y$ is sub-Gaussian,  $Y^2$ is sub-exponential such that:

\[ \mb{E}\left[ e^{\lambda Y^2}\right]\le e \ \ \ \text{and} \ \ \ \ \mb{E}\left[Y^2\right ]\le  2 \frac{d L^2}{\nu n} \]

for $\lambda<\frac{n \nu }{4 d L^2}$. Therefore if :

\[ \gamma = \frac{\nu m^2}{8dL^3}\]

Simplifying the bound further gives:

\begin{align*}
    \mb{E}\left[\frac{1}{p_{1,s}}\right]&\le 6\sqrt{\pi \kappa}+2\sqrt{\pi\kappa}e^{2\gamma\kappa \log B_1} \left(\sqrt{4\gamma \kappa \log B_1} e + 2\sqrt{\frac{ e \gamma (n+1)L}{m^2} \frac{d L^2}{\nu n}} \right) \\
    &\le  6\sqrt{\pi \kappa}+2\sqrt{\pi\kappa}e^{\frac{\log B_1}{4}}(\sqrt{\frac{ \log B_1}{2}} e+2\sqrt{e})
\end{align*}

where we have used the fact that $\kappa,d\ge1$ and the fact that we can assume without loss of generality that $L/\nu\ge 1$. Thus, this bound simplifies to:

\begin{align*}
    \mb{E}\left[\frac{1}{p_{1,s}}\right]&\le 6\sqrt{\pi \kappa}+2\sqrt{\pi\kappa}e^{2\gamma\kappa \log B_1} \left(\sqrt{4\gamma \kappa \log B_1} e + 2\sqrt{\frac{ e \gamma (n+1)L}{m^2} \frac{d L^2}{\nu n}} \right) \\
    &\le  2\sqrt{\pi\kappa}\left(B_1\right)^{\frac{1}{4}} \left(\sqrt{\frac{  \log B_1}{2}} e+7\right)\\
    &\le  4\sqrt{\pi\kappa}\left(B_1\right)^{\frac{1}{4}} \left(\sqrt{\log B_1}+4\right)\\
    &\le 64\sqrt{\kappa B_1}
\end{align*}
where we used the fact that $x^{1/4}(\sqrt{\log x}+4)\le 8 \sqrt{x}$ for $x\ge 1$ and $\sqrt{\pi}<2$ to simplify our bound.
\end{proof}

The last technical lemma upper bounds the two terms defined in Lemma~\ref{lemma:termI_andII_main}.

\begin{lemma}
\label{eq:true_posterior_final_lemma}
     Suppose the likelihood, true reward distributions, and priors satisfy Assumptions 1-3, then for  $\gamma_a=\frac{\nu_a m_a^2}{8d_aL_a^3}$:
     \begin{align}
     \sum_{s=1}^{T-1}\mb{E}\left[ \frac{1}{p_{1,s}}-1\right] \le  64\sqrt{\frac{L_1}{m_1}B_1} \left\lceil \frac{8eA_1^2}{m\Delta_a^2}(D_1+4\sigma_1 \log{2} ) \right\rceil +1
     \label{eq:regret_term1}
     \end{align}
     \begin{align}
     \sum_{s=1}^T \mb{E}\left[ \mb{I}\left (p_{a,s}>\frac{1}{T}\right )\right]  \le \frac{8 e A_a^2}{m\Delta_a^2}(D_a+2\sigma_a\log(T))
     \label{eq:regret_term2}
     \end{align}
Where for $a \in \mc{A}$,  $D_a$ is given by:
\[ D_a= \log B_a+\frac{8d_a^2L_a^3}{m_a^2\nu_a}  \ \ \ \ \ \sigma_a=\frac{256 d_a L_a^3}{m_a^2\nu_a}+\frac{8d_a L_a^2}{m_a\nu_a}\]
\end{lemma}

\begin{proof}
We begin by showing that \eqref{eq:regret_term1} holds. To do so, we first note that, by definition $p_{1,s}$ satisfies:
\begin{align}
    p_{1,s}&=\mb{P}(r_{1,t}(s)>\bar r_1-\epsilon |\mc{F}_{t-1})\\
    &=1-\mb{P}(r_{1,t}(s)-\bar r_1<-\epsilon |\mc{F}_{t-1})\\
    &\ge 1-\mb{P}(|r_{1,t}(s)-\bar r_1|>\epsilon |\mc{F}_{t-1})\\
    &\ge 1-\mb{P}_{\theta \sim \mu^{(s)}_{1}}\left(\|\theta -\theta^*\|> \frac{\epsilon}{A_1} \right)
    \label{eq:seq}
\end{align}

where the last inequality follows from the fact that $r_{1,t}(s)$ and $\bar r_1$ are  $A_a$-Lipschitz functions of $\theta \sim \mu^{(s)}_1$ and $\theta^*$ respectively.

We then use the fact that  the posterior distribution $\mb{P}_{\theta \sim \mu^{(s)}_{1}}$ satisfies the concentration bound from Theorem~\ref{thm:theorem1}. Therefore, we have that:
\begin{align}
\label{eq:posterior_exp_form}
   \mb{P}_{\theta \sim \mu^{(s)}_{1}}\left(\|\theta -\theta^*\|> \frac{\epsilon}{A_1} \right) \le \exp\left(-\frac{1}{2\sigma_1}\left( \frac{m n \epsilon^2}{2e A^2_1}-D_1\right) \right),
\end{align}
where we use the constant $D_1$ and $\sigma_1$ defined in the proof of Theorem~\ref{thm:theorem1} to simplify notation. We remark that this bound is not useful unless:
\[ n>\frac{2eA_1^2}{\epsilon^2m}D_1.\]
Thus, choosing $\epsilon=(\bar r_1 -\bar r_a)/2=\Delta_a/2$ and $\ell$ as:
\[ \ell=\left\lceil \frac{8eA_1^2}{m\Delta_a^2}(D_1+2\sigma_1 \log{2} ) \right\rceil.\]
we proceed as follows:
\begin{align*}
    \sum_{s=\ell}^{T-1}\mb{E}\left[ \frac{1}{p_{1,s}}-1 \right]
    &\le  \sum_{s=0}^{T-1} \frac{1}{1-\frac{1}{2}\delta(s)}-1\\
    &\le  \int_{s=1}^{\infty} \frac{1}{1-\frac{1}{2}\delta(s)}-1 ds
\end{align*}
where:
\[ \delta(s)= \exp\left(-\frac{1}{2\sigma_1}\left(\frac{m \epsilon^2}{2e A^2_{1}} s\right) \right),  \]
and the first inequality follows from our choice of $\ell$ and the second by upper bounding the sum by an integral. To finish, we write $\delta(s)=exp(-c*s)$, and solve the integral to find that :
\[ \int_{s=1}^{\infty} \frac{1}{1-\frac{1}{2}\delta(s)}-1 ds = \frac{\log{2}- \log{(2 e^{c}-1)}}{c}+1 \le \frac{\log{2}}{c}+1.\] 
plugging in for  $c$ gives:
\begin{align*}
\sum_{s=1}^{T-1}\mb{E}\left[ \frac{1}{p_{1,s}}-1 \right]  &\le   \sum_{s=1}^{\ell-1}\mb{E}\left[ \frac{1}{p_{1,s}}-1 \right]+\frac{8 e A_1^2}{m\Delta_a^2}2\sigma_1\log{2} +1 \\
&\le 64\sqrt{\frac{L_1}{m_1}B_1} \left\lceil \frac{8eA_1^2}{m\Delta_a^2}(D_1+4\sigma_1 \log{2} ) \right\rceil +1\\
\end{align*}

To show that \eqref{eq:regret_term2} holds, we do a similar derivation as in \eqref{eq:seq}:
\begin{align*}
    \sum_{s=1}^T \mb{E}\left[ \mb{I}\left (p_{a,s}>\frac{1}{T}\right )\right] &= \sum_{s=1}^T \mb{E}\left[ \mb{I}\left (\mb{P}(r_{a,t}(s)-\bar r_a>\Delta_a-\epsilon |\mc{F}_{t-1}) >\frac{1}{T}\right )\right] \\
    & = \sum_{s=1}^T \mb{E}\left[ \mb{I}\left (\mb{P}(r_{a,t}(s)-\bar r_a>\frac{\Delta_a}{2} |\mc{F}_{t-1}) >\frac{1}{T}\right )\right] \\
    &\le \sum_{s=1}^T \mb{E}\left[ \mb{I}\left (\mb{P}\left(|r_{a,t}(s)-\bar r_a|>\frac{\Delta_a}{2}\bigg|\mc{F}_{t-1}\right) >\frac{1}{T}\right )\right] \\
     &\le \sum_{s=1}^T \mb{E}\left[ \mb{I}\left (\mb{P}_{\theta \sim \mu^{(s)}_{a}[\gamma_a]}\left(\|\theta-\theta^*\|>\frac{\Delta_a}{2A_a}\right) >\frac{1}{T}\right )\right]. \\
\end{align*}
Using the posterior concentration result from Theorem~\ref{thm:theorem1} we upper bound the number of pulls $\bar n$ of arm $a$ such that for all $n\ge \bar n$:
\[ \mb{P}_{\theta \sim \mu^{(n)}_{a}[\gamma_a]}\left(\|\theta-\theta^*\|>\frac{\Delta_a}{2A_a}\right) \le \frac{1}{T}. \]
Since the posterior for arm $a$ after $n$ pulls of arm $a$ has the same form as in \eqref{eq:posterior_exp_form}, we can choose $\bar n$ as:
\[ \bar n=\frac{8 e A_a^2}{m\Delta_a^2}(D_a+2\sigma_a\log(T)). \]
This completes the proof.

\end{proof}

Given these lemma's the proof of Theorem~\ref{thm:exact_regret} is straightforwards. For clarity, we restate the theorem below:
\begin{theorem2}
    When the likelihood and true reward distributions satisfy Assumptions 1-3 and $\gamma_a=\frac{\nu_am_a^2}{8d_aL_a^3}$ we have that the expected regret after $T>0$ rounds of Thompson sampling with exact sampling satisfies:
    \begin{align*}
        \mb{E}[R(T)]&\le \sum_{a>1} \frac{ C A_a^2}{m_a\Delta_a}\left(\log B_a+d_a^2\kappa_a^3+d_a\kappa_a^3 \log(T)\right)+\sqrt{\kappa_1 B_1}\frac{C A_1^2}{m_1\Delta_a} \left(1+\log B_1+d_1^2\kappa_1^3\right) +\Delta_a
    \end{align*}
    Where $C$ is a universal constant independent of problem-dependent parameters.
\end{theorem2}

\begin{proof}

We invoke Lemmas~\ref{lemma:termI} and \ref{lemma:termII_exact}, to find that:

\begin{align}
   &\mb{E}\left[T_a(T)\right] \le \underbrace{\sum_{s=1}^{T-1}\mb{E}\left[ \frac{1}{p_{1,s}}-1 \right]}_{(I)}+\underbrace{\sum_{s=1}^T \mb{E}\left[ \mb{I}\left (1-p_{a,s}>\frac{1}{T}\right )\right]}_{(II)}
   \label{eq:regretterm1}
\end{align}

Now, invoking Lemma~\ref{eq:true_posterior_final_lemma}, we use the upper bounds for terms $(I)$ and $(II)$ in the regret decomposition and expanding $D_a$ and $D_1$ to give that:

\begin{align*}
    \mb{E}[R(T)]&\le\sum_{a>1} \frac{8 e A_a^2}{m_a\Delta_a}\left(\log B_a+8d_a\kappa_a^3\left(d_a+ 66 \log(T)\right)\right)\\ &\qquad \qquad  +\sqrt{\kappa_1 B_1}\frac{512 e A_a^2}{m_1\Delta_a^2} \left(1+\log B_1+8d_1\kappa_1^3\left(d_1+ 132 \log(2)\right)\right) +\Delta_a\\
    &\le\sum_{a>1} \frac{ C A_a^2}{m_a\Delta_a}\left(\log B_a+d_a^2\kappa_a^3+d_a\kappa_a^3 \log(T)\right)\\ 
    &\qquad \qquad  +\sqrt{\kappa_1 B_1}\frac{C A_1^2}{m_1\Delta_a} \left(1+\log B_1+d_1^2\kappa_1^3\right) +\Delta_a
\end{align*}

\end{proof}

\subsection{Regret of Approximate Sampling}
\label{sec:approx_regret_proof_appendix}
For the proof of Theorem~\ref{thm:approx_regret}, we proceed similarly as for the proof of Theorem~\ref{thm:exact_regret}, but require another intermediate lemma to deal with the fact that the samples from the arms are no longer conditionally independent given the filtration (due to the fact that we use the last sample as the initialization of the filtration). To do so, we first define the event:
\[ Z_{a}(T)=\cap_{t=1}^{T-1} Z_{a,t},\]
where:
\[ Z_{a,t}=\left\{\|\theta_{a,t}-\theta^*_a\| <  \sqrt{\frac{18e}{n m_a}} \left( d_a+\log B_a+2 \left( 16+ \frac{4dL_a^2}{\nu_a m_a}\right) \log1/\delta_1\right)^{\frac{1}{2}} \right \}, \]

\begin{lemma}
Suppose the likelihood and reward distributions satisfy Assumptions 1-4, Then the regret of a Thompson sampling algorithm with approximate sampling can be decomposed as:

\begin{align}
    \mb{E}[R(T)] \le \sum_{a>1} \Delta_a \mb{E}\left[T_a(T) \,\Bigg|\,Z_{a}(T) \cap Z_{1}(T)\right] +2\Delta_a 
    \label{eq:regret_decomp_approx}
\end{align}

\label{lemma:approx_Regret_decomposition}
\end{lemma}

\begin{proof}

We begin by conditioning on the event $Z_{a}(T) \cap Z_{1}(T)$ for each $a \in \mc{A}$, where we note that by construction $p_Z=\mb{P}((Z_{a}(T)^c  \cup Z_{1}(T)^c )) \le \mb{P}(Z_{1}(T)^c )+\mb{P}(Z_{a}(T)^c )=2T\delta_1)$ (since via Lemma~\ref{lemma:approx_dist}, the probability of each event in $Z_{a}(T)^c$ and $Z_{1}(T)^c$ is less than $\delta_1$).

Therefore, we must have that:
\begin{align*}
\mb{E}[T_a(T)] &\le  \mb{E}\left[T_a(T) \,\Bigg|\, Z_{a}(T) \cap Z_{1}(T)\right] +\mb{E}\left[T_a(T) \,\Bigg|\, (Z_{a}(T)^c \cup Z_{1}(T)^c)\right]p_Z \\
 &\leq \mb{E}\left[T_a(T) \,\Bigg|\,Z_{a}(T) \cap Z_{1}(T)\right] +2T\delta_3 \mb{E}\left[T_a(T) \,\Bigg|\, (Z_{a}(T)^c \cup Z_{1}(T)^c)\right]\\
 & \le \mb{E}\left[T_a(T) \,\Bigg|\,Z_{a}(T) \cap Z_{1}(T)\right] +2\delta_3 T^2,
\end{align*}
where in the first line we use the fact that $1-p_Z\le 1$ and in the last line we used the fact that $T_a(T)$ is trivially less than $T$. Choosing $\delta_1=1/T^2$ completes the proof.
\end{proof}

With this decomposition in hand, we can now proceed as in Lemma~\ref{lemma:anti_conc_exact} to provide anti-concentration guarantees for the approximate posteriors.

\begin{lemma}
\label{lemma:approx_anti_conc}
Suppose the likelihood and true reward distributions satisfy Assumptions 1-4: then if $\gamma_1=\frac{\nu m^2}{32(16L\nu m+4dL^3)}$, for all $n=1,...,T$ all samples from the the (stochastic gradient) ULA method with the hyperparameters and runtime as described in Theorem~\ref{theorem3} satisfy: 

\[ \mb{E}\left[ \frac{1}{p_{1,n}} \right] \le 27\sqrt{B_1}\]
\end{lemma}

\begin{proof}
We begin by using the last step of our Langevin Dynamics and show that it exhibits the desired anti-concentration properties. In particular, we know that $\theta_{1,t} \sim \mc{N}(\theta_{1,Nh},\frac{1}{\gamma}I)$, such that:

\begin{align*}
p_{1,s}&=Pr\left(\alpha^T\left(\theta-\theta_{1,Nh}  \right)\ge \alpha^T\left(\theta^*-\theta_{1,Nh}  \right) -\epsilon\right)\\
&\ge Pr\left  (Z\ge \underbrace{A\|\theta_{1,Nh}-\theta_*\|}_{:=t}\right) 
\end{align*}
where $Z \sim \mc{N}(0,\frac{A^2}{nL\gamma}I)$ by construction.

Now using a lower bound on the cumulative density function of a Gaussian random variable, we find that, for $\sigma^2=\frac{A^2}{nL\gamma}$:

\[  p_{1,s}\ge \sqrt{\frac{1}{2 \pi}}  \begin{cases} 
\frac{\sigma t}{t^2+\sigma^2} e^{-\frac{t^2}{2\sigma^2}}\ \ \ \ &:  t>\frac{A}{\sqrt{nL\gamma}} \\ 0.34  \ \ \ \ &:  t\le \frac{A}{\sqrt{nL\gamma}} \end{cases}\]

Thus we have that:

\begin{align*}
    \frac{1}{p_{1,s}}
&\le \sqrt{2 \pi} \begin{cases} 
\left(\frac{t}{ \sigma}+1\right) e^{\frac{t^2}{2\sigma^2}}\ \ \ \ &:  t> \frac{A}{\sqrt{nL\gamma}} \\ 3  \ \ \ \ &:  t\le \frac{A}{\sqrt{nL\gamma}} \end{cases}
\end{align*}

Taking the expectation of both sides with respect to the samples $X_1,...,X_n$, we find that:

\begin{align*}
    \mb{E}\left[\frac{1}{p_{1,s}}\right]&\le 3\sqrt{2\pi}+\sqrt{2\pi}\mb{E}\left[ \left(\sqrt{nL\gamma}\|\theta_{1,Nh}-\theta_*\|+1\right) e^{nL\gamma \|\theta_{1,Nh}-\theta_*\|^2} \right]\\
    &\le 3\sqrt{2\pi}+\sqrt{2\pi nL\gamma}\sqrt{\mb{E}\left[\|\theta_{1,Nh}-\theta_*\|^2\right]}\sqrt{\mb{E}\left[e^{nL\gamma \|\theta_{1,Nh}-\theta_*\|^2}\right]}+\sqrt{2\pi}\mb{E}\left[e^{\frac{nL\gamma}{2} \|\theta_{1,Nh}-\theta_*\|^2} \right] 
\end{align*}

Now, we remark that, from Theorems~\ref{theorem_ULA} and~\ref{theorem_SG}, we have that for both approximate sampling schemes:

\[\mb{E}\left[\|\theta_{1,Nh}-\theta^*\|^2 \right]\le \frac{18}{mn}\left(d+\log B+32+ \frac{8dL^2}{\nu m} \right) \]

Further, we note that $\|\theta_{1,Nh}-\theta^*\|^2$ is a sub-exponential random variable. To see this, we analyze its moment generating function:

\begin{align*}
    \mb{E}[e^{nL\gamma\|\theta_{1,Nh}-\theta^*\|^2}]&= 1+\sum_{i=1}^\infty  \mb{E}\left[\frac{(nL\gamma)^i\|\theta_{1,Nh}-\theta_*\|^{2i}}{i!}\right]\\
\end{align*}

Borrowing the notation from the proof of Theorem~\ref{thm:theorem1}, we know that

\[ \mb{E}\left[\|\theta_{1,Nh}-\theta_*\|^{2p}\right]\le 3\left(\frac{2D}{mn}+\frac{4\sigma p}{mn} \right)^{p}\]

where:
\[ D=d+\log B \ \ \ \ \text{and} \ \  \ \ \ \sigma=16+\frac{4dL^2}{\nu m}\]
Plugging this in above gives:

\begin{align*}
    \mb{E}[e^{\gamma\|\theta_{1,Nh}-\theta_*\|^2}]&\le  1+3\sum_{i=1}^\infty  \frac{\left(\frac{2nL\gamma D+4 nL\gamma\sigma i}{mn} \right)^i}{i!}\\
    &\le 1+\frac{3}{2}\sum_{i=1}^\infty\frac{1}{{i!}}\left(\frac{4nL\gamma D}{mn} \right)^i +\frac{3}{2}\sum_{i=1}^\infty \frac{1}{i!}\left(\frac{8nL\gamma \sigma i}{nm}\right)^i \\
    &\le \frac{3}{2} e^{\frac{4nL\gamma D}{mn}}+\frac{3}{2}\sum_{i=1}^\infty \left( \frac{8nL\gamma e\sigma }{nm}\right)^i
\end{align*}
where, we have use the identities $(x+y)^i\le 2^{i-1}(x^i+y^i)$ for $i\ge 1$, and $i!\ge (i/e)^i$ to simplify the bound.

If $\gamma\le \frac{m}{32 L\sigma}$, then we have that:

\[ \mb{E}[e^{nL\gamma\|\theta_{1,Nh}-\theta_*\|^2}]\le \frac{3}{2}\left(e^{\frac{4nL\gamma D}{m}} +2.5\right) \]

which, together with the upper bound on $\gamma$ gives:

\begin{align*}
    \mb{E}\left[\frac{1}{p_{1,s}}\right]
    &\le 3\sqrt{2\pi}+\frac{3}{2}\sqrt{\frac{16\pi nL\gamma}{m} \left(D +2\sigma \right)}\left(e^{\frac{2nL\gamma D}{m}}+2\right)+\frac{3}{2}\sqrt{2\pi}\left(e^{\frac{4nL\gamma D}{m}}+7.5\right)\\
    &\le 3\sqrt{2\pi}+\frac{3}{2}\left(\sqrt{\frac{\pi(d+\log B)}{2\sigma}}+\sqrt{\pi}\right)\left(e^{\frac{d+\log B}{16 \sigma }}+2\right)+\frac{3}{2}\sqrt{2\pi}\left(e^{\frac{d+\log B}{8 \sigma}}+2.5\right)
\end{align*}

where we used the sub-additivity of $\sqrt{x}$, the fact that $\sqrt{\frac{3}{2}}<\frac{3}{2}$, $sqrt{2.5}<2$ and substituted in the values for $\sigma$ and $D$ to simplify the boung. Finally since$\frac{L^2}{m\nu}>1$, we find that $\sigma>max(4d,1)$, allowing us to simplify the bound further to:

\begin{align*}
        \mb{E}\left[\frac{1}{p_{1,s}}\right]&\le 3\sqrt{2\pi} + \frac{3}{2}\sqrt{\frac{\pi}{8}+\frac{\log B}{2}} \left(2B^{1/    16}+2\right)+\frac{3}{2}\sqrt{2\pi}\left(2B^{1/8}+2.5\right)\\
    &\le 18+ \frac{3}{\sqrt{2}}\underbrace{\left(B^{1/16}+B^{1/16}\sqrt{\log B}+\log B+2 B^{1/8}\right)}_{I}\\
    &\le 18+12/\sqrt{2}\sqrt{B} \le 27\sqrt{B}
\end{align*}

where to simplify the bound we used the fact that $\sqrt{\pi}<2$ and $I\le 4\sqrt{B}$ and that $18+12/\sqrt{2}x\le 27x$ for $x\ge 1$.

\end{proof}

With this lemma in hand, we can now proceed as in Lemma~\ref{eq:true_posterior_final_lemma} to finalize the proof of Theorem~\ref{thm:approx_regret}.

\begin{lemma}
\label{lemma:approx_posterior_final_lemma}
     Suppose the likelihood, true reward distributions, and priors satisfy Assumptions 1-4, the samples are generated from the sampling schemes described in Theorem~\ref{theorem_SG} and Theorem~\ref{theorem_ULA}, and $\gamma_a=\frac{m_a}{32L_a \sigma_a}$ then:
     \begin{align}
     \sum_{s=1}^{T-1}\mb{E}\left[ \frac{1}{\widehat p_{1,s}}-1 \bigg| Z_{1}(T)\right] \le  27\sqrt{B_1}\left\lceil \frac{144eA_1^2}{m\Delta_a^2}( d_1+\log B_1+4\sigma_1 \log T+12d_1\sigma_1\log2 ) \right\rceil +1
     \label{eq:approx_regret_term1}
     \end{align}
     \begin{align}
     \sum_{s=1}^T \mb{E}\left[ \mb{I}\left (\widehat p_{a,s}>\frac{1}{T}\right )\bigg| Z_{a}(T)\right]  \le \frac{144 e A_a^2}{m\Delta_a^2}(d_a+\log B_a+10d_a\sigma_a\log(T)),
     \label{eq:approx_regret_term2}
     \end{align}
where $\widehat p_{a,s}$ is the distribution of a sample from the approximate posterior $\widehat \mu_a$ after $s$ samples have been collected, and for $a \in \mc{A}$,  $\sigma_a$ is given by:
\[ \sigma_a=16+\frac{4d_a L_a^2}{m_a \nu_a}. \]
\end{lemma}

\begin{proof}
We begin by showing that \eqref{eq:approx_regret_term1} holds. To do so, we proceed identically as in the proof of Lemma~\ref{eq:true_posterior_final_lemma} to  note that, by definition $\widehat p_{1,s}$ satisfies:
\begin{align}
    \widehat p_{1,s}&=\mb{P}(r_{1,t}(s)>\bar r_1-\epsilon |\mc{F}_{t-1})\\
    &=1-\mb{P}(r_{1,t}(s)-\bar r_1<-\epsilon |\mc{F}_{t-1})\\
    &\ge 1-\mb{P}(|r_{1,t}(s)-\bar r_1|>\epsilon |\mc{F}_{t-1})\\
    &\ge 1-\mb{P}_{\theta \sim \widehat \mu^{(s)}_{1}}\left(\|\theta -\theta^*\|> \frac{\epsilon}{A_1} \right),
    \label{eq:seq1}
\end{align}
where the last inequality follows from the fact that $r_{1,t}(s)$ and $\bar r_1$ are  $A_a$-Lipschitz functions of $\theta \sim \mu^{(s)}_1$ and $\theta^*$ respectively.

We then use the fact that conditioned on $ Z_{1}(T)$, the approximate posterior distribution $\mb{P}_{\theta \sim \widehat \mu^{(s)}_{1}}$ satisfies the identical concentration bounds from Lemmas~\ref{lemma:W_p_delta_SGLD} and Lemma~\ref{lemma:W_p_delta_ULA}. Substituting in the assumed value of $\gamma_1$, and simplifying, we have that the distribution of the samples conditioned on $Z_1(T)$ satisfy:

\[ \mathbb{P}_{\theta_{1,t} \sim \bar \mu^{(s)}_{1} [\gamma_1]} \left(\|\theta_{1,t}-\theta_1^\ast\|_2 > \sqrt{\frac{36e}{m_1 n} \left( d_1+\log B_1 +4\sigma_1\log T+6d_1\sigma_1\log{1/\delta_2} \right)} \bigg| Z_{n-1} \right)<\delta_2. , \]

Equivalently, we have that:
\begin{align}
\label{eq:conditional_posterior_exp_form}
   \mb{P}_{\theta \sim \bar \mu^{(s)}_{1}}[\gamma_1]\left(\|\theta -\theta^*\|> \frac{\epsilon}{A_1} \right) \le exp\left(-\frac{1}{6d_1\sigma_1}\left( \frac{m_1 n \epsilon^2}{36e A^2_{1}}-\bar D_1\right) \right),
\end{align}
where we define $\bar D_1= d_1+\log B_1+4\sigma\log T $, to simplify notation. We remark that this bound is not useful unless:
\[ n>\frac{16eA_1^2}{\epsilon^2m_1} \bar D_1.\]
Thus, choosing $\epsilon=(\bar r_1 -\bar r_a)/2=\Delta_a/2$, we can choose $\ell$ as:
\[ \ell=\left\lceil \frac{144 eA_1^2}{m\Delta_a^2}(\bar D_1+6d_1\sigma_1 \log{2} ) \right\rceil.\]
With this choice of $\ell$, we proceed exactly as in the proof of Lemma~\ref{eq:true_posterior_final_lemma} to find that :
\begin{align*}
     \sum_{s=1}^{T-1}\mb{E}\left[ \frac{1}{\widehat p_{1,s}}-1 \bigg| Z_{1}(T)\right] &\le 27\sqrt{B_1}\ell+\sum_{s=\ell}^{T-1}\mb{E}\left[ \frac{1}{p_{1,s}}-1 \bigg| Z_{1}(T)\right] \\
     &\le  27\sqrt{B_1}\left\lceil \frac{144eA_1^2}{m\Delta_a^2}(\bar D_1+12d_1\sigma_1\log{2} ) \right\rceil +1,
\end{align*}

where we used the upper bound from Lemma~\ref{lemma:approx_anti_conc} to bound the first $\ell$ terms in the first inequality.

To show that \eqref{eq:approx_regret_term2} holds, we use a similar derivation as in \eqref{eq:seq1}:
\begin{align*}
    &\sum_{s=1}^T \mb{E}\left[ \mb{I}\left (p_{a,s}>\frac{1}{T}\right )\bigg| Z_{a}(T) \right] \le \sum_{s=1}^T \mb{E}\left[ \mb{I}\left (\mb{P}_{\theta \sim \bar \mu^{(s)}_{a}[\gamma_a]}\left(\|\theta-\theta^*\|>\frac{\Delta_a}{2A_a}\right) >\frac{1}{T}\right )\bigg| Z_{a}(T)\right]
\end{align*}
Since on the event $Z_{a}(T)$, the posterior concentration result from Lemmas~\ref{lemma:W_p_delta_SGLD} and Lemma~\ref{lemma:W_p_delta_ULA} holds, it remains to upper bound the number of pulls $\bar n$ of arm $a$ such that for all $n\ge \bar n$:
\[ \mb{P}_{\theta \sim \bar \mu^{(n)}_{a}[\gamma_a]}\left(\|\theta-\theta^*\|>\frac{\Delta_a}{2A_a}\right) \le \frac{1}{T}. \]

Since the posterior for arm $a$ after $n$ pulls of arm $a$ has the same form as in \eqref{eq:posterior_exp_form}, we can choose $\bar n$ as:
\[ \bar n=\frac{144 e A_a^2}{m\Delta_a^2}(\bar D_a+6d_a\sigma_a\log(T)). \]
Using the fact that $d_a>\ge1$ to simplify the bound completes the proof.
\end{proof}

Putting the results of Lemmas~\ref{lemma:approx_Regret_decomposition} and ~\ref{lemma:approx_posterior_final_lemma} together gives us our final theorem:

\begin{theorem2}[Regret of Thompson sampling with (stochastic gradient) Langevin algorithm]
\label{thm:approx_regret_appendix}
When the likelihood and true reward distributions satisfy Assumptions 1-4: we have that the expected regret after $T>0$ rounds of Thompson sampling with the (stochastic gradient) ULA method with the hyper-parameters and runtime as described in Lemmas~\ref{lemma:W_p_delta_ULA} (and~\ref{lemma:W_p_delta_SGLD} respectively), and $\gamma_a=\frac{\nu_a m_a^2}{32(16L_a\nu_a m_a+4d_aL_a^3)}=O\left(\frac{1}{d_a \kappa_a^3}\right)$ satisfies:
\begin{align*}
        \mb{E}[R(T)]\le& \, \sum_{a>1} \frac{C A_a^2}{ m_a\Delta_a}\left( d_a+ {\log B_a}+d_a^2\kappa_a^2 \log T \right)\\ &\qquad \qquad  +\frac{C \sqrt{B_1} A_1^2}{m_1\Delta_a}\left( 1+\log B_1+d_1\kappa_1^2\log{T}+d_1^2\kappa_1^2 \right) +3\Delta_a.
\end{align*}    
    where $C$ is a universal constant that is independent of problem dependent parameters and $\kappa_a=L_a/m_a$.
\end{theorem2}
\begin{proof}

To begin, we invoke Lemma~\ref{lemma:approx_Regret_decomposition}, which shows that we only need to bound the number of times a suboptimal arm $a \in \mc{A}$ is chosen on the `nice' event $ Z_{1}(T) \cap Z_{a}(T)$ where the gradient of the log likelihood has concentrated and the approximate samples have been in high probability regions of the posteriors. We then invoke Lemmas~\ref{lemma:termI} and \ref{lemma:termII_exact}, to find that:

\begin{align}
   &\mb{E}\left[T_a(T) \,\Bigg|\, Z_{1}(T) \cap Z_{a}(T)\right] \le 1+\ell \\
   &\qquad \qquad +\underbrace{\sum_{s=\ell}^{T-1}\mb{E}\left[ \frac{1}{p_{1,s}}-1 \bigg|  Z_{1}(T)\right]}_{(I)}+\underbrace{\sum_{s=1}^T \mb{E}\left[ \mb{I}\left (1-p_{a,s}>\frac{1}{T}\right )\bigg|  Z_{a}(T)\right]}_{(II)}
   \label{eq:approx_regretterm1}
\end{align}

Now, invoking Lemma~\ref{eq:true_posterior_final_lemma}, we use the upper bounds for terms $(I)$ and $(II)$ in the regret decomposition, use our choice of both $\delta_1$ and $\delta_3=1/T^2$, expanding $D_a$ and $D_1$, and use the fact that $\lceil x \rceil \le x+1$ to give that:

\begin{align*}
    \mb{E}[R(T)]&\le\sum_{a>1} \frac{144 e A_a^2}{m_a\Delta_a}\left( d_a+\log B_a+10d_a\left(16+\frac{4d_aL_a^2}{\nu_a m_a}\right)\log(T) \right)\\ &\qquad \qquad  +27\sqrt{B_1}\frac{144eA_1^2}{m_1\Delta_a}\left( 1+d_1+\log B_1+4\left(16+\frac{4d_1L_a^2}{\nu_1 m_1}\right)\left(\log{T}+3d_1\log2\right) \right) +3\Delta_a.\\
    &\le\sum_{a>1} \frac{C A_a^2}{ m_a\Delta_a}\left( d_a+ {\log B_a}+d_a^2\kappa_a^2 \log T \right)\\ &\qquad \qquad  +\frac{C \sqrt{B_1} A_1^2}{m_1\Delta_a}\left( 1+\log B_1+d_1\kappa_1^2\log{T}+d_1^2\kappa_1^2 \right) +3\Delta_a.
\end{align*}

 Using the fact that $\kappa_a\ge 1$ and that $d_1\ge 1$ allows us to simplify to get our desired result.

\end{proof} %

\section{Details in the Numerical Experiments}
\label{sec:numex}

We benchmark the effectiveness of approximate Thompson sampling against both UCB and exact Thompson sampling across three different Gaussian multi-armed bandit instances with $10$ arms. We remark that the use of Gaussian bandit instances is due to the fact that the closed form for the posteriors allows for us to properly benchmark against exact Thompson sampling and UCB, though our theory applies to a broader family of prior/likelihood pairs.

In all three instances we keep the reward distributions for each arm fixed such that their means are evenly spaced from $0$ to $10$ ($\bar r_1=1$, $\bar r_2=2$, and so on), and their variances are all $1$. In each instance we use different priors over the means of the arms to analyze whether the approximate Thompson sampling algorithms preserve the performance of exact Thompson sampling.
In the first instance, the priors reflect the correct orderings of the means. We use Gaussian priors with variance 4, and means evenly spaced between $5$ and $10$ such that $\mb{E}_{\pi_1}[X]=5$, and $\mb{E}_{\pi_{10}}[X]=10$. In the second instance, the prior for each arm is a Gaussian with mean $7.5$ and variance $4$. Finally, the third instance is `adversarial' in the sense that the priors reflects the complete opposite ordering of the means. In particular, the priors are still Gaussians such that their means are evenly spaced between $5$ and $10$ with variance $4$, but this time $\mb{E}_{\pi_1}[X]=10$, and $\mb{E}_{\pi_{10}}[X]=5$.

As suggested in our theoretical analysis in Section~\ref{sec:approx}, we use a constant number of steps for both ULA and SGLD to generate samples from the approximate posteriors. In particular, for ULA, we take $N=100$ and  double that number for SGLD $N=200$. We also choose the stepsize for both algorithms to be $\frac{1}{32 T_a(t)}$. For SGLD, we use a batch size of $\min(T_a(t),32)$. Further, since $d_a=\kappa_a=1$ since this is a Gaussian family, we take the scaling to be $\gamma_a=1$. 
The regret is calculated as $\sum_{t=1}^T \bar r_{10}-\bar r_{A_t}$ for the three algorithms and is averaged across 100 runs. Finally, for the implementation of UCB, we used the time-horizon tuned UCB  \citep{TorCsabaBandits} and the known variance, $\sigma^2$ of the arms in the upper confidence bounds (to maintain a level playing field between algorithms):
\[ UCB_a(t)=\frac{1}{T_a(t)}\sum_{i=1}^{t-1} X_{A_i}\ind{A_i=a} + \sqrt{\frac{4\sigma^2\log{2T}}{T_a(t)}}.
\]


%
%
%
%
%
%
%
%
%
%
%
%
%
%
%
%
%
%
%
%
%
%


%
%

\end{document}